%% file: goaloriented.tex
\documentclass{article}

\usepackage[preprint]{neurips_2021}

\usepackage[utf8]{inputenc} 
\usepackage[T1]{fontenc}    
\usepackage{hyperref}       
\usepackage{url}            
\usepackage{booktabs}       
\usepackage{amsfonts}       
\usepackage{nicefrac}       
\usepackage{microtype}      

\usepackage{longtable}
\usepackage{graphicx}
\usepackage{caption}
\usepackage{subcaption}

\title{Formatting Instructions For NeurIPS 2020}

\author{%
  Léonard Blier\\
  Inria, Université Paris Saclay, FAIR
  \And
  Yann Ollivier
  \\ FAIR
}

\usepackage{algorithm}
\usepackage{algorithmic}

\input{header.en.tex} 

\usepackage{xcolor}
\usepackage{wrapfig}
\usepackage{caption}
\newcommand{\todo}[1]{{\color{red} TODO: {#1}}}

\def\del{\operatorname{\delta}\hspace{-.4ex}{}} 

\newcommand{\tar}{\mathrm{tar}}

\newcommand{\supp}{\mathrm{supp}}

\newcommand{\tildem}{\tilde m}

\newcommand{\Qspace}{\mathcal{Q}}

\newcommand{\td}{\mathrm{TD}}
\newcommand{\tdn}{\mathrm{TD}^{(n)}}

\newcommand{\her}{\mathrm{HER}}
\newcommand{\uvfa}{\mathrm{UVFA}}
\newcommand{\ddqn}{\delta\text{-}\mathrm{DQN}}
\newcommand{\dtd}{\delta\text{-}\mathrm{TD}}
\newcommand{\dtdn}{\delta\text{-}\mathrm{TD}(n)}
\newcommand{\dac}{\delta\text{-}\mathrm{AC}}
\newcommand{\dppo}{\delta\text{-}\mathrm{PPO}}
\newcommand{\adv}{\mathrm{adv}}
\newcommand{\expl}{\mathrm{expl}}
\newcommand{\ks}{K_\mathcal{S}}

\newcommand{\rhoSA}{\rho_{\mathrm{SA}}}
\newcommand{\rhoSG}{\rho_{\mathrm{SG}}}

\usepackage{array}

\newcolumntype{L}[1]{>{\let\newline\\\arraybackslash\hspace{0pt}}m{#1}}

\title{
  Unbiased Methods for Multi-Goal RL
}

\begin{document}

\maketitle

\begin{abstract}
  In multi-goal reinforcement learning (RL) settings, the reward for each
  goal is sparse, and located in a small neighborhood of the goal.
In large dimension, the
probability of reaching a reward vanishes and the agent receives little
learning signal. 
Methods such as \emph{Hindsight Experience Replay} (HER) tackle this
issue by also learning from realized but unplanned-for goals. But HER is known to
introduce bias \citep{plappert2018multi}, and can converge to low-return policies by
overestimating chancy outcomes.
First, we vindicate HER by proving that it is actually unbiased in \emph{deterministic}
environments, such as many optimal control settings.
Next, 
for stochastic environments in continuous spaces, we tackle sparse rewards by directly taking the infinitely sparse
reward limit.
We fully formalize the problem of multi-goal RL with infinitely sparse
Dirac rewards at each goal. We introduce unbiased deep $Q$-learning and
actor-critic algorithms that can handle such infinitely sparse rewards,
and test them in toy environments.
\end{abstract}

Most standard \emph{reinforcement learning} (RL) methods fail when faced with very sparse
reward signals.
Multi-task reinforcement learning attempts to solve this problem by
presenting agents with a diverse set of tasks and learn a task-dependent policy in the hope that
the agent could leverage knowledge from some tasks on
others~\citep{jaderberg2016reinforcement, hausman2018learning, Nagabandi2019LearningTA}. \emph{Multi-goal} reinforcement learning is a sub-field of multi-task RL,
where the different tasks consist in reaching particular \emph{goals} in the environment. 

\emph{Universal Value Function Approximators} (UVFA)
 \citep{pmlr-v37-schaul15} extend the classical Q-learning and Temporal
 Difference (TD) algorithms to the multi-goal setting. It learns
the goal-conditioned value-function $V^{\pi}(s, g)$ or $Q$-function
$Q^\ast(s,a,g)$ for every state-goal pair,
with function approximation, via a TD algorithm. Still, no learning
occurs until a reward is observed, 
and UVFA fails in many high dimensional environments, when the probability of reaching the target goal is low and the agent almost never gets any learning signal.

\emph{Hindsight Experience Replay} (HER)
\citep{HER} is a possible solution to this issue. It leverages
information between goals via the following principle: trajectories
aiming at a goal $g$ but reaching a goal $g'$ can be used for learning
exactly as if the trajectory had been aiming at $g'$ from start. This
strategy has proved successful in practice, but is known to be
\emph{biased}~\citep{manela2021bias, Lanka2018ARCHERAR}. In their request
for research for robotic multi-goal environments,
\cite{plappert2018multi} list the necessity for an unbiased version of
HER, as such bias can lead to low-return policies. 

Our contributions are:

\begin{itemize}
\item We show that in \emph{deterministic} environments, HER is actually
unbiased (Theorem~\ref{nfthm:her-determ}). This case covers many standard control or robotic environments, in which HER is known to perform well. This result strengthens HER theoretically.

\item We show that sparse rewards in a multi-goal setting can be handled,
counter-intuitively,
by dealing directly with the infinitely sparse reward limit: then the
sparse reward contribution can be computed algebraically instead of
sampled.
The resulting Q-learning and actor-critic algorithms are unbiased even in
the stochastic case and handle multi-goal RL without having to observe sparse rewards,
although their variance is higher than HER. For this, we fully formalize
the problem of multi-goal RL with infinitely sparse rewards.
\end{itemize}

\section{Multi-Goal Reinforcement Learning and Vanishing Rewards}
\label{sec:multi-goal-rl}

\paragraph{Definition.}
We define a multi-goal RL environment as a variant of a Markov decision
process (MDP) including a goal space. The
MDP is defined by a state-space $\mathcal{S}$, 
an action space $\mathcal{A}$ (discrete or continuous), a
discount factor $\gamma$, and a transition probability measure $P(\d
s'|s, a)$ which describes the probability that the next state is $s'$
after taking action $a$ in state $s$; for stochastic continuous
environments, this
is generally a continuous probability distribution over $s'$, hence the
notation
$\d s'$ which represents the probability to be in an infinitesimal set
$\d s'$
around $s'$.

The
goal space is a set $\mathcal{G}$ together with a function $\phi \from
\mathcal{S}\rightarrow\mathcal{G}$ defining for every state $s$ a
corresponding goal $g = \phi(s)$, which is the goal \emph{achieved} by
state $s$.
The objective of the agent is to \emph{reach} a goal $g$. This is usually
formalized by defining a reward function $R_{\varepsilon}(s, g)$ as $1$
when a given distance  between the achieved goal $\phi(s)$ and the target
$g$ is lower than a fixed value $\varepsilon$: $R_{\varepsilon}(s, g)
\deq \1_{\|\phi(s) - g\| \leq \varepsilon}$ for a fixed norm $\|.\|$ on $\mathcal{G}$.
Thus, each goal $g\in \mathcal{G}$ defines an ordinary MDP with reward
$R(s,g)$, and $Q$ and value functions $Q^\ast_\eps(s,a,g)$, $V^\pi_\eps(s,g)$.
A goal-conditioned policy $\pi(a|s, g)$ is a probability distribution over the action space $\mathcal{A}$ for every $(s, g) \in \mathcal{S}\times\mathcal{G}$. 

We assume that, for a multi-goal policy $\pi(a|s, g)$, we are able to
sample trajectories in the environment by sampling a goal $g\sim
\rho_{\mathcal{G}}(\d g)$, a starting state $s_{0}\sim \rho_{0}(\d
s_{0}|g)$, and then by sampling at step $t$ the action $a_{t}\sim \pi(a |
s_{t}, g)$ and the next state $s_{t+1}\sim P(\d s' |s_{t}, a_{t})$. We
use the notation $P^{\pi}(\d s'|s, g) \deq \int_{a}\pi(a|s, g)P(\d s'|s, a)$. 

\paragraph{Universal Value Function Approximations.}
\label{sec:uvf}

UVFA
\citep{pmlr-v37-schaul15} allow
for learning the value function $V_{\varepsilon}^{\pi}(s, g) =
\E_{a_{t}\sim \pi(.|s_{t}, g), s_{t+1}\sim P(.|s_{t},
a_{t})}\left[\sum_{t\geq 0}\gamma^{t} R_{\varepsilon}(s_{t},
g)|s_{0}=s\right]$ and the optimal $Q$-function
$Q^{\ast}_{\varepsilon}(s, a, g)$. Formally, $Q$-learning with UVFA can be defined
as
standard $Q$-learning on the augmented state space
$\tilde{\mathcal{S}} \deq \mathcal{S}\times\mathcal{G}$, with the
transition distribution $\tilde P$ defined as follows: if action $a$ is
performed in state $(s, g)$, the next state $\tilde s'$ is $(s', g)$ with
$s'\sim P(\d s'|s, a)$. The augmented environment is not a multi-goal
environment, and the policy $\pi(a|s, g) = \pi(a| \tilde s)$
becomes a standard non-goal-dependent policy in $\tilde{\mathcal{S}}$.
The UVFA Q-learning update
corresponds to standard parametric $Q$-learning on the augmented
environment.

In practice, we consider a parametric function $Q_{\theta}(s, g)$, and we want to learn $\theta$ such that $Q_{\theta}(s, g)$
approximates $Q^{\ast}_{\varepsilon}(s, g)$. If $\theta$ is our current
estimate and $Q_{\tar}$ a target $Q$-function the Q-learning UVFA stochastic
update $\widehat{\delta\theta}_{\uvfa}$ is defined as follows.  We
consider an exploration policy $\pi_{\expl}(a|s, g)$. When a transition
$(s, a, s', g)$ is observed, with $a\sim \pi_{\expl}(.|s, g)$ and $s'\sim
P(.|s, g)$,  $\widehat{\delta\theta}_{\uvfa}$ is:
\begin{align}
  \widehat{\delta\theta}_{\uvfa}(s, a, s', g) &\deq
  -\frac12\partial_{\theta}\left(Q_{\theta}(s, a, g) - R_{\varepsilon}(s,
  g) - \gamma \sup_{a'}Q_{{\tar}}(s', a', g)\right)^{2}
                           \label{eq:exp_update_uvf_q}
\end{align}
Then, we update $\theta$ with $\theta \leftarrow \theta + \eta
\widehat{\delta\theta}_{\uvfa}$, where $\eta$ is the learning rate. The
update $\widehat{\delta\theta}_{\uvfa}$ is an unbiased estimate of
$\nicefrac12 \partial_{\theta}\|Q_{\theta} - T\cdot
Q_{{\tar}}\|^{2}$ where $T$ is the optimal Bellman operator,
$T\cdot Q(s, a, g) = R_{\varepsilon}(s, g) + \gamma \E_{s'\sim P(.|s,
a)}\left[\sup_{a'}Q(s', a', g)\right]$, whose unique fixed point is
$Q^{\ast}_{\varepsilon}$. 
In particular, in the tabular setting, this
guarantees that a function $Q_{\infty}$ is a fixed point of UVFA if and only if $T\cdot Q_{\infty} = Q_{\infty}$, which means $Q_{\infty} = Q^{\ast}_{\varepsilon}$.

\paragraph{UVFA and vanishing rewards.}
\label{sec:uvf-snr}
A major problem with multi-goal setups is the low probability with which
each specific goal $g$ is achieved, since rewards are observed only in a
ball of radius $\eps$ around the goal. 
In a continuous noisy environment of dimension $n$, reaching a goal up to
precision $\eps$ becomes almost surely impossible when $\eps\to 0$. With
noise in dimension $n$, the probability to exactly reach a predefined
goal $g$ scales like $O(\eps^n)$. In particular, the $Q$ and value
functions vanish like $O(\eps^n)$ when $\eps$ is small.
The situation is different in continuous deterministic environments. If
it is possible to reach a goal exactly by selecting the right action,
then the optimal $Q$-function $Q^\ast_\eps$ does not vanish, even if
$\eps=0$.

With the UVFA update, the probability to observe a
reward $\1_{\|\phi(s) - g\| \leq \varepsilon}$
vanishes like
$O(\eps^n)$ for continuous exploration policies. So 
even if $Q^\ast_\eps$ itself does not vanish, the learning algorithm
for $Q^\ast_\eps$ may vanish.
In practice, in an
environment of dimension $n=6$, UVFA is not able to learn anymore
(experiment in Fig.~\ref{fig:experiments}).
This vanishing issue cannot be solved solely by an exploration
strategy: the issue is not the lack of diversity in visited
states but rather the state space is too large to be visited by an
exploration trajectory \citep{HER}. Solving the issue of sparse rewards 
requires gathering some
information even from \emph{failing trajectories} which do not reach
their initial goal, namely,
leveraging the structure of multi-goal environments by using that every
state achieves \emph{some} goal. This the case in HER but not UVFA.

In this work we study 
algorithms which leverage the multi-goal structure
and do not vanish even in the limit
$\eps\to 0$. We will focus on \emph{unbiased} algorithms, which ensure
that the true $Q$ or value function is indeed a fixed point, by
stochastic gradient arguments.
%
UVFA is unbiased but vanishes when $\eps\to 0$. HER does not vanish, but
is known to be biased. In
Section~\ref{sec:bias-her-stochastic} we prove that HER is unbiased in
deterministic environments.
Sections~\ref{sec:unbiased-dqn} and
\ref{sec:pg} present non-vanishing, unbiased algorithms for
stochastic environments; however, they are less efficient than HER in
deterministic environments.

\section{Hindsight Experience Replay in Stochastic or Deterministic Environments }


\emph{Hindsight Experience Replay} (HER) \citep{HER} is a way to 
solve the issue of sparse rewards for multi-goal environments by
leveraging the mutual information between goals. The principle is the
following: trajectories aiming at a goal $g$ but reaching a goal $g'$ can
be used for learning exactly as if the trajectory had been aiming for
$g'$ from start. Formally, when observing a trajectory $\tau = (g, s_{0},
a_{0}, s_{1}, a_{1}, ...)$, HER samples two random integers $0\leq K\leq L$, and
performs a $Q$-learning update at step $s_{K}$, but for a re-sampled goal
$g'$ that is, with some probability,  either $g'=g$ or $g' = \phi(s_{L})$, the goal achieved by
the $L$-th state in the trajectory:
  $\widehat{\delta\theta}_{\her}(\tau, K, L) \deq
  \frac12\partial_{\theta}\left(Q_{\theta}(s_{K}, a_{K}, g') -
  R_{\varepsilon}(s_{K}, g') - \gamma \sup_{a'}Q_{{\tar}}(s_{K+1}, a', g')\right)^{2}$. 
In particular, the HER update does not vanish even for $\eps=0$: with
nonzero probability, $K=L$ and $g'=\phi(s_L)$, so that
$R_{\varepsilon}(s_{K}, g')=1$.

\paragraph{Bias of HER in stochastic environments.}
\label{sec:bias-her-stochastic}
HER is known to be biased in a general setting~\citep{manela2021bias,
Lanka2018ARCHERAR, plappert2018multi}, and this bias corresponds to a
well-known psychological bias \citep{hindsight_bias}. Here is a simple
way to design \emph{counter-examples} environments which exhibit this HER
bias. Consider a finite multi goal environment and add a single action
$a^{\ast}$  which, from any state $s$, sends the agent to a uniform
random state $s'$ and then \emph{freezes} it, which means the agent will
always stay at $s'$.

Both in theory and practice, HER will learn to always select the action
$a^{\ast}$ (third plot in Fig.~\ref{fig:experiments}). The intuition is
the following: when the agent acts with $a^{\ast}$ and reaches a random
state $s'$, HER reinforces $a^{\ast}$ as a good way to reach $s'$ from
$s$, while this was purely random. Formally, the following statement
(proof in Appendix~\ref{app:proof-her-bias}) shows that HER will
overestimate the value of action $a^{\ast}$. We say that $Q_{\infty}$ is
a \emph{fixed point} of HER if
$\E_{\tau, K, L}\left[\widehat{\delta\theta}_{\her}(\tau, K, L)\right] =
0$ when $Q_{\theta} = Q_{\tar}=Q_{\infty}$.


\begin{thm}
\label{thm:freeze}
  Let $\mathcal{M}$ be any finite multi-goal environment, and
  $\tilde{\mathcal{M}}$ the modified environment with the \emph{freeze}
  action $a^{\ast}$. Then $\tilde{\mathcal{M}}$ is a
  \emph{counter-example} to HER, which is biased in this environment.
  Namely, if $Q_{\infty}$ is a fixed point of HER for
  $\tilde{\mathcal{M}}$, then for every \emph{unfrozen} state $s$ and
  goal $g$, $Q_{\infty}$ will \emph{overestimate} the value of
  $a^{\ast}$: $Q_{\infty}(s, a^{\ast}, g) > Q^{\ast}(s, a^{\ast}, g)$
  where $Q^\ast$ is the true value function.
\end{thm}

Generally, HER is overestimating chancy outcomes, by estimating that
any action (even random) that led to some goal was a good way to reach
that goal. This is clear in the example of the freeze-after-random-jump
actions in Theorem~\ref{thm:freeze}.
Thus, HER has no reason
to learn reliably in a stochastic environment. Other hindsight methods
such as  \citep{Rauber2019HindsightPG} experience a similar bias.

\paragraph{HER is unbiased in deterministic environments.}

Despite its bias, HER is efficient in practice, especially in continuous
control environments. We vindicate HER theoretically by showing that HER
is unbiased in \emph{deterministic} environments.
We say that an environment is \emph{deterministic} is the next state $s_{t+1}$ is uniquely determined by the current state $s_{t}$ and an action $a_{t}$. This covers many usual environments such as robotic environments.
\begin{thm}
  \label{nfthm:her-determ}
  In a deterministic multi-goal environment such that every target state
  is reachable from any starting state, HER is an unbiased Q-learning
  method. Namely, there is a Euclidean norm $\|.\|_{\her}$  such that if
  $Q_{\theta}$ is the current estimate of $Q^{\ast}$, the HER update
  $\widehat{\delta\theta_{\text{HER}}}$ is an unbiased stochastic
  estimate of the gradient step between $Q_{\theta}$ and the target
  function $TQ_{\tar}$:
  $\E\left[\widehat{\delta\theta}_{\text{HER}}\right]    =
  \nicefrac{1}{2}\partial_{\theta}\|Q_{\theta} - TQ_{\tar}\|_{\her}^{2}$. 
  In particular, the true $Q$-function $Q^\ast$ is a fixed point of HER in
  expectation: if $Q_\theta$ and $Q_\tar$ are equal to $Q^\ast$ then
  $\E\left[\widehat{\delta\theta}_{\text{HER}}\right]    =0$.
\end{thm}

The proof and a more detailed statement are given in Appendix~\ref{sec:convergence-her}. 
This result vindicates HER for deterministic environments: HER leverages
the structure of multi-goal environments, is not vanishing when the
rewards are sparse, and is mathematically well-grounded.


\section{Multi-Goal RL via Infinitely Sparse Rewards}
\label{sec:def_dirac_reward}

\subsection{Taking the Infinitely Sparse Reward Limit}

In Section~\ref{sec:bias-her-stochastic}, we saw that while HER is
well-founded in deterministic environments, it is biased in the
stochastic case and can learn low-return policies
(Figure~\ref{fig:experiments}). We now introduce unbiased methods for
multi-goal RL in the general setting, including stochastic environments.

In  continuous state spaces, the reward is usually defined as
$R_{\varepsilon}(s, g) =  \1_{\|\phi(s) - g\| \leq \varepsilon}$.  When
$\varepsilon \rightarrow 0$, the probability of reaching the reward with
a stochastic policy goes to $0$, and for any stochastic policy, the value
function $V^{\pi}_{\varepsilon}(s, g)$ converges to $0$ as well. To avoid
this vanishing issue, we need a scaling factor,  and consider the reward
$\frac{1}{\lambda(\varepsilon)}R_{\varepsilon}(s, g)$, with
$\lambda(\varepsilon)$ the volume of the ball of size $\varepsilon$ in
goal space.
When $\varepsilon \rightarrow 0$, this rescaled reward \emph{converges}
to the \emph{Dirac reward}:
\begin{equation}
  \label{eq:def_dirac_reward}
R(s, \d g) \deq \delta_{\phi(s)}(\d g),
\end{equation}
where $\delta_{x}$ is the Dirac measure at $x$. Intuitively, the Dirac reward $R(s, \d g)$ is infinite if the goal is reached ($\phi(s) = g$) and $0$ elsewhere. Formally, the reward is not a function but a \emph{measure} on the goal space $\mathcal{G}$ parametrized by the state $s$.

However, even after such a scaling, the UVFA update still vanishes with
high probability for small $\eps$ (this just scales things by
$1/\lambda(\eps)$).  We
will build algorithms that work directly in the limit $\varepsilon= 0$:
%
%
replacing the sparse reward $R_{\varepsilon}(s, g)$ by the
\emph{infinitely sparse} reward $R(s, \d g) = \delta_{\phi(s)}(\d g)$
will allow us to leverage the Dirac structure to remove the vanishing
rewards issue.

\paragraph{Computing the exact contribution of sparse rewards.}
We now explain how to leverage the multi-goal sparse reward structure.
The key idea is that, with $\eps=0$,
the contribution of the reward term in the Bellman equation can be computed exactly in
expectation. Infinitely sparse rewards can be treated algebraically.
This derivation is informal;
the formal proof is in Appendix~\ref{app:paramq}.

Let us start with the expectation of the UVFA update
\eqref{eq:exp_update_uvf_q} with $\eps>0$ and rewards rescaled by
$1/\lambda(\eps)$:
\begin{align*}
& \delta\theta_{\uvfa}  = \E_{s, a, s', g}\left[\widehat{\delta\theta}_{\uvfa}(s, a, s',
g)\right]
 \\&=  - \frac{1}{2}\partial_{\theta}\E_{{s, a, g, s'}}\left[\left(Q_{\theta}(s, a,
 g) - \frac{1}{\lambda(\varepsilon)}R_{\varepsilon}(s, g) - \gamma
 \max_{a'}Q_{{\tar}}(s',a', g)\right)^{2}\right]
  \\ &= \E_{{s, a, g}}\left[\partial_{\theta} Q_{\theta}(s, a,
  g)\frac{1}{\lambda(\varepsilon)}R_{\varepsilon}(s, g)\right]    -
  \E_{{s, a, g, s'}}\left[\partial_{\theta} Q_{\theta}(s, a,
  g)\left(Q_{\theta}(s, a, g) - \gamma \max_{a'}Q_{{\tar}}(s',a',
  g)\right)\right].
\end{align*}
This update cannot be used for small $\eps$, because $R_\eps(s,g)$ is $0$
most of the time, even though the expectation is nonzero and a huge
$1/\lambda(\eps)$
reward is observed with low probability.


But when $\eps\to 0$, the rescaled reward
$\frac{1}{\lambda(\varepsilon)}R_{\varepsilon}(s, g)$  converges to the
Dirac reward $\delta_{\phi(s)}$. Therefore, we can rewrite this first
term as
\begin{align*}
  \frac{1}{\lambda(\varepsilon)}\E_{{s, a, g}}\left[\partial_{\theta} Q_{\theta}(s, a, g)R_{\varepsilon}(s, g)\right] 
  &\rightarrow_{\varepsilon\rightarrow 0} \E_{{s, a, g}}\left[\partial_{\theta} Q_{\theta}(s, a, g)\delta_{\phi(s)}(\d g)\right] 
  \\&=  \E_{s, a}\left[\partial_{\theta} Q_{\theta}(s, a,
  \phi(s))\right].
\end{align*}
In this expression, sparse reward issues are avoided, just by taking the goal
$g=\phi(s)$ associated with the currently visited state $s$. Instead of
waiting to reach a goal to update the $Q$-function, this updates the
$Q$-function for the currently realized goal.

The resulting algorithm, $\delta$-DQN, is described in Theorem~\ref{thm:update_q} and
Algorithm~\ref{alg:example}. The
proper mathematical treatment (below and in the Appendix) of the Dirac
limit
shows that this actually estimates, not $Q$ itself, but
the \emph{density} $q(s,a,g)$ of the distribution of realized goals with
respect to  the
goal sampling distribution $\rho_{\mathcal{G}}(\d g)$ of the environment. 
This density $q$ can be used to rank actions (indeed, the scaling by
$\rho_{\mathcal{G}}$ between $q$ and $Q$ only depends on the goal $g$, so
for a fixed goal, states and actions are ranked the same way).
In general, working with
probability
densities is the only way that makes sense in the presence of noise,
as the probability to exactly reach a goal will be $0$.

A similar treatment holds for policy gradient
(Sections~\ref{sec:unbi-policy-eval}--\ref{sec:pg}).



\subsection{Unbiased Multi-Goal Q-learning with Infinitely Sparse Rewards}
\label{sec:unbiased-dqn}

Our goal here is to formally define multi-goal $Q$-learning with infinitely sparse rewards.  In
general, the probability of reaching any goal \emph{exactly} is $0$:
instead we will learn the probability \emph{distribution} of the goals
reached by a policy, and compute the probability to reach each infinitesimal element
$\d g$ in goal space.  This is done by treating everything as measures
over $\mathcal{G}$:
the reward $\delta_{\phi(s)}(\d g)$ is a measure, and the
value functions $V^{\pi}(s, \d g)$ or optimal action-value function
$Q^{\ast}(s, a, \d g)$ are measures on $\mathcal{G}$ as well. In the
following, we define these objects in detail, and show how to learn them in
practice.

First, we define an \emph{optimal Bellman operator} $T$ on \emph{action-value measures}, and the optimal action-value measure $Q^{\ast}(s, a, \d g)$. Then, we derive $\delta$-DQN, a deep Q-learning algorithm  with infinitely sparse rewards for multi-goal RL. 

\paragraph{Optimal Bellman equation and optimal Q-function.}


We first define $Q^{\ast}(s, a, \d g)$, the \emph{optimal action-value measure}, the mathematical object corresponding to the usual optimal $Q$-function $Q^{\ast}$ but infinitely sparse rewards. The following theorem defines the optimal Bellman operator for action-value measures, and $Q^{\ast}(s, a, \d g)$ as its fixed point. It is formallly stated in Appendix~\ref{app:optQ}.


\begin{defithm}
\label{nfthm:optQ_maintext}
Let $Q(s, a, \d g)$  a measure on $\mathcal{G}$ parametrized by
$s,a\in\mathcal{S}\times\mathcal{A}$.
We define  the optimal Bellman 
operator $T$ which sends $Q$ to $T\cdot Q$ with
\begin{equation}
(T\cdot Q)(s,a,\d g)\deq \delta_{\phi(s)}(\d g)+\gamma \,\E_{s'\sim P(\d
s'|s,a)} \sup_{a'} Q(s',a',\d g)
\end{equation}
where $\delta_{\phi(s)}$ is the Dirac measure at $\phi(s)\in \mathcal{G}$.
We define the \emph{optimal action-value measure} $Q^\ast$ as follows.
Set $Q_{0}(s, a, \d g) \deq 0$, and $Q_{n+1} \deq T Q_{n}$. Then
$Q_{n}(s, a, \d g)$ converges to some $Q^\ast(s,a, \d g)$. Moreover, this
$Q^\ast(s,a, \d g)$
solves the fixed point equation $TQ^\ast=Q^\ast$.
\end{defithm}

\paragraph{$Q$-learning with function approximations, with infinitely sparse rewards.}
From the fixed point equation for $Q^{\ast}$, 
 we would like to learn a model
of $Q^{\ast}(s, a, \d g)$ with function approximation. We will represent
measures over goals via their \emph{density} with respect to the goal
sampling function $\rho_{\mathcal{G}}$ of the environment. Namely,
we will approximate 
$Q^{\ast}(s, a, \d g)$ by a model
$Q_{\theta}(s, a, \d g) =
q_{\theta}(s, a, g)\rho_{\mathcal{G}}(\d g)$ where $
q_{\theta}(s, a, g)$ is an ordinary function, and learn $q_\theta$.
Hence, $q_{\theta}$ 
may be approximated by any parametric model, such as a
neural network.

The following theorem properly defines an unbiased stochastic $\ddqn$ update with infinitely sparse rewards for the density $q_\theta(s, a, g)$:
\begin{thm}
  \label{thm:update_q}
  Let $Q_{\theta}(s, a, \d g) = q_{\theta}(s, a, g)\rho_{\mathcal{G}}(\d
  g)$ be a current estimate of $Q^{\ast}(s, a, \d g)$. Let likewise 
  $Q_{\tar}(s, a, \d g) = q_{\tar}(s, a, g)\rho_{\mathcal{G}}(\d
    g)$ be
  a target $Q$-function, and consider the following update to bring
  $Q_\theta$ closer to $TQ_{\tar}$ with $T$ the optimal Bellman operator.
  
  Let $(s, a, s')$ be a sample of the environment such that $s' \sim P(s'|s, a)$ and $g\sim \rho_{\mathcal{G}}$ is sampled independently. Let $\widehat{\delta \theta}_{\ddqn}(s, a, s', g)$ be
\begin{equation}
  \label{eq:q_udpate_delta}
  \widehat{\delta\theta}_{\ddqn}(s, a, s', g) \deq \partial_\theta q_\theta(s,a,\phi(s))+\partial_\theta
q_\theta(s,a,g)\left(\gamma \max_{a'} q_{\tar}(s',a',g)-q_\theta(s,a,g)\right)
  \end{equation}
  Then $\widehat{\delta\theta}_{\ddqn}$ is an unbiased estimate of the Bellman error:
        $\E\left[\widehat{\delta\theta}_{\ddqn}\right] =
	\frac{1}{2}\partial_{\theta}\|Q_{\theta} - TQ_{\tar}\|^{2}$,
	where the Euclidean norm $\norm{\cdot}$ on measures is defined in
	Theorem~\ref{thm:paramq} (Appendix~\ref{app:paramq}).

	In particular, the true optimal state-action measure $Q^\ast$ is
	a fixed point of this update: if $Q_\theta=Q_{\tar}=Q^{\ast}$
	then $\E\left[\widehat{\delta\theta}_{\ddqn}\right] =0$.
\end{thm}

This update leads to $\delta$-DQN (Algorithm~\ref{alg:example}, which
corresponds to standard DQN with infinitely sparse rewards. For continuous actions, $\delta$-DQN can be modified similarly to DDPG \citep{Lillicrap2016ContinuousCW}. 

\paragraph{Example: the tabular case.}
The tabular case highlights the difference between UVFA and $\ddqn$. When
a transition $(s, a, s', g)$ is observed, the UVFA update is:
\begin{equation}
  \label{eq:tabular_uvfa}
  Q(s, a, g) \leftarrow Q(s, a, g) + \eta \left(\1_{\phi(s)=g}+ \gamma \max_{a'}Q(s', a', g)- Q(s, a, g)\right)  
\end{equation}
where $\eta$ is the learning rate. The only modified value is $Q(s, a,
g)$.

With $\ddqn$, we learn the density $q$ of $Q(s, a, \d g)$ with
respect to $\rho_{\mathcal{G}}$. Assume that $\rho_{\mathcal{G}}(g)$ is
the uniform measure over the finite goal space $\mathcal{G}$. Then we
learn $q(s, a, g) = |\mathcal{G}|\times Q(s, a, g)$. 
For a tabular model, the $\delta$-DQN
update in Equation~\eqref{eq:q_udpate_delta} is
\begin{align}
  \label{eq:tabularddqn}
  q(s, a, \phi(s)) &\leftarrow q(s, a, \phi(s)) + \eta
  \\ q(s, a, g) &\leftarrow q(s, a, g) + \eta \left(\gamma \max_{a'}q(s',
  a', g) - q(s, a, g)\right).
  \label{eq:tabularddqn2}
\end{align}
Here two values are updated: in addition to $(s,a,g)$, the trajectory
visiting $s$ is also used to update the value for the goal $\phi(s)$. The
first part always increases $q$ at the goal $\phi(s)$ achieved by $s$; the second
part at $(s,a,g)$ has no reward contribution, and decreases $q$ at
$(s,a,g)$ by a factor $(1-\eta)$ while propagating the value from $s'$.
In expectation, the decrease at $(s,a,g)$ compensates the increase at
$(s,a,\phi(s))$: this compensation is exact
when $q$ is the exact solution.

As a comparison, the tabular HER update works as follows: when observing
a trajectory $(g, s_{0}, a_{0}, s_{1}, ...)$, a transition $(s, a, s', g)
= (s_{K}, a_{K}, s_{K+1}, g)$ for some $K\geq 0$ is selected; then HER
samples $L \geq K$, defines $g' \deq \phi(s_{L})$ as the re-sampled goal,
then applies the UVFA update \eqref{eq:tabular_uvfa} but with $(s, a, s', g')$ instead of $(s, a, s', g)$.
When $L=K$, the goal sampled by HER is $g' = \phi(s)$: this is somewhat
similar to $\ddqn$, except $\ddqn$ resamples an independant goal instead of
$g'$ for the second term instead. Despite this similarity, HER is biased in stochastic environments and can converge to a low-return policy, while $\ddqn$ is unbiased.

\begin{algorithm}[tb]
   \caption{$\delta$-DQN}
   \label{alg:example}
\begin{algorithmic}
  \STATE  {\bfseries Input:} Randomly initialized model $q_{\theta}(s, a, g)$; $\phi$; exploration policy $\pi_{\expl}(a|s, g)$; goal function $\phi$; memory buffer $\mathtt{TransitionMemory}$, $T$ the maximum trajectory length
  \REPEAT
    \FOR{$K$ trajectories}
      \STATE Get a goal $g$ and an initial state $s_{0}$
      \FOR{$0\leq t \leq T$ steps do}
        \STATE Sample $a_{t} \sim \pi_{\text{expl}}(.|s_{t}, g)$, execute $a_{t}$ and observe $s_{t+1}$
        \STATE Store in the transition memory the transition $\mathtt{TransitionMemory} \leftarrow (s_{t}, a_{t}, s_{t+1})$ 
        \ENDFOR
      \FOR{$L$ gradient steps}
        \STATE{Sample $(s, a, s') \sim \mathtt{TransitionMemory}$ and $g \sim \rho_{\mathcal{G}}$}
        \STATE $\widehat{\delta\theta}_{\ddqn} \deq \partial_\theta q_\theta(s,a,\phi(s))+\partial_\theta
        q_\theta(s,a,g)\left(\gamma \max_{a'} q_\theta(s',a',g)-q_\theta(s,a,g)\right)$.
        \STATE Stochastic gradient step: $\theta \leftarrow \theta + \eta \widehat{\delta\theta}_{\ddqn} $.
      \ENDFOR
      \ENDFOR
  \UNTIL{end of learning}
\end{algorithmic}
\end{algorithm}

\subsection{Unbiased Policy Evaluation with Infinitely Sparse Rewards}
\label{sec:unbi-policy-eval}
Similarly to $\ddqn$, there exists an actor-critic  algorithm for multi-goal
environments with infinitely sparse rewards. 
We start with policy evaluation, then derive the policy gradient algorithm.

Learning the value function $V^\pi(s, \d g)$ directly without bias
poses technical issues due to the double dependency of $V^{\pi}(s, \d
g)$ on $g$ (first via the location of the reward, second, via
the goal-dependent policy $\pi(.|.,g)$). This is discussed in
Appendix~\ref{app:update-policy-eval}.
%

Instead,
we learn a richer object, $M^{\pi}(s, g, \d g')$, the value function of
$s$ if the reward is a Dirac at $g'$ but the agent follows the policy
$\pi(a|s,g)$ for goal $g$. This is defined as the measure over goals
\begin{equation}
  \label{eq:def_M}
M^{\pi}(s, g, \d g') \deq \E_{a_{t}\sim\pi(.|s_{t}, g), s_{t+1}\sim P(.| s_{t}, a_{t})}\left[\sum_{t\geq 0}\gamma^{t}\delta_{\phi(s_{t})}(\d g') | s_{0} = s\right]
\end{equation}
$M^{\pi}(s, g, \d g')$ represents the \emph{successor goal measure}, and is related to the successor state measure~\citep{Blier2021LearningSS}. Compared to $V^{\pi}(s, \d g)$, $M^{\pi}(s, g, \d g')$ splits the two effects of the goal $g$ in two variables $g$ and $g'$. $V^{\pi}(s, \d g)$ can be derived from $M^{\pi}(s, g, \d g')$ as $V^{\pi}(s, \d g) = M^{\pi}(s, g, \d g)$ (see Appendix~\ref{app:value-measure-continuous}). $M^{\pi}$ is a fixed point of the Bellman operator $T^{\pi}$ defined as:
\begin{equation}
  (T^{\pi}\cdot M)(s, g, \d g')  \deq \delta_{\phi(s)}(\d g') + \gamma \E_{a\sim \pi(a|s, g), s'\sim P(\d s'|s, a)}\left[M(s', g, \d g')\right]
\end{equation}
A rigorous proof of the existence of $M^{\pi}$ as well as its fixed point Bellman equation is given in Theorem~\ref{thm:app_update_m} in the supplementary.
Similarly to the $\ddqn$ update obtained in Theorem~\ref{thm:update_q},
we can now derive an unbiased $\dtd$ update for $M^{\pi}$, leveraging the
structure of the Dirac reward and removing the issue of vanishing
rewards. As for $Q^{\ast}(s, a, \d g)$, because $M^{\pi}(s, g, \d g')$ is
a measure, we learn  a model $m_{\theta}(s, g, g')$ of its density with
respect to $\rho_{\mathcal{G}}$, namely, $M_{\theta}(s, g, \d g') =
m_{\theta}(s, g, g')\rho_{\mathcal{G}}(\d g')$.

\begin{thm}
  \label{thm:update_m}
  Let $M_{\theta}(s, g, \d g') = m_{\theta}(s, g, g')\rho_{\mathcal{G}}(\d g')$ be a current estimate of $M^{\pi}(s, g, \d g')$. Let likewise $M_{\tar}(s, g, \d g') = m_{\tar}(s, g, g')\rho(\d g')$ be a target $M$, and consider the following update to bring $Q_{\theta}$ closer to $T^{\pi} Q_{\tar}$ with $T^{\pi}$ the Bellman operator.

Let $(s, a, s', g, g')$ be samples of the environment such that $a \sim \pi(a | s, g)$, $s' \sim P(s'|s, a)$ and $g'\sim\rho_{\mathcal{G}}$ is a goal sampled independently. Let $\widehat{\delta \theta}_{\dtd}$ be 
  \begin{equation}
    \label{eq:m_update_td}
    \widehat{\delta\theta}_{\dtd}(s, a, s', g, g') \deq \partial_{\theta}m_{\theta}(s, g, \phi(s)) + \partial_{\theta}m_{\theta}(s, g, g') \left(\gamma m_{\tar}(s', g, g') - m_{\theta}(s, g, g')   \right)
  \end{equation}
  Then $\widehat{\delta\theta}_{\dtd}$  is an unbiased  estimate of the Bellman error:
  $\E_{s, a, s', g, g'}\left[\widehat{\delta\theta}_{\dtd}(s, a, s', g, g')\right] = \frac{1}{2}\partial_{\theta}\|M_{\theta} - T^{\pi}M_{\tar}\|^{2}$, where the norm $\norm{\cdot}$ on measures is defined in Theorem \ref{thm:app_update_m} (Appendix~\ref{app:update-policy-eval}).

  In particular, the true successor goal measure $M^{\pi}(s, g, \d g')$ is a fixed point of this udpate: if $M_{\theta} = M_{\tar} = M^{\pi}$, then $\E\left[\widehat{\delta\theta}_{\dtd}\right]=0$.
    \end{thm}

Similarly to update $\widehat{\delta\theta}_{\ddqn}$, the update $\widehat{\delta\theta}_{\dtd}$ has two parts: the first part $\partial_{\theta}m_{\theta}(s, g, \phi(s))$ represents the reward update for the goal achieved in the current state $s$, and removes the vanishing reward issue. The second part propagates the rewards along transitions.

   We can also define a horizon-$n$ $\dtdn$ update if we have access to
   longer sub-trajectories $\tau = (g, s_{0},a_{0}, s_{1}, ...)$. The
   update at a state $s_k$ in the trajectory is (Appendix,
   Theorem~\ref{thm:app_update_m}) \begin{equation} \label{eq:m_update_tdn}
   \widehat{\delta\theta}_{\dtdn}(\tau, k, g') \deq
   \sum_{l=0}^{n-1}\gamma^{l}\partial_{\theta}m_{\theta}(s_{k}, g,
   \phi(s_{k+l})) +   \partial_{\theta}m_{\theta}(s_{k}, g, g')
   \left(\gamma^{n} m_{\theta}(s_{k+n}, g, g') - m_{\theta}(s_{k}, g, g')
   \right)  \end{equation} where $g'\sim \rho_{\mathcal{G}}$ is sampled
   independently.
The first part increases the value estimate at state $s_{k}$ for every of the $n$
goals $\phi(s_{k}), ..., \phi(s_{k+n-1})$ achieved in the next $n$ steps:
this
corresponds to  the $n$-step  return with Dirac rewards. The second
part propagates the value along transitions. This is similar to HER in
that future goals achieved along the trajectory are explicitly used, and
could thus improve sample efficiency.
However, computational complexity is an issue. In non-multi-goal
environments, algorithms such as PPO
\citep{Schulman2017ProximalPO} compute the
TD($n$) update at every step of the trajectory. This is computable with $O(n)$
forward passes through the value model $v_\theta$, because it only
requires to compute $v_\theta(s_0),\ldots,v_\theta(s_n)$.
Here we have to compute $m_\theta(s_k,g,\phi(s_{k+l}))$ for every $k$ and
$l$, leading to an $O(n^2)$ complexity (though this could potentially be
sub-sampled as in HER).
This makes it slow in practice, and
$\dtdn$ was not tested experimentally here.

\subsection{Multi-Goal Policy Gradient}
\label{sec:pg}

We now derive the actor-critic algorithm. 
The classical approach with reward $R_\eps$ for $\eps>0$ considers
the expected return 
$J_{\varepsilon}(\pi) = \E_{g\sim p_{\mathcal{G}}, s_{0}\sim
p_{0}(.|g)}\left[\sum_{t\geq 0}\gamma^{t}R_{\varepsilon}(s_{t}, g) |
s_{0} = s\right] = \int_{s_{0}, g}V^{\pi}_{\varepsilon}(s, g)\,\rho_{\mathcal{G}}(\d g)\rho_{0}(\d s_{0}|g)$
with a sampled goal $g\sim \rho_{\mathcal{G}}(\d g)$ and sampled initial
state $s_{0}\sim \rho_{0}(\d s_{0}|g)$, and subsequent actions sampled
from the policy for $g$. As in $\ddqn$, we want to derive an algorithm
solving the vanishing reward issue directly for $\eps=0$. We first show
that the limit makes sense, then derive the corresponding update in
\eqref{eq:update_pi} below.

\begin{thm}
  \label{thm:eq_eps_delta_main}
  Under \emph{continuity assumptions} (Assumption~\ref{assumption:density} in the Appendix), there is a function $J(\pi)$   such that, for every parametric policy $\pi_{\theta}(a|s, g)$: 
  \begin{align}
    \frac{1}{\lambda(\varepsilon)}J_{\varepsilon}(\pi_{\theta})\rightarrow_{\varepsilon\rightarrow 0} J(\pi_{\theta})  \quad~\quad \text{ and } \quad~\quad   \frac{1}{\lambda(\varepsilon)}\partial_{\theta}J_{\varepsilon}(\pi_{\theta})\rightarrow_{\varepsilon\rightarrow 0} \partial_{\theta}J(\pi_{\theta})
  \end{align}
 where $\lambda(\varepsilon)$ is the volume of a ball of size
 $\varepsilon$ in goal space. We call $J(\pi)$ the expected return with
 infinitely sparse rewards. Moreover,
$
   J(\pi) \deq   \int_{s_{0}, g}  V^{\pi}(s_0,\d g)
   \,p_{\mathcal{G}}(g)\,\rho_{0}(\d s_{0}|g)
$
  where $p_{\mathcal{G}}(g)$ is the density of $\rho_{\mathcal{G}}(\d g)$
  with respect to Lebesgue measure on goals.
\end{thm}

We now derive an estimate of $\partial_{\theta}J(\pi_{\theta})$ for a
parametric policy  $\pi_\theta(a|s, g)$. We assume access to transition
samples $(s, a, s', g)$ such that $a\sim \pi(.|s, g)$, $s'\sim P(\d s'|s,
a)$ and $s$ is sampled from the goal-dependant \emph{discounted
visitation frequencies} $\nu^{\pi}(\d s|g) = (1-\gamma)\sum_{t\geq
0}\gamma^{t}\rho_{0}(\d s_{0}|g)(P^{\pi})^{t}(\d s|s_{0}, g)$: namely,
states $s$
on a trajectory sampled from $\pi$ with goal $g$. 

We can define the actor critic update with infinitely sparse rewards 
by using the model $m(s,g,g)$ as an estimate of the values, and applying
the ordinary policy gradient theorem \citep{sutton2018reinforcement} on the extended space
$\mathcal{S}\times \mathcal{G}$ to include the goals (see
Appendix~\ref{app:policy-gradient}). This leads to
\begin{equation}
\label{eq:update_pi}
\widehat{\delta\theta}_{\dac}(s, a, s', g) \deq \partial_{\theta}\log \pi_{\theta}(a|s, g)\left(\gamma m_{\theta_{M}}(s', g, g) - m_{\theta_{M}}(s, g, g)\right)
\end{equation}
where $m_{\theta_{M}}(s, g, g')$ is the model of the value density
learned in Section~\ref{sec:unbi-policy-eval}.
%
This is justified by 
the following statement, which is an informal version of
Theorem~\ref{thm:pol_grad} in Appendix~\ref{app:policy-gradient}: namely,
if the value function model
$m_{\theta_M}$ is correct, then 
this actor-critic update is an unbiased estimate of
$\partial_{\theta}J(\pi_{\theta})$.

\begin{nfthm}
  \label{nfthm:pol_grad}
  If $m_{\theta_{M}}(s, g, g)\lambda(\d g)$ approximates $V^{\pi}(s, \d g)$ as a measure, then $\E_{s, a, s', g}\left[\widehat{\delta\theta}_{\dac}(s, a, s', g)\right]$ approximates $\partial_{\theta}J(\pi_{\theta})$.
\end{nfthm}

This update, together with the one for $m$ in
Theorem~\ref{thm:update_m}, make up the
$\delta$-Actor-Critic algorithm (Algorithm~\ref{alg:ac}).
We can similarly define a PPO algorithm
(Appendix~\ref{app:experiments-details}),
used 
in the experiments.

\begin{algorithm}[tb]
   \caption{One-step $\delta$-Actor-Critic}
   \label{alg:ac}
\begin{algorithmic}
  \STATE {\bfseries Input:} Model $m_{\theta_{M}}(s, g)$; policy $\pi_{\theta}$; goal function $\phi$; $T$ the maximum trajectory length
  \STATE Get a goal $g$ and an initial state $s_{0}$ from the environment
  
    \FOR{$0\leq t \leq T$ steps} 
      \STATE Sample $a_{t} \sim \pi(a|s_{t}, g)$
      \STATE Execute action $a_{t}$ and observe the next state $s_{t+1}$
      \STATE Sample an independent goal $g' \sim \rho_{\mathcal{G}}(\d g')$
      \STATE $\widehat{\delta\theta}_{\dtd} \deq \partial_{\theta}m_{\theta_{M}}(s_{t}, g, \phi(s_{t})) + \partial_{\theta}m_{\theta_{M}}(s_{t}, g, g') \left(\gamma m_{\theta_{M}}(s_{t+1}, g, g') - m_{\theta_{M}}(s_{t}, g, g')   \right)$
      \STATE  $\widehat{\delta\theta}_{\dac} = \gamma^{t} \times \partial_{\theta}\log\pi_{\theta_{\pi}}(a_{t}|s_{t}, g)\left(\gamma m(s_{t+1}, g, g) - m(s_{t}, g, g)   \right)$
      \STATE $\theta_{M} \leftarrow \theta_{M} + \eta_{M}\widehat{\delta\theta}_{\dtd}$
      \STATE $\theta_{\pi} \leftarrow \theta_{\pi} + \eta_{\pi}\widehat{\delta\theta}_{\dac}$
    \ENDFOR
\end{algorithmic}
\end{algorithm}

\section{Experiments}
\label{sec:experiments}


\paragraph{The \texttt{Torus} environment.}
We first define  the \texttt{Torus$(n)$} environment, which is a
continuous version of the \emph{flipping coin} environment introduced in
~\citep{HER}. The state space is the $n$-th dimensional torus,
represented as $\mathcal{S} = [0, 1)^{n}$, and can be obtained from the
$n$-dimensional hypercube by gluing the opposite faces together.
The action space is $\mathcal{A} = \{1, \ldots, n\}\times \{-\alpha, \alpha\}$ and action $a=(i, u)$ in state $s$ moves the position on the axis $i$ of a quantity  $u$, then the environment adds a Gaussian noise. Formally $s' \sim \left((s +u.e_{i} + \mathcal{N}(0, \sigma^{2})) \mod 1\right)$, where $(e_{j})_{1\leq j\leq n}$ is the canonical basis $(e_{i})_{k} = \1_{i=k}$. 
We consider the environment in dimensions $n=4$ and $n=6$. We also
consider the modified environment with the \emph{freeze} action described
in Section~\ref{sec:bias-her-stochastic}. For every environment, we
observe trajectories of length 200, and the reported metric is the
rescaled negative L$1$ distance to the goal at the end of trajectory
$-\frac{1}{n}\|s-g\|_{1}$.
The experimental details are in Appendix~\ref{app:experiments-details}.

We compare UVFA, HER, $\ddqn$, and $\delta$-PPO
(defined in Appendix~\ref{app:experiments-details} based on $\dac$). Each
algorithm fails in some environment: additional experiments in the
Appendix show that
$\delta$-DQN and $\delta$-PPO are both failing to
learn when the dimension of  the torus increases, while HER is still able
to learn.
This is discussed in Section~\ref{sec:discussion}.
While UVFA, HER and $\ddqn$ are similar algorithms and can be compared
as actor-critic methods handle the trajectory samples in a different way from
$Q$-learning methods.
Still, we observe that $\delta$-PPO learns successfully in the same environments as $\delta$-DQN, and also failing  when $\delta$-DQN does.

\paragraph{The FetchReach environment.}

The FetchReach environment \citep{plappert2018multi} is a robotic arm
environment in which the objective is for the extremity of the arm to
reach a given 3D position. The environment is deterministic, so HER is expected to perform well.
Here, all methods learn successfully.
We also experimented $\ddqn$ and $\delta$-PPO on more complex
environments of the same robotic suite, such as FetchPush, but both
methods fail in this setting, while HER was successful. 

\begin{figure}[t]
  \includegraphics[width=\textwidth]{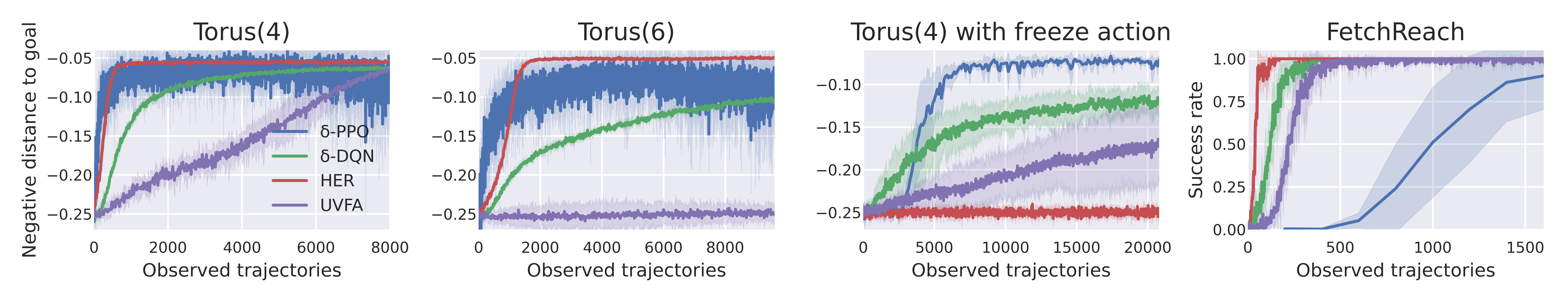}
  \caption{We compare UVFA, HER, $\delta$-DQN in toy environments. We observe different regimes: with a highly stochastic environment (\texttt{Torus} with freeze action), HER is unable to learn because of its bias, whereas UVFA and $\ddqn$ are. When the state dimension becomes too large (\texttt{Torus$(6)$}), UVFA is unable to learn because of the vanishing reward issue. In environments in which HER is able to learn, it is the most efficient method, and $\ddqn$ is always performing better than UVFA.}
  \label{fig:experiments}
\end{figure}

\section{Limitations and Future Work}
\label{sec:discussion}

The algorithms using infinitely sparse rewards always perform better than
UVFA, and perform better than HER in environments designed to exhibit the
HER bias issue. But they do not perform as well as HER in some standard
environments, and are unable to learn at all in more complex environments such as FetchPush. We discuss two technical limitations of $\delta$-DQN and
$\delta$-Actor-Critic.

The first issue is the function approximation. Learning the models
$Q_{\theta}(s, a, \d g) = q_{\theta}(s, a, g)\rho(\d g)$ of $Q^{\ast}$
and $M_{\theta}^{\pi}(s, g_{1}, \d g_{2}) = m_{\theta}(s, g_{1},
g_{2})\rho(\d g_{2})$ of $M^{\pi}$ requires approximating a Dirac
distribution (when $g_{2} = \phi(s)$) with a continuous density. The
theorems justify this, but in practice the functions $m_{\theta}$ and
$q_{\theta}$ have to reach multiple orders of magnitude (high values
close to the goal, low everywhere else), and the values need to be
accurate in these two regimes. Representing multiple orders of magnitude
in
neural networks may require a well-suited family of parametric functions.

A second issue is variance. The Dirac 
rewards remove the \emph{infinite} variance of vanishing rewards in
UVFA when $\varepsilon \rightarrow 0$. But the variance of the
remaining term can be high. Consider the tabular case
\eqref{eq:tabularddqn}--\eqref{eq:tabularddqn2}: $\delta$-DQN learns
significantly faster than UVFA on the \emph{diagonal} $Q(s, a, g)$ when
$g=s$, thanks to the Diracs. 
But this does not change the way the reward
is propagated to other states, due to the independent sampling of $g$ in
\eqref{eq:tabularddqn2}.
Selecting goals $g$ more
correlated to the state $s$ as in HER could also be helpful, but this is
not obvious to do without re-introducing HER-style bias.


\section{Conclusion}

We have proved that there exist unbiased goal-oriented RL algorithms which
do not vanish when rewards become sparse: it is possible to
deal with sparse rewards in RL directly via the infinitely sparse reward
limit, although this does not solve all variance issues. We have also proved that another multi-goal method, HER, is
unbiased and has the correct fixed point in all deterministic
environments.


\section*{Acknowledgments}

We would like to thank Ahmed Touati for his technical help, and Corentin Tallec, Alessandro Lazaric, Nicolas Usunier and Jonathan Laurent for their helpful comments and advice.
\bibliographystyle{icml2021}
\bibliography{biblio}

\vfill

\begin{enumerate}

\item For all authors...
\begin{enumerate}
  \item Do the main claims made in the abstract and introduction accurately reflect the paper's contributions and scope?
    {\answerYes{}}
  \item Did you describe the limitations of your work?
    {\answerYes{}}
  \item Did you discuss any potential negative societal impacts of your work?
    {\answerNA{}}
  \item Have you read the ethics review guidelines and ensured that your paper conforms to them?
    {\answerYes{}}
\end{enumerate}

\item If you are including theoretical results...
\begin{enumerate}
  \item Did you state the full set of assumptions of all theoretical results?
    {\answerYes{}}
	\item Did you include complete proofs of all theoretical results?
    {\answerYes{}}
\end{enumerate}

\item If you ran experiments...
\begin{enumerate}
  \item Did you include the code, data, and instructions needed to reproduce the main experimental results (either in the supplemental material or as a URL)?
    {\answerYes{}}
  \item Did you specify all the training details (e.g., data splits, hyperparameters, how they were chosen)?
    {\answerYes{}}
	\item Did you report error bars (e.g., with respect to the random seed after running experiments multiple times)?
    {\answerYes{}}
	\item Did you include the total amount of compute and the type of resources used (e.g., type of GPUs, internal cluster, or cloud provider)?
    {\answerYes{}}
\end{enumerate}

\item If you are using existing assets (e.g., code, data, models) or curating/releasing new assets...
\begin{enumerate}
  \item If your work uses existing assets, did you cite the creators?
    {\answerYes{}}
  \item Did you mention the license of the assets?
    {\answerYes{}}
  \item Did you include any new assets either in the supplemental material or as a URL?
    {\answerYes{}}
  \item Did you discuss whether and how consent was obtained from people whose data you're using/curating?
   \answerNA{}
  \item Did you discuss whether the data you are using/curating contains personally identifiable information or offensive content?
    \answerNA{}
\end{enumerate}

\item If you used crowdsourcing or conducted research with human subjects...
\begin{enumerate}
  \item Did you include the full text of instructions given to participants and screenshots, if applicable?
    \answerNA{}
  \item Did you describe any potential participant risks, with links to Institutional Review Board (IRB) approvals, if applicable?
    \answerNA{}
  \item Did you include the estimated hourly wage paid to participants and the total amount spent on participant compensation?
    \answerNA{}
\end{enumerate}

\end{enumerate}

\vfill

\pagebreak

\appendix

  \begin{longtable}[h]{p{.2\textwidth}p{.8\textwidth}} \toprule
    \textbf{Notation} & \textbf{Definition}
    \\    \midrule[.2em]
    \endfirsthead
    \toprule
    \emph{Continued from previous page} \\
    \midrule
    \textbf{Notation} & \textbf{Definition}
    \\    \midrule[.2em]
    \endhead
    \midrule
    \emph{Continuing on next page...}
    \\ \bottomrule
    \endfoot
    \bottomrule
    \endlastfoot
    $\mathcal{S}$ & State space \\
    $\mathcal{A}$ & Action space \\
    $\gamma$ & Discount factor, $ 0 \leq \gamma < 1$ \\
    $P(\d s' | s, a)$ & Transition probability measure, over $\mathcal{S}$, for every $s, a \in \mathcal{S}\times\mathcal{A}$. \\
    $\mathcal{G}$ & Goal space \\
    $n$ & Dimension of the goal space \\
    $\phi$ & Goal function $\phi : \mathcal{S} \rightarrow \mathcal{G}$. The goal $\phi(s)$ is the goal \emph{achieved} in $s$. \\
    $\rho_{\mathcal{G}}(\d g)$ & Goal sampling distribution. \\
    $p_{\mathcal{G}}(g)$ & If it exists, the density of $\rho_{\mathcal{G}}$ with respect to $\lambda$.\\
    $\rho_{0}(\d s_{0}|g)$ & Initial state sampling distribution \\
    $p_{0}(s_{0}|g)$ & If it exists, the density of $\rho_{0}$ with respect to $\lambda$.\\
    $\lambda(\d.)$ & Lebesgue measure \\
    $\varepsilon$ & Threshold for the sparse reward. $\varepsilon > 0$ \\
    $\1_{A}(x)$ & Function equal to $1$ if $x\in A$, and $0$ is $x \notin A$.\\
    $R_{\varepsilon}(s, g)$ & Sparse reward around goal $g$: $R_{\varepsilon}(s, g) = \1_{\|\phi(s)-g\|\leq\varepsilon}(s))$ \\
    $\lambda(\varepsilon)$ & Volume of a sphere of radius $\varepsilon$: $\lambda(\varepsilon) = \lambda(\{x \quad \text{s.t.} \quad \|x\|\leq\varepsilon\})$ \\
    $\pi$ & Goal dependent policy. If $\mathcal{A}$ is discrete, $\pi(a|s, g)$ is the probability of selecting action $a$. If $\mathcal{A}$ is continuous, it is the density of selecting $a$ with respect to Lebesgue measure. \\
    $P^{\pi}(\d s' | s, g)$ & Transition probability measure for policy $\pi$:
                              \\ & $P^{\pi}(\d s' | s, g) = \int_{a}\pi(a|s, g)P(\d s'|s, a)$. \\
    $\tau$ & Trajectory: $\tau = (g, s_{0}, a_{0}, s_{1}, ...)$, with $g\sim \rho_{\mathcal{G}}$, $s_0\sim \rho_{0}(.|g)$, \\ & $a_{t}\sim\pi(.|s_{t}, g)$, $s_{t+1}\sim P(.|s_{t}, a_{t})$.\\
    $V^{\pi}_{\varepsilon}(s, g)$ & Value function for reward $\varepsilon$: \\ & $V^{\pi}_{\varepsilon}(s, g) = \E_{s_{t+1}\sim P^{\pi}(.|s_{t}, g)}\left[\sum_{t\geq 0}\gamma^{t}R_{\varepsilon}(s_{t}, g) | s_{0}=s\right]$ \\
    $Q^{\pi}_{\varepsilon}(s, a, g)$ & Action-value function for reward $\varepsilon$: \\ & $Q^{\pi}_{\varepsilon}(s, a, g) = \E_{a_{t}\sim \pi(.|s_{t}, g), s_{t+1}\sim P(.|s_{t}, a_{t})}\left[\sum_{t\geq 0}\gamma^{t}R_{\varepsilon}(s_{t}, g) | s_{0}=s, a_{0}=a\right]$ \\
   $\pi^{\ast}$ & Optimal policy \\
    $Q^{\ast}_{\varepsilon}(s, a, g)$ & Optimal action-value function for reward $R_{\varepsilon}$: $Q^{\ast} = Q^{\pi^{\ast}}$. \\
    $\tilde{\mathcal{S}}, \tilde P$ & Augmented MDP: $\tilde{\mathcal{S}} = \mathcal{S}\times\mathcal{G}$ and for every $\tilde s = (s, g)$, action $a$, next state $\tilde s$ is sampled as $\tilde s = (s', g)$ where $s'\sim P(\d s'|s, a)$. \\
    $Q_{\theta}$, $V_\theta$ & Models of $Q^{\ast}$, $V^{\pi}$ parametrized by $\theta$. \\
    $Q_{\tar}$, $V_{\tar}$ & Target values \\
    $Q_{\infty}$ & Fixed point of an algorithm  \\
    $\pi_{\expl}$ & Exploration policy \\
    $T$ & Optimal Bellman operator, defined for function $Q(s, a, g)$ or measures $Q(s, a, \d g)$ \\
    ~ & $\bullet$ Functions: $(T\cdot Q)(s, a, g) = R_{\varepsilon}(s, g) + \gamma \E_{s'\sim P(.|s, a)}\left[\sup_{a'}Q(s',a', g)\right]$.\\
    ~ & $\bullet$ Measures: $(T\cdot Q)(s, a, \d g) = \delta_{\phi(s)}(\d g) + \gamma \E_{s'\sim P(.|s, a)}\left[\sup_{a'}Q(s',a', \d g)\right]$.\\
    $\widehat{\delta\theta}_{\uvfa}(s, a, s', g)$ & Stochastic \emph{Universal Value Function Approximators} update for $Q$-learning \\
    \\ \midrule
    \multicolumn{2}{c}{\textbf{Hindsight Experience Replay}} \\
    $K$ & Step $s_{K}$ of the update for a trajectory $\tau = (g, s_{0}, a_{0}, s_{1}, ...)$ \\
    $g'$ & Re-sampled goal by HER \\
    $L$ & Step of the resampling goal: $g' = \phi(s_{L})$. \\
    $\widehat{\delta\theta}_{\her}(\tau, K, L)$ & HER stochastic update \\
    $\|.\|_{\her}$ & Norm such that $\widehat{\delta\theta}_{\her}(\tau, K, L)$ is an unbiased estimate of Bellman error \\
    $a^{\ast}$ & freeze-after-random-jump additional action \\
    \\ \midrule
    \pagebreak
    \multicolumn{2}{c}{\textbf{Infinitely Sparse Rewards}} \\
    $\delta_{x}(\d x')$ & Dirac measure located in $g$ \\
    $R(s, \d g)$ & Infinitely spare Dirac reward: $R(s, \d g) = \delta_{\phi(s)}(\d g)$.\\
    $Q^{\ast}(s, a, \d g)$ & Optimal action-value \emph{measure} \\
    $M^{\pi}(s, g, \d g')$ & Successor goal \emph{measure}:  \\ & $M^{\pi}(s, g, \d g') = \E_{a_{t}\sim\pi(.|s_{t}, g), s_{t+1}\sim P(.| s_{t}, a_{t})}\left[\sum_{t\geq 0}\gamma^{t}\delta_{\phi(s_{t})}(\d g') | s_{0} = s\right]$ \\
    $V^{\pi}(s, \d g)$ & Value \emph{measure}: $V^{\pi}(s, \d g) = M^{\pi}(s, g, \d g)$ \\
    $q_{\theta}(s, a, g)$ & Model of the density of $Q^{\ast}(s, a, \d g)$ with respect to $\rho_{\mathcal{G}}$ parametrized by $\theta$ \\
    $Q_{\theta}(s, a, \d g)$ & Model of $Q^{\ast}(s, a, \d g)$ defined via its density: $Q_{\theta}(s, a, \d g) = q_{\theta}(s, a, g)\rho_{\mathcal{G}}(\d g)$\\
    $q_{\tar}$ & Target values \\
    $\widehat{\delta\theta}_{\ddqn}(s, a, s', g)$ & Stochastic update of $q_{\theta}(s, a, g)$ for $\ddqn$ \\
    $\eta$ & Learning rate\\
    $T^\pi$ & Bellman operator: \\ & $(T^{\pi}\cdot M)(s, g, \d g')= \delta_{\phi(s)}(\d g') + \gamma\E_{s'\sim P^{\pi}(.|s, g)}\left[M^{\pi}(s', g, \d g')\right]$. \\
    $m_{\theta}(s, g, g')$ & Model of the density of $M^{\pi}(s, g, \d g')$ with respect to $\rho_{\mathcal{G}}(\d g')$ parametrized by $\theta$ \\
    $M_{\theta}(s, g, \d g')$ & Model of $M^{\pi}(s, g, \d g')$ defined via its density: \\ & $M_{\theta}(s, g, \d g') = m_{\theta}(s, g, g')\rho_{\mathcal{G}}(\d g)$\\
    $\widehat{\delta\theta}_{\dtd}(s, a, s', g, g')$ & Stochastic update of $m_{\theta}(s, g, g')$ for $\dtd$\\
    $\widehat{\delta\theta}_{\dtdn}(\tau, k, g')$ & Stochastic update of $m_{\theta}(s, g, g')$ for $\dtdn$\\
    $J_{\varepsilon}(\pi)$ & Expected return $J_{\varepsilon}(\pi) = \E_{g\sim\rho_{\mathcal{G}}, s_{0}\sim \rho_{0}(.|g)}\left[\sum_{t\geq0}\gamma^{t}R_{\varepsilon}(s_{t}, g)|s_{0}=s\right]$\\
    $J(\pi)$ & Expected return with infinitely sparse rewards \\
    $\nu^{\pi}(\d s|g, s_{0})$ & Discounted visitation frequencies: \\ & $\nu^{\pi}(\d s| s_{0}, g) = (1-\gamma)\sum_{t\geq 0}\gamma^{t}(P^{\pi})^{t}(\d s|s_{0}, g)$ \\
    $\theta_{M}$ & In policy gradient, parameter of the critic $m_{\theta_{M}}(s, g, g')$ \\
    $\widehat{\delta\theta}_{\dac}(s, a, s', g)$ & Stochastic update for $\dac$ \\
    \bottomrule
    \caption{Notation in the main text}
    \label{tab:notation}
  \end{longtable}

\section{Experiments Details}
\label{app:experiments-details}

In this section, we present the experiment details of Section~\ref{sec:experiments}. Every experiment was performed on a single GPU.


\paragraph{The \texttt{Torus(n)} environment} The state space of the \texttt{Torus$(n)$} environment is the $n$-th dimensional torus, $\mathcal{S} = [0, 1)^{n}$, and can be obtained from the
$n$-dimensional hypercube by gluing the opposite faces together. If the current state is $s = (s_{1}, ..., s_{n})$, we define the observation of the agent as $(\cos(2\pi s_{1}), ..., \cos(2\pi s_{n}), \sin(2\pi s_{1}), ..., \sin(2\pi s_{n}))\in [-1, 1]^{2n}$. We use this representation in order to remove the discontinuity of the representation $[0, 1)^{n}$. This representation contains all the information of the state $s$ and the environment is still fully observable (and not partially observable). The action space is $\mathcal{A} = \{1, \ldots, n\}\times \{-\alpha, \alpha\}$ and action $a=(i, u)$ in state $s$ moves the position on the axis $i$ of a quantity  $u$, then the environment adds a Gaussian noise. Formally $s' \sim \left((s +u.e_{i} + \mathcal{N}(0, \sigma^{2})) \mod 1\right)$, where $(e_{j})_{1\leq j\leq n}$ is the canonical basis $(e_{i})_{k} = \1_{i=k}$. In practice, we take $\alpha = 0.1$, and $\sigma = \frac{0.1}{n}$. The reward is $R_{\varepsilon}(s, g) = \1_{\|s-g\|\leq \varepsilon}$ where $\|.\|$ is the rescaled  L$1$ distance in the Torus: 
$\|s-g\| = \frac{1}{n}\sum_{i=1}^{n}\min((s_{i}-g_{i}) \mod 1, |((s_{i}-g_{i}) \mod 1) - 1|)$. In practice, we use $\varepsilon = 0.05$. At the beginning of an episode, we sample a goal uniformly in the environment, then we observe trajectories of length $200$. We set $\gamma = .995$.

\paragraph{FetchReach} FetchReach is a standard environment from~\cite{plappert2018multi}. The objective is to reach a \emph{goal} position in 3 dimension with the end of the robotic arm. The observation space $\mathcal{S}$ is of dimension $10$ and contains positions and velocities, such that the environment is Markov, fully observable, and deterministic. The action space $\mathcal{A}$ is continuous and of dimension $4$. The goal space $\mathcal{G}$ is of dimension $3$, and the goal represent the position of the end of the robotic arm. Trajectories are of length $50$.

\paragraph{Q-learning experiments}

Here we describe experiments with UVFA, HER and $\ddqn$, which have similar structure. For every algorithm, we use the same neural network to learn $Q_{\theta}(s, a, g)$ or $q_{\theta}(s, a, g)$. Simlarly to DDPG~\citep{Lillicrap2016ContinuousCW}, if the action space $\mathcal{A}$ is continuous, we additionally learn a deterministic policy $\pi_{\theta}: \mathcal{S}\times\mathcal{G} \rightarrow \mathcal{A}$.  We use a dueling architecture \citep{Wang2016DuelingNA}: we learn a \emph{value} network $v_{\theta}(s, g)$ and an \emph{advantage} network $\adv_{\theta}(s, a, g)$. We then define $q_{\theta}(s, a, g) = v_{\theta}(s, g) + \widetilde \adv_{\theta}(s, a, g)$, where $\widetilde \adv_{\theta}(s, a, g)$ is the \emph{rescaled} advantage, and is defined as $\widetilde \adv_{\theta}(s, a, g) =  \adv_{\theta}(s, a, g) - \frac{1}{|\mathcal{A}|}\sum_{a'\in\mathcal{A}}\adv_{\theta}(s, a', g)$ if $\mathcal{A}$ is finite, and $\widetilde \adv_{\theta}(s, a, g) =  \adv_{\theta}(s, a, g) - \adv_{\theta}(s, \pi(s, g), g)$ if $\mathcal{A}$ is continuous. The networks for $v_{\theta}$, $a_{\theta}$ and $\pi_{\theta}$ are 3-hidden layers MLP of width 256 and ReLU activations. The inputs of $v_{\theta}$ and $\pi_{\theta}$ are the concatenation of $s$ and $g$. If $\mathcal{A}$ is continuous, the input of $\adv_{\theta}$ is the concatenation of $s, a, g$. If $\mathcal{A}$ is discrete, the input of $\adv_{\theta}$ is the concatenation of $s$ and $g$, and its output is of dimension $|\mathcal{A}|$, every dimension corresponding to an action.

Most hypereparameters are shared among the three methods: we observe batchs of trajectories of size $16$ for the Torus experiments, and of size $2$ for the FetchReach environment.  At every epoch, we observe a batch of trajectories and store it in a memory buffer of size $10^{6}$ transitions. We use an $\varepsilon$-greedy exploration strategy, with $\varepsilon = 0.2$. At every epoch, we sample $100$ batches from the replay buffer for the Torus experiments, and $50$ for the FetchReach environment. For HER, we use the \texttt{future} sampling strategy for goals: when sampling a transition $(s, a, s', g)$, with probability $0.2$ we define $g' = g$, and with probabiliyt $0.8$ we sample $g'$ uniformly in the future of $s$. For $\ddqn$ in the Torus environment, we sample independant goals with $\rho_{\mathcal{G}}$ uniform distribution in the Torus. In FetchReach, we do not assume we have access to the goal sampling distribution. Therefore, we re-sample independant goals from the memory buffer. For every method, observations and goals are normalized. We use a target network with parameter $\theta_{\tar}$ and update the target as $\theta_{\tar} \leftarrow (1-\alpha)\theta_{\tar} + \alpha\theta$ with $\alpha = 0.05$ after every epoch. Every model is trained with the Adam optimizer with $\beta_{1} = 0.9$ and $\beta_{2} = 0.999$.

For every method and environment, the most sensitive hyperparameters were selected with a grid-search. For HER, UVFA and $\ddqn$, we selected the learning rate of the optimizer from a range $\{1e-6, 3e-6, 1e-5, 3e-5, 1e-4, 3e-4, 1e-3\}$. For HER and UVFA, we additionally selected $R$ a reward scaling factor, in $\{1e-2, 1e-1, 1, 10, 100, 1000, 1e4\}$. For $\ddqn$, we also selected a parameter $c_{\delta}$ corresponding to the scaling of the reward: the scaled infinitely sparse reward is $R(s, \d g) = c_{\delta}\delta_{\phi(s)}(\d g)$. We experimented all the possible hyperparameters of this grid separately on every environment on a single run and selected the best hyperparameters. The values in Figure~\ref{fig:experiments} are the mean performance evaluated with $5$ different random seeds, and the confidence intervals represent the standard deviation of the reported metric accross the $5$ independent runs. In practice, the reward scaling factor for UVFA  is $10$ for all the Torus environments and $100$ for FetchReach. The reward factor is $1$ for HER for all the Torus environments and $10$ for FetchReach. The learning rate for UVFA is $1e-4$ for all the Torus environments and $1e-3$ for FetchReach. The learning rate for HER is $3e-4$ for all the Torus environments and $1e-3$ for HER. For $\ddqn$, the learning rate is $1e-5$ for all the Torus environments, and $1e-4$ for the FetchReach environment. The reward scaling coefficient $c_{\delta}$ is $1e-2$ for every environments.

\paragraph{$\delta$-PPO experiments}

The $\dppo$ is defined from $\dac$ similarly to PPO~\citep{Schulman2017ProximalPO} from actor critic methods. 
We learn the model $m_{\theta}(s, g, g')$ of the density of $M^{\pi}(s, g, \d g')$ with respect to $\rho_{\mathcal{G}}$, and $\pi_{\theta}(a|s, g)$ a parametric policy. We used a shared architecture: we define $h_{\theta}(s, g, g')$ a network computing a hidden representation of dimension $H$. Then, we define two linear layers $L^m_{\theta}$ and $L^{\pi}{\theta}$, and define $m_{\theta}(s, g, g') = L^{m}_{\theta}(h_{\theta}(s, g, g'))$ and $\pi_{\theta}(a|s, g) = L^{\pi}_{\theta}(h_{\theta}(s, g, g'))$. In practice, $h_{\theta}$ is a 2-hidden layers MLP with ReLU activations (except at the last layer), with width $H=256$ for the internal and output layers.

A step of $\dppo$ is defined as follow. We first gather a buffer of trajectories with the current policy $\pi_{\theta}$. Then, we define $\theta'\deq \theta$. For every transition  $(s, a, s', g)$ in the buffer and every epoch $e \leq E$, we sample an independant goal $g'$ and compute:
\begin{align}
  \widehat{\delta\theta}_{M} &\leftarrow \widehat{\delta\theta}_{\dtd}(s, a, s', g, g')
  \\ \adv &\leftarrow \gamma m_{\theta_{M}}(s', g, g) - m_{\theta_{M}}(s, g, g)
  \\ r(\theta') &\leftarrow \frac{\pi_{\theta'}(a|s, g)}{\pi_{\theta}(a|s, g)}
         \\ \tilde r(\theta') &\leftarrow \mathrm{clip}(r, 1-u, 1+u)
  \\ \widehat{\delta\theta}_{\pi} &\leftarrow \partial_{\theta'}\left(\min\left(\adv \times r(\theta'), \adv \times \tilde r(\theta')\right)\right)
                               \\ \widehat{\delta\theta} &\leftarrow \widehat{\delta\theta}_{\pi} + c_{M}\times \widehat{\delta\theta}_{M} 
\end{align}
where $c_{M}$ allow to scale the two updates. Then we use $\widehat{\delta\theta}$ and with Adam optimizer to obtain a new value for $\theta'$. We did not use an entropy regularizer aw we observed that the diversity of actions was not an issue in practice.

For the Torus environment, the independent goals $g'$ are sampled fron $\rho_{\mathcal{G}}$ the uniform distribution of goals in the environment. For FetchReach, we do not assume we know $\rho_{\mathcal{G}}$ and sample goals from the buffer.

In practice, at every step of the $\dppo$ algorithm we observe a batch of $2$ trajectories for Torus$(4)$ and Torus$(6)$, $100$ for the Torus$(4)$ with the freeze action $a^{\ast}$, and $200$ for FetchReach. Three hyperparameters were selected independently for every environment via a grid search: $E$ the number of epochs per $\dppo$ step, the learning rate of Adam optimizer, and the coefficient $c_{M}$. We performed a grid search with a single run per tuple of parameters. Then, the reported results in Figure~\ref{fig:experiments} are averaged over $5$ different random seeds with the selected hyperparameters. The number of epoch $E$ per step was selected as lowest number which achieved close-to-optimal performance accross the range $\{1, 2, 5, 10, 20, 50, 100\}$. In practice, $E=20$ in the Torus$(4)$ and Torus$(6)$ environments, $E=10$ in the  Torus$(4)$ with the freeze action $a^{\ast}$, and $E=50$ for FetchReach. The learning rate was selected in the set $\{1e-6, 3e-6, 1e-5, 3e-5, 1e-4, 3e-4, 1e-3\}$, and in practice is $1e-4$ for every environment. The coefficient $c_{M}$ was selected in $\{1e-4, 1e-3, 1e-2, 1e-1, 1e0, 1e1, 1e2, 1e3\}$ and in practice is $1e-3$ for every Torus environment and $1e-1$ for the FetchReach environment.

\paragraph{Additional experiments}
We experimented $\ddqn$ and $\dppo$ in more complex environments such Torus of higher dimension, or other environments of OpenAI Robotic suite~\citep{plappert2018multi}. In the Torus environment, both methods fail when the dimension increases above $15$ while HER is still able to learn. More importantly, $\dppo$ and $\ddqn$ did not learn at all in environments such as FetchPush (which is easy to solve with HER) or HandReach, which has similar structure but higher dimension than FetchReach. In the FetchPush environment, the objective is to push a cube with a robotic arm to a given goal. We observed that the issue of our methods was not an exploration issue, since the robotic arm oftens reaches and pushes the cube randomly. We tried to increase the generalization accross goals with the $\dtdn$ update, but it was to computationally expensive, as explained in Section~\ref{sec:unbi-policy-eval}. Limitations of $\ddqn$ and $\dppo$ which could explain these results are discussed in Section~\ref{sec:discussion}.

\section{Proofs of Theorems on HER}

\subsection{HER is Unbiased in Deterministic Environments}
\label{app:convergence-her}

We prove that HER is an unbiased method in deterministic environments. In order to define HER, we assume access to samples of trajectories $(g, s_{0}, a_{0}, s_{1}, a_{1}, ...) \sim \rho(g, s_{0}, a_{0}, s_{1}, a_{1}, ...)$ with $g \sim \rho_{\mathcal{G}}(\d g)$, $s_{0}\sim \rho_{0}(\d s_{0}|g)$, and for every $k \geq 0$, $a_{k} \sim \pi_{\expl}(a|s_{k}, g)$ where $\pi_{\expl}$ is an exploration policy, $s_{k+1} \sim P(\d s|s_{k}, a_{k})$. For simplicity, we will assume the trajectories are infinite. 



Here we consider HER with the \texttt{future} strategy descibed in the original paper: goals are re-sampled from a trajectory as goal reached later in the trajectory. We formalize  HER as follows: we sample a trajectory $\tau = (g, s_{0}, a_{0}, s_{1}, a_{1},  ...) \sim \rho(g, s_{0}, a_{0}, s_{1}, a_{1}, ...)$, a Bernoulli variable $U \sim \mathcal{B}(\alpha)$, and two  independent integer random variables $K, L$, from distributions $p_{K}$ and $p_{L}$, such that for every $k, l$, $p_{K}(k) >0$ and $p_{L}(l) > 0$. The bernoulli variable $U$ represents the random choice of using the standard Q-learning update, or the HER update with a resampled goal. The random variable $K$ represents the timestep of the transition we will use for the Q-learning update, and $L$ represents the timestep used to sample a new goal $g'$ for the \texttt{future} sampling strategy. Then, the update $\widehat{\delta\theta_{\her}}(\tau, U, K, L)$ is defined as:
\begin{itemize}
\item If $U=0$: $$\widehat{\delta\theta_{\her}}(\tau, U=0, K, L) \deq \partial_{\theta}\frac12 (Q_{\theta}(s_{K}, a_{K}, g) - R(s_{K}, g) - \gamma \sup_{a'}Q(s_{K+1}, a', g))^{2},$$ which corresponds to the usual Q-learning update as defined in UVFA~\citep{pmlr-v37-schaul15}.
\item If $U = 1$ we set $g'= \phi(s_{K+L+1})$ and: $$\widehat{\delta\theta_{\her}}(\tau, U=1, K, L) \deq \partial_{\theta}\frac12 (Q_{\theta}(s_{K}, a_{K}, g') - R(s_{K}, g') - \gamma \sup_{a'}Q(s_{K}, a', g'))^{2},$$which corresponds to a Q-learning update for a re-sampled goal $g' = \phi(s_{K+L})$, a goal achieved later in the trajectory.
\end{itemize}

We say that environment is a \emph{continuous deterministic environment} if there is a continuous function $\psi: \mathcal{S}\times\mathcal{A} \rightarrow \mathcal{S}$ such that for every $(s, a)\in \mathcal{S}\times\mathcal{A}$, $P(\d s' |s, a) = \delta_{\psi(s, a)}(\d s')$. In particular, any discrete deterministic environment is a continuous deterministic environment for the discrete topology. Therefore, the following theorem can be applied to discrete environments.

\label{sec:convergence-her}
\begin{thm}[ (Formal statement of Theorem~\ref{nfthm:her-determ})]
  \label{thm:her-determ}
  We assume the environment is a continuous deterministic environment. We also assume that for every pair of states $(s, s')$, $s'$ is reachable from $s$, which means there is a sequence of actions $(a_{1}, ..., a_{k})$ such that applying these actions from $s$ leads to $s'$. Finally, we assume that the support of the exploration policy $\pi_{\expl}(a|s, g)$ is the entire action space $\mathcal{A}$ for every $s, g$.

  Then, there is an euclidean norm $\|.\|$ such that, for every $\theta$, the HER update with the \texttt{future} sampling strategy at $\theta$, $\widehat{\delta\theta_{\text{HER}}}$ is an unbiased estimate of the gradient step between $Q_{\theta}$ and the target function $Q_{\text{target}}\deq T_{\max}Q_{\theta}$:
  \begin{equation}
    \E\left[\widehat{\delta\theta_{\text{HER}}}\right]    = \partial_{\theta}\frac{1}{2}\|Q_{\theta} - Q^{\tar}\|^{2}
  \end{equation}
  If the state space $\mathcal{S}$ is finite, HER has a single fixed point $Q_{\infty}$, which is equal to $Q^{\ast}$. 
\end{thm}

The euclidean norm $\|.\|$ in the theorem will depend on the exploration policy $\pi_{\expl}(a|s, g)$. Therefore, if the exploration policy is changing during learning, the norm will will be changing as well.

\begin{proof}
  The principle of the proof is the following. We study the sampling distribution of transitions $\mu_{\her}(s, a, s', g)$ with HER. The bias of HER comes from the fact that the sampling of goals $g$ with $\mu_{\her}(s, a, s', g)$ is not independant of $s'$ knowing $(s, a)$. On the contrary, in deterministic environments, the disribution of $g$ knowing $(s, a)$ is independant of $s'$ because $s'$ is uniquely determined by $(s, a)$. 

  We study the sampling distribution of transitions $(s, a, s', g)$ used in HER. Formally, we sample a transition $(s, a, s', g)$ by sampling $\tau, U, K, L$ and defining $(s, a, s', g') \deq \Phi(\tau, U, K, L)$ as:
\begin{itemize}
\item If $U = 0$, $\Phi(\tau, U=1, K, L) = (s_{k}, a_{k}, s_{k+1}, g)$
\item If $U = 1$, $\Phi(\tau, U=1, K, L) = (s_{k}, a_{k}, s_{k+1}, \phi(s_{K+L}))$
\end{itemize}
Then, $\her$ update can be equivalently defined as: sample $(\tau, U, K, L)$ as described above, define $(s, a, s', g) = \Phi(\tau, U, K, L)$, and:
\begin{equation}
    \widehat{\delta\theta_{\her}}(s, a, s', g) \deq \partial_{\theta}\frac12 (Q_{\theta}(s, a, s', g) - R(s, g) - \gamma \sup_{a'}Q(s', a', g))^{2}
  \end{equation}
  Therefore:
  \begin{equation}
         \E\left[\widehat{\delta\theta_{\text{HER}}}\right] = \partial_{\theta}\E_{(s, a, s', g) \sim \mu_{\her}}\frac12 (Q_{\theta}(s, a, g) - R(s, g) - \gamma \sup_{a'}Q(s', a', g))^{2}
  \end{equation}
where we define $\mu_{\her}$ to be the distribution of $(s, a, s', g)$ given by the distribution of $\Phi_{\ast}(\rho \otimes p_{U} \otimes p_{L} \otimes p_{K})$, where $\Phi_{\ast}$ is the \emph{push-forward} operator on measures. We now compute $\mu_{\her}$. Let $f: \mathcal{S}\times \mathcal{A}\times\mathcal{S}\times\mathcal{G} \rightarrow \R$ be a test function, we have:
\begin{align}
  \E_{s, a, s', g \sim \mu_{\her}}\left[f(s, a, s', g)\right] &= \E_{\tau, U, K, L}\left[f(\Phi(\tau, U, K, L))\right]
  \\
  \begin{split}
    = \quad &(1-\alpha) \E_{\tau, U, K, L}\left[f(\Phi(\tau, U, K, L)) | U= 0\right]
    \\ &+ \alpha \E_{\tau, U, K, L}\left[f(\Phi(\tau, U, K, L)) | U= 1\right]
  \end{split}
\end{align}
Moreover:
\begin{align}
  \E_{\tau, U, K, L}& \left[ f(\Phi(\tau, U, K, L)) | U= 0\right] = \sum_{K}p_{K}(k)\int_{g, s_{0}, a_{0}, ...}\rho(g, s_{0}, a_{0}, ...) f(s_{k}, a_{k}, s_{k+1}, g)
  \\ &= \sum_{k}p_{K}(k)\int_{g, s_{0}, a_{0}, ...}\rho_{\mathcal{G}}(g)\rho_{0}(s_{0}|g)(P^{\pi_{\exp}})^{k}(s|s_{0}, g)\pi_{\expl}(a|s, g)P(s'|s, a)f(s, a, s', g)
  \\ &= \int_{s, a, s', g}f(s, a, s', g) \left(\rho_{\mathcal{G}}(g) \int_{s_{0}}\sum_{k}p_{K}(k) \rho_{0}(s_{0}|g) (P^{\pi_{\exp}})^{k}(s|s_{0}, g)\pi_{\expl}(a|s, g)P(s'|s, a)\right)
       \\ &= \int_{s, a, s', g}f(s, a, s', g) \rho_{\mathcal{G}}(g)\nu(s| g)\pi_{\expl}(a|s, g)P(s'|s, a)
\end{align}
with
\begin{equation}
  \label{eq:nu_her}
\nu(s| g) \deq \rho_{\mathcal{G}}(g) \int_{s_{0}}\rho_{0}(s_{0}|g)\sum_{k}p_{K}(k)  (P^{\pi_{\exp}})^{k}(s|s_{0}, g)
\end{equation}
which is the future distribution of states $s$ when sampling a goal $g$ and following the exploration policy $\pi_{\expl}(.| ., g)$, with $p_{K}$ as the distribution of future timesteps. If $p_{K}(k) = (1-\gamma)\gamma^{k}$, this definition of $\nu$ coincides with the definition of $\nu^{\pi}$ in the following sections. This is the reason why we use the same notation, even though $\nu$ is here slightly more general.

We now compute:
\begin{align}
  \E_{\tau, U, K, L}&\left[f(\Phi(\tau, U, K, L)) | U= 1\right] = \sum_{k, l}p_{K}(k)p_{L}(l)\int_{g, s_{0}, a_{0}, ...}\rho(g, s_{0}, a_{0}, ...)f(s_{k}, a_{k}, s_{k+1}, \phi(s_{k+l}))
\end{align}
If $l=0$, the re-sampled goal is $g'=\phi(s)$. Else, the law of $g'$ knowing $s_{k}, a_{k}, s_{k+1}$ is the law of $\phi(s_{k+l})$, which by using the Markov property is the law  of $\phi(\tilde s)$  if $\tilde s$ is sampled as $(P^{\pi_{\expl}})^{l-1}(.|s_{k+1}, g)$. Therefore:
\begin{align}
  &\E_{\tau, U, K, L}\left[f(\Phi(\tau, U, K, L)) | U= 1\right] = \sum_{k, l \geq 0}p_{K}(k)p_{L}(l)\int_{g, s_{0}, a_{0}, ...}\rho(g, s_{0}, a_{0}, ...)f(s_{k}, a_{k}, s_{k+1}, \phi(s_{k+l}))
  \\ \begin{split}
      = &\sum_{k}p_{K}(k)\int_{g, s_{0}, ..., s_{k+1}}\rho(g, s_{0}, ..., s_{k+1})  \left(p_{L}(0)f(s_{k}, a_{k}, s_{k+1}, \phi(s_{k}))\right) +
        \\ &+ \sum_{k}p_{K}(k)\int_{g, s_{0}, ..., s_{k+1}}\rho(g, s_{0}, ..., s_{k+1})\left(\sum_{l\geq 1}p_{L}(l)\int_{\tilde s} (P^{\pi_{\expl}})^{l-1}(\tilde s | s_{k+1}, g)f(s_{k}, a_{k}, s_{k+1}, \phi(\tilde s))\right)
    \end{split}
\end{align}
We define $\mu_{\mathtt{future}}(\d g'|s, s', g) \deq p_{L}(0)\delta_{\phi(s)}(\d g') + \sum_{l\geq 1}p_{L}(l)\phi_{\ast}(\pi_{\exp}\ast P)^{l-1}(g'|s', g)$, where $\phi_{\ast}$ is the \emph{push-forward} on measures, and we have:
\begin{align}
  \E_{\tau, U, K, L}&\left[f(\Phi(\tau, U, K, L)) | U= 1\right] =
  \\ &= \sum_{k }p_{K}(k)\int_{g, s_{0}, a_{0}, ..., s_{K+1}, \tilde s}\rho(g, s_{0}, a_{0}, ..., s_{k+1}) \mu_{\mathtt{future}}(g'|s_{k}, s_{k+1}, g)f(s_{k}, a_{k}, s_{k+1}, g')
 \\ &= \int_{s, a, s', g'}\left(\int_{g}\rho_{\mathcal{G}}(g)\nu(s| g)\pi_{\expl}(a|s, g)\mu_{\mathtt{future}}(g'|s, s', g)\right)P(s'|s, a)f(s, a, s', g').
\end{align}

Therefore,
\begin{equation}
  \mu_{\her}(s, a, s', g) = (1-\alpha)\rho_{\mathcal{G}}(g)\nu(s| g)\pi_{\expl}(a|s, g)P(s'|s, a) + \alpha \left(\int_{\tilde g}\rho_{\mathcal{G}}(\tilde g)\nu(s| \tilde g)\pi_{\expl}(a|s, \tilde g)\mu_{\mathtt{future}}(g|s, s', \tilde g)\right)P(s'|s, a)
\end{equation}

We now use the deterministic hypothesis. We know that for every $s, a$, $P(\d s'|s, a) = \delta_{\psi(s, a)}(\d s')$. We have, for any $s, a$:
\begin{align}
  P(\d s'|s, a)\mu_{\mathtt{future}}(g|s, s', \tilde g) &= \delta_{\psi(s, a)}(\d s')\mu_{\mathtt{future}}(g|s, s', \tilde g)
  \\ &= \delta_{\psi(s, a)}(\d s')\mu_{\mathtt{future}}(g|s, \psi(s, a), \tilde g)
\end{align}

Therefore:
\begin{align}
  &\mu_{\her}(s, a, s', g) =
                            \\ &= (1-\alpha)\rho_{\mathcal{G}}(g)\nu(s| g)\pi_{\expl}(a|s, g)P(s'|s, a) + \alpha \left(\int_{\tilde g}\rho_{\mathcal{G}}(\tilde g)\nu(s| \tilde g)\pi_{\expl}(a|s, \tilde g)\mu_{\mathtt{future}}(g|s, \psi(s, a), \tilde g)\right)P(s'|s, a)
       \\ &= \tilde \mu (s, a, g)P(s'|s, a)
\end{align}
where $$\tilde \mu (s, a, g) \deq (1-\alpha)\rho_{\mathcal{G}}(g)\nu(s, g)\pi_{\expl}(a|s, g) + \alpha \left(\int_{\tilde g}\rho_{\mathcal{G}}(g)\nu(s| \tilde g)\pi_{\expl}(a|s, \tilde g)\mu_{\mathtt{future}}(g|s, \psi(s, a),\tilde g)\right) $$

Therefore:
\begin{align}
  \E\left[\widehat{\delta\theta_{\her}}\right] &= \partial_{\theta}\int_{s, a, s', g} \tilde \mu (s, a, g)P(s'|s, a)(Q(s, a, g) - R(s, g') - \gamma \sup_{a'}Q(s', a', g))^{2}
  \\ &= \partial_{\theta}\int_{s, a, s', g} \tilde \mu (s, a, g)\delta_{\psi(s, a)}(s')(Q(s, a, g) - R(s, g') - \gamma \sup_{a'}Q(s', a', g))^{2}
  \\ &= \partial_{\theta}\int_{s, a, s', g} \tilde \mu (s, a, g)(Q(s, a, g) - R(s, g) - \gamma \sup_{a'}Q(\psi(s, a), a', g))^{2}
  \\  &= \partial_{\theta}\int_{s, a, g} \tilde \mu (s, a, g) (Q(s, a, g) - R(s, g) - \gamma \E_{s'\sim P(\d s'|s, a)}\sup_{a'}Q(s', a', g))^{2}
        \\ &= \partial_{\theta}\int_{s, a, g} \tilde \mu (s, a, g)(Q(s, a, g) - T\cdot Q(s, a, g))^{2}
\end{align}

We define $\|Q\|_{\tilde \mu}$ as:
\begin{align}
  \|Q\|_{\tilde \mu}^{2} \deq \int_{s, a, g} \tilde\mu(s, a, g)Q(s, a, g)^{2}.
\end{align}

We now prove that $\|.\|_{\tilde \mu}$ is a norm for the space of continuous functions on $\mathcal{S}\times\mathcal{A}\times\mathcal{G}$. This is equivalent to showing that the support of the probability measure $\tilde \mu$, $\supp (\tilde \mu)$ is equal to $\mathcal{S}\times \mathcal{A} \times \mathcal{G}$. Because $\tilde \mu (s, a, g) \geq (1-\alpha)\rho_{\mathcal{G}}(g)\nu(s| g)\pi_{\expl}(a|s, g)$, we know that $\supp (\rho_{\mathcal{G}}(g)\nu(s| g)\pi_{\expl}(a|s, g))\subset \supp (\tilde\mu)$. Since for every $s, g$, $\supp (\pi_{\expl}(a|s, g)) = \mathcal{A}$, $\supp (\rho_{\mathcal{G}}(g)\nu(s| g)\pi_{\expl}(a|s, g)) = \supp (\rho_{\mathcal{G}}(g)\nu(s| g))\times\mathcal{A}$. Moreover, $\supp \rho_{\mathcal{G}} = \mathcal{G}$. Therefore, we only need to prove that for every $g$,  $\supp (\nu(.|g) )= \mathcal{S}$.

Let $g \in \mathcal{G}$. Because of the definition of $\nu$ and because $p_{K}(k) > 0$ for every $k$, we have $\supp (\nu(s | g) )= \bigcup_{k \geq 0, s_{0}\in \mathcal{S}} \supp \left( ( P^{\pi_{\expl}})^{k}(s|s_{0}, g)\right)$. 

We define the function $\Psi : \mathcal{S} \times \left(\cup_{k \geq 1}\mathcal{A}^{k} \right)\rightarrow \mathcal{S}$, corresponding to the action of sequences of action , as follows: for every $a$, $\Psi(s, a) = \psi(s, a)$, and for every  $k$, $(a_{1}, ..., a_{k}) \in \mathcal{A}^{k}$, $\Psi(s, (a_{1}, ..., a_{k+1})) \deq \psi(\Psi(s, (a_{1}, ..., a_{k})), a_{k+1})$. $\Psi$ is continuous. Moreover, we assumed that for any pair of states $(s, s')$, there is $k \geq 0$ and a sequence of actions $(a_{0}, ..., a_{k})$ such that applying this sequence of actions from $s$ leads to $s'$. This means that for every $s$, $\Psi(s, .)$ is a surjective continuous function.

Moreover, with 
\begin{align}
  \supp(P^{\pi_{\expl}})^{k+1}(s| s_{0}, g) &= \cup_{s\in \supp( P^{\pi_{\expl}})^{k}(s|s_{0}, g)}\supp\left(\psi(s, \cdot)_{\ast}\pi_{\expl}(.|s, g)\right)
  \\ &\supseteq \cup_{s\in \supp( P^{\pi_{\expl}})^{k}(s|s_{0}, g)}\left(\psi(s, \supp (\pi_{\expl}(.|s, g)) )\right)
\end{align}
by using the continuity of $\psi(s, .)$. Then:
\begin{align}
  \supp( P^{\pi_{\expl}})^{k+1}(s| s_{0}, g) &\supseteq \cup_{s\in \supp(P^{\pi_{\expl}})^{k}(s|s_{0}, g)}\left(\psi(s,  \mathcal{A})\right)
  \\ &= \psi(\supp(P^{\pi_{\expl}})^{k}(s|s_{0}, g)\times  \mathcal{A})
\end{align}
By induction, we have: $\supp(P^{\pi_{\expl}})^{k}(s| s_{0}, g) \supseteq \Psi(s, \mathcal{A}^{k})$. Therefore:
\begin{align}
  \supp (\nu(s | g) ) &= \bigcup_{k \geq 0, s_{0}\in \mathcal{S}} \supp (P^{\pi_{\expl}})^{k}(s|s_{0}, g)
  \\ &\supseteq \bigcup_{k \geq 0, s_{0}\in \mathcal{S}} \Psi(s_{0}, \mathcal{A}^{k})
  \\ &= \bigcup_{s_0 \in \mathcal{S}} \Psi(s_{0}, \bigcup_{k \geq 0}  \mathcal{A}^{k})
  \\ &= \mathcal{S}
\end{align}

This concludes the proof. The main property we use in the theorem is that $\mu_{\mathtt{future}}(g'|s, s', g)$ is independant of $s'$. Therefore, a simple way to remove HER bias is to define $p_{L}(l) = \1_{l=0}$. Still, this would not remove the issue of vanishing rewards, since the fixed point of HER are the same than those of UVFA.

\end{proof}

In the following, we will use again the results derived above. In particular, we know that:

\begin{equation}
  \E\left[\widehat{\delta \theta}_{\text{HER}}\right] = \E_{(s, a, s', g)\sim \mu_{\her}}\left[\partial_{\theta}\frac12 (Q_{\theta}(s, a, g) - R(s, g) - \gamma \sup_{a'}Q(s', a', g))^{2}\right]
\end{equation}
with
\begin{align}
  \mu_{\her}(s, a, s', g) &= (1-\alpha)\rho_{\mathcal{G}}(g)\nu(s| g)\pi_{\expl}(a|s, g)P(s'|s, a) + \alpha \left(\int_{\tilde g}\rho_{\mathcal{G}}(\tilde g)\nu(s| \tilde g)\pi_{\expl}(a|s, \tilde g)\mu_{\mathtt{future}}(g|s', \tilde g)\right)P(s'|s, a)
  \\ \mu_{\mathtt{future}}(g'|s', g) &= \sum_{l}p_{L}(l)\phi_{\ast}(\pi_{\exp}\ast P)^{l}(g'|s', g)
  \\ \nu(s| g) &= \rho_{\mathcal{G}}(g) \int_{s_{0}}\rho_{0}(s_{0}|g)\sum_{k}p_{K}(k)  (\pi_{\exp}\ast P)^{k}(s|s_{0}, g)
\end{align}

\subsection{Proof of HER bias}
\label{app:proof-her-bias}

Let $\mathcal{M} = \langle\mathcal{S}, \mathcal{G}, \mathcal{A}, P, R, \gamma\rangle$ be a multi-goal finite Markov Decision Process, with $\mathcal{G} = \mathcal{S}$ and $R(s, g) = \1_{s=g}$. We define $S = |\mathcal{S}|$ the number of states.

  Let $\tilde{\mathcal{M}}$ be the augmented MDP with a \emph{freeze} action $a^{\ast}$, defined as:
  \begin{itemize}
  \item The augmented state space $\tilde{\mathcal{S}} = \mathcal{S} \times \{0, 1\}$, where $\tilde s = (s, x)$ is said to be frozen if $x=1$. 
  \item The augmented action space $\tilde{\mathcal{A}} = \mathcal{A} \cup \{a^{\ast}\}$, where $a^{\ast}$ is the \emph{freeze} action.
  \item The goal space does not change ($\tilde{\mathcal{G}} = \mathcal{G} = \mathcal{S}$). For an augmented state $\tilde s = (s, x)$, the reward is $\tilde{R}(\tilde s, g) = \tilde{R}((s, x), g) = R(s, g)$
  \item If $\tilde s = (s, x)$ and $\tilde s' = (s', x')$ are two augmented states, the transition operator $\tilde P(\tilde s'|\tilde s, a)$:
    \begin{itemize}
    \item If the state is frozen ($x=1$), the agent can't move: $\tilde P((s', y)|(s, x), a) = \1_{s'=s}\1_{y=1}$
    \item If the state is not frozen ($x=0$) and $a= a^{\ast}$, the agent is sent to a uniformly random frozen state: $\tilde P((s', y)|(s, 0), a) = \1_{y=1}\frac{1}{\mathcal{S}}$
    \item Else, the dynamic is the same than for $\mathcal{M}$: if $x=0$ and $a\neq a^{\ast}$, then $P((s', y)|(s, 0), a) = \1_{y=0}P(s'|s, a)$.
    \end{itemize}

\end{itemize}

We can now prove the existence of MDPs such that HER will be biased in these environments.

\begin{thm}[ (Formal statement of Theorem~\ref{thm:freeze})]
  Let $\mathcal{M}$ be a finite MDP, and $\tilde{\mathcal{M}}$ the augmented MDP with the freeze action $a^{\ast}$ defined above. We assume that for every $s, a, g$ the exploration policy satisfies $\pi_{\expl}(a|s, g) > 0$, and that for every every $s, g$, $\nu(s| g) >0 $, where $\nu$ is defined in equation~(\ref{eq:nu_her}). This means that from the given distribution, every state $s$ has a non-zero probability of being reached when following the exploration policy conditioned by $g$: $\pi_{\expl}(a|s, g)$.
  
Let $Q_{\infty}$ be a fixed point of tabular HER, and $Q^{\ast}$ the true optimal Q-function. Then, for every unfrozen state $(s, 0)$ and goal $g$, HER overstimates the value of action $a^{\ast}$:
\begin{equation}
  Q_{\infty}((s, 0), a^{\ast}, g) > Q^{\ast}((s, 0), a^{\ast}, g)
\end{equation}
\end{thm}

\begin{proof}

  The principle of the proof is the following. First, we prove that for \emph{frozen} states $\tilde s = (s, 1)$, HER converge converge to the true value $Q_{\infty}(\tilde s, a, g) =Q^{\ast}(\tilde s, a, g)$. Then, we compute the action-value of action $a^{\ast}$ for every \emph{unfrozen} state for the true $Q^{\ast}$ and for the fixed point $Q_{\infty}$. HER samples transitions $((s, 0), a^{\ast}, (s', 1), g)$. Let us consider the law of $s'$ knowing $s, a^{\ast}, g$: with probability $(1-\alpha)$ the goal $g$ was re-sampled from the \texttt{future} sampling strategy, therefore, because after $a^{\ast}$ the position will be frozen, we know that  $s'=g$, the goal is reached and the final return is $O(\frac{1}{1-\gamma})$. With probability $\alpha$, the goal $g$ the original goal, the law of $s'$ is uniform, and the return is of order $O(\frac{1}{S(1-\gamma)})$. Therefore, when estimating the return after action $a^{\ast}$ with HER, the computed value will be of order $O(\frac{(1-\alpha)}{1-\gamma})$, while the true value is of order $O(\frac{1}{S(1-\gamma)})$.

  We now prove the theorem. We consider $Q_{\infty}$, a fixed point of the algorithm, which means that starting from $Q_{\infty}$, the stochastic update defined by HER has mean $0$: $\E\left[\widehat{\delta Q}_{\text{HER}}\right] = 0$. We know that  
  \begin{align}
    \E\left[\widehat{\delta Q}_{\text{HER}}\right] &= \E_{(s, a, s', g)\sim \mu_{\her}}\left[\partial_{\theta}\frac12 (Q_{\theta}(s, a, g) - R(s, g) - \gamma \sup_{a'}Q(s', a', g))^{2}\right]
    \\ &= \E_{(s, a, s', g)\sim \mu_{\her}}\left[E_{s, a, g} (Q_{\theta}(s, a, g) - R(s, g) - \gamma \sup_{a'}Q(s', a', g))\right],
  \end{align}
  where $(E_{s, a, g})$ is the canonical basis of the tabular model. Therefore, for every $(s, a, g)$ (because $\mu_{\her}(s, a, g) > 0$ for every $(s, a, g)$), we have:
  \begin{align}
    Q_{\infty}(s, a, g) = R(s, g) + \gamma \E_{s'\sim\mu_{\her}(s'|s, a, g)}\left[\sup_{a'}Q_{\infty}(s',a', g)\right]
  \end{align}

  First, we prove that the values of frozen states $Q_{\infty}((s, 1), a, g)$ is equal to the true optimal Q-values. In that case, $\tilde P(\tilde s'|(s, 1), a) = \delta_{(s, 1)}(\tilde s')$ we can check that $\mu_{\her}(\tilde s'|(s, 1), a, g) = \delta_{(s, 1)}(\tilde s')$. Therefore:
  \begin{align}
    Q_{\infty}((s, 1), a, g) = R(s, g) + \gamma \sup_{a'}Q_{\infty}((s, 1),a', g)
  \end{align}
  Therefore for every $s, a, g$, $Q_{\infty}((s, 1),a, g) = \frac{1}{1-\gamma}R(s, g)$.

  Then, we compute the values of $Q_{\infty}((s, 0), a^{\ast}, g)$, with the \emph{freeze} action for an unfrozen state. We have:
  \begin{align}
  Q_{\infty}((s, 0), a^{\ast}, g) &= R(s, g) + \gamma \E_{(s', y)\sim\mu_{\her}((s', y)|(s, 0), a^{\ast}, g)}\sup_{a'}Q_{\infty}((s', y), a', g)
    \\ &= R(s, g) + \frac{\gamma}{1-\gamma}  \E_{(s', y)\sim\mu_{\her}((s', 1)|s, a^{\ast}, g)}\left[\1_{s'=g}\right]
    \\ &= R(s, g) + \frac{\gamma}{1-\gamma} \mu_{\her}((s', y)=(g, 1)|s, a, g)
\end{align}
because $\mu_{\her}((s', y)|(s, 0), a^{\ast}, g)$ is non zero only if $y=1$, and $Q_{\infty}((s', 1), a', g) = \frac{1}{1-\gamma}R(s', g) = \frac{1}{1-\gamma}\1_{s'=g}$. We now compute $\mu_{\her}((s', y)=(g, 1)|s, a, g)$. We use that $P((s', y)|(s, 0), a^{\ast}) = \1_{y=1}/S$, and $\mu_{\mathtt{future}}(g|(s', y)) = \1_{s'}$ if $y=1$.
\begin{align}
    \mu_{\her}((s, 0), a^{\ast}, (s', 1), g) &= (1-\alpha)\mu_{0}((s, 0), g)\pi_{\expl}(a^{\ast}|(s, 0), g)\frac{1}{S} + \alpha \left(\int_{\tilde g}\mu_{0}((s, 0), \tilde g)\pi_{\expl}(a^{\ast}|(s, 0), \tilde g)\right)\frac{1}{S}\1_{g=s'}
\end{align}
Therefore, for every $s' \neq g$: $\mu_{\her}((s, 0), a^{\ast}, (s', 1), g) < \mu_{\her}((s, 0), a^{\ast}, (g, 1), g)$. So:
\begin{align}
  \sum_{s'}\mu_{\her}((s, 0), a^{\ast}, (s', 1), g) < S\mu_{\her}((s, 0), a^{\ast}, (g, 1), g)
\end{align}
and finally $\mu_{\her}((s', y)=(g, 1)|(s, 0), a^{\ast}, g) > \frac{1}{S}$. Then we have:
\begin{align}
  Q_{\infty}((s, 0), a^{\ast}, g) >  R(s, g) + \frac{\gamma}{S(1-\gamma)} 
\end{align}
On the contrary, we can easily check that for any policy $\pi$, $Q^{\pi}((s, 0), a^{\ast}, g) = R(s, g) + \frac{\gamma}{S(1-\gamma)}$. In particular, by taking $\pi = \pi^{\ast}$, we have:
\begin{align}
  Q_{\infty}((s, 0), a^{\ast}, g) >  Q^{\ast}((s, 0), a^{\ast}, g)
\end{align}
\end{proof}

\section{Goal-dependent $Q$-functions in continuous spaces}

\subsection{Optimal Bellman Operator for action-value measures}
\label{app:optQ}

With continuous states and goals, in a stochastic environment, the  goal-dependent optimal $Q$-function $Q^{\ast}_{\varepsilon}$ with reward $R_{\varepsilon}(s, g) = \1_{\|\phi(s)-g\|\leq\varepsilon}$ vanishes when $\varepsilon\rightarrow 0$: the
probability of exactly reaching a goal state is usually $0$. Likewise, a
direct application of TD would never learn anything because rewards would
likely never be observed.

Instead, the goal-dependent $Q$-function is a \emph{measure} over goals.
Intuitively, for every infinitesimally small set of goals $\d g$, the
quantity $Q^{\ast}(s,a,\d g)$ is the expected amount of time spent in $\d g$ by
the policy that tries to maximize time spent in $\d g$, starting at $(s,a)$.

Formally, for every state-action $(s,a)$, $Q^\ast(s,a,\cdot)$ is a measure
over goals, solution to the Bellman equation
\begin{equation}
\label{eq:goalQ}
Q^\ast(s,a,\d g)=\delta_{\phi(s)}(\d g)+\gamma \,\E_{s'\sim P(\d s'|s,a)}
\max_{a'} Q^\ast(s',a',\d g)
\end{equation}
where, as above, $\phi\from S\to G$ is the function defining the target features, and
where $\delta_{\phi(s)}$ is the Dirac measure at $\phi(s)$ in goal space.
This is an equality between measures, and the supremum is a supremum of
measures \citep[Section 4.7]{bogachev2007measure}.

Existence and uniqueness of solutions, and a formal derivation of a TD
algorithm, are nontrivial in this setting.
Uniqueness never holds without restrictions: the infinite measure always
solves \eqref{eq:goalQ}. But it is not possible to restrict ourselves to
finite-mass measures, because sometimes the solution we want has infinite
mass.  The need to deal with possibly infinite measures restricts the use
of uniqueness proofs by $\gamma$-contractivity arguments in some norm.

Intuitively, the total mass $Q^\ast(s,a,\mathcal{G})$ of the goal state $\mathcal{G}$ describes
how much different action sequences result in non-overlapping
distributions of states.  If the state space $\mathcal{A}$ is finite and $|\mathcal{A}| = A$, the total mass of the horizon-$t$ part of the $Q^\ast$-function can be as much as $\gamma^t A^t$: this
is realized when every possible sequence of $t$ actions leads to a
disjoint part of the state of goals. In Appendix~\ref{app:exqopt} we
provide a simple continuous MDP in which every action sequence leads to a
distinct state: as there are an infinite number of action sequences when
$t\to \infty$, the total mass $Q^\ast(s,a,\mathcal{G})$ is infinite.

We still prove the existence of a canonical solution, equal both to the
\emph{smallest} solution and to the limit of the horizon-$t$ solution
when $t\to \infty$.

\newcommand{\zerom}{\mathbf{0}}

\begin{thm}[ (Formal statement of Theorem~\ref{nfthm:optQ_maintext})]
\label{thm:optQ_appendix}
Let $\mathcal{Q}$ be the set of functions from $\mathcal{S}\times \mathcal{A}$ into positive
measures over $\mathcal{G}$. Assume the set of actions $\mathcal{A}$ is countable.
Let $T$ be the Bellman
operator mapping $Q\in \mathcal{Q}$ to $T\cdot Q$ with
\begin{equation}
  \label{eq:def_optQ}
T\cdot Q(s,a,\cdot)\deq \delta_{\phi(s)}(\cdot)+\gamma \,\E_{s'\sim P(\d
s'|s,a)} \sup_{a'} Q(s',a',\cdot)
\end{equation}
where the supremum is a supremum of measures and $\delta_{\phi(s)}$ is
the Dirac measure at $\phi(s)\in \mathcal{G}$.

Let $\zerom\in \mathcal{Q}$ be the measure $0$.

Let $Q_t\deq T^t\zerom$. (By expanding the definition of $T$, this is the
solution of the expectimax problem at time horizon $t$.)
Then when $t\to\infty$, for every state-action $(s,a)$ and for every
measurable set $G\subset \mathcal{G}$, $Q_t(s,a,G)$ converges to a finite or
infinite limit $Q^\ast(s,a,G)$. This limit $Q^\ast$ is an element of $\mathcal{Q}$ and
solves the
Bellman equation $TQ^\ast=Q^\ast$. It is the smallest such
solution. In
finite state spaces, it is the only solution with finite mass. Moreover,
for any goal-dependent policy $\pi$, its Bellman operator $T^\pi$ and $Q$-value
$Q^\pi\deq \lim_{t\to\infty} (T^\pi)^t\zerom$ can be defined similarly (see equation~\eqref{eq:def_tpi_q})  and
satisfy $Q^\pi\leq Q^\ast$ as measures.
\end{thm}

\begin{proof} Assume the action space $\mathcal{A}$
is countable. Let $\Qspace$ be the set of
measurable functions from $\mathcal{S}\times \mathcal{A}$ to the set of measures on $\mathcal{G}$. 

For $Q_1$ and $Q_2$ in $\Qspace$, we write $Q_1\leq Q_2$ if
$Q_1(s,a,X)\leq Q_2(s,a,X)$ for any state-action $(s,a)$ and measurable
set $X\subset \mathcal{G}$.
The
Bellman operator of Definition~\ref{eq:def_optQ} acts on $\Qspace$ and is
obviously monotonous: if $Q_1\leq Q_2$ then $TQ_1\leq TQ_2$. 

Since the zero measure $\zerom\in \Qspace$ is the smallest measure, we have
$T\zerom\geq \zerom$. Since $T$ is monotonous, by induction we have
$T^{t+1}\zerom\geq T^t\zerom$ for any $t\geq 0$. Thus, the
$(T^t\zerom)_{t\geq 0}$ form an increasing sequence of measures.
Therefore, for every state-action $(s,a)$ and measurable set $X$, the
sequence $(T^t\zerom)(s,a,X)$ is increasing, and thus converges to a
limit. We denote this limit by $Q^\ast(s,a,X)$. We have to prove that
$Q^\ast\in \mathcal{Q}$, namely, that for each $(s,a)$, $Q^\ast(s,a,\cdot)$ is
a measure. The only non-trivial point is $\sigma$-additivity.

Denote
$Q_t\deq T^t\zerom$.
If $(X_i)$
is a countable collection of disjoint measurable sets, we have
\begin{multline}
Q^\ast(s,a,\cup_i X_i)=\lim_{t\to\infty} Q_t(s,a,\cup_i X_i)=\lim_{t\to\infty} \sum_i
Q_t(s,a,X_i)\\=\sum_i \lim_{t\to\infty} Q_t(s,a,X_i)=\sum_i Q^\ast(s,a,X_i) 
\end{multline}
where the limit commutes with the sum thanks to the monotone convergence
theorem, using that $Q_t$ is non-decreasing. Therefore, $Q^\ast$ is a
measure.

Let us prove that $TQ^\ast=Q^\ast$. We have
\begin{equation}
TQ^\ast(s,a,\cdot)=\delta_{\phi(s)}+\gamma \E_{s'\sim P(s'|s,a)}\sup_{a'}
Q^\ast(s',a',\cdot)
\end{equation}
by definition. For any $s'$, denote $\tilde
Q_t(s',\cdot)\deq\sup_{a'} Q_t(s',a',\cdot)$ where the supremum is as
measures over $\mathcal{G}$. Since $Q_t$ is non-decreasing, so is $\tilde Q_t$.

For any state $s'$, we have
\begin{equation}
\sup_{a'}
Q^\ast(s',a',\cdot)=\sup_{a'} \sup_t Q_t(s',a',\cdot)=
\sup_t \sup_{a'} Q_t(s',a',\cdot)=\sup_t \tilde Q_t(s',\cdot)
\end{equation}
since supremums commute. Now, since $\tilde Q_t$ is non-decreasing,
thanks to the monotone convergence theorem, the supremum commutes with
integration over $s'\sim P(s'|s,a)$ (which does not depend on $t$), namely,
\begin{multline}
\E_{s'\sim P(s'|s,a)}\sup_{a'}
Q^\ast(s',a',\cdot)=
\E_{s'\sim P(s'|s,a)} \sup_t \tilde Q_t(s',\cdot)
\\=
\sup_t \E_{s'\sim P(s'|s,a)}\tilde Q_t(s',\cdot)
=
\sup_t \E_{s'\sim P(s'|s,a)} \sup_{a'} Q_t(s',a',\cdot)
\end{multline}
and so $TQ^\ast=\sup_t TQ_t$. Now, since $Q^t=T^t\zerom$, we have
$TQ^t=T^{t+1}\zerom$, so that $\sup_{t\geq 0} TQ^t=\sup_{t\geq
1}T^t\zerom=Q^\ast$. So $Q^\ast$ is a fixed point of $T$.

Let us prove that $Q^\ast$ is the smallest such fixed point. Let $Q'$
such that $TQ'=Q'$. Since $\zerom\leq Q'$ and $T$ is monotonous, we have
$T\zerom \leq TQ'=Q'$. By induction, $T^t\zerom\leq Q'$ for any $t\geq
0$. Therefore, $\sup_t T^t\zerom \leq Q'$, i.e., $Q^\ast\leq Q'$.

The statement for finite state spaces reduces to the 
classical uniqueness property of
the usual $Q$ function, separately for each goal state.

Optimality of the policy is proved by following classical arguments. Let
$\pi(a|s,g)$ be any goal-dependent policy and let $Q\in \Qspace$. 
Define the
Bellman operator associated to $\pi$ by
\begin{equation}
  \label{eq:def_tpi_q}
(T^\pi Q)(s,a,\cdot)\deq \delta_s + \gamma \E_{s'\sim P(s'|s,a)} \sum_{a'}
(\pi\ast Q)(s',a',\cdot)
\end{equation}
where for each action $a$, the measure $(\pi \ast Q)\in \Qspace$ is
defined via $(\pi \ast Q)(s',a',X)\deq \int_{g\in X} \pi(a'|s',g) Q(s',a',\d
g)$, so that the sum of $(\pi \ast Q)$ over all actions $a'$ represents the
expected value of $Q(s',a',\cdot)$ under the goal-dependent policy $\pi$;
this formulation allows the policy to depend on the goal.

Since $\pi$ is a probability distribution, we have
\begin{equation}
\sum_{a'} (\pi \ast Q)(s',a',X)\leq \max_{a'} Q(s',a',X)
\end{equation}
where the right-hand-side is a maximum of measures (thus selecting the best
$a'$ for each goal): this is clear from decomposing $X$ into the
components where each action $a'$ is optimal.

Therefore, for any $Q\in \Qspace$, we have the inequality of measures
\begin{equation}
T^\pi Q\leq TQ
\end{equation}
where $T$ is the optimal Bellman operator from above. Since the latter is monotonous over
$Q\in \Qspace$, for any $Q,Q'\in \Qspace$ with $Q\leq Q'$, we have $T^\pi
Q\leq T Q'$.

Consequently, by induction, $(T^\pi)^t \zerom \leq T^t\zerom$ for any
horizon $t\geq 0$. The monotonous limit $Q^\pi\deq \lim_{t\to\infty}(T^\pi)^t\zerom$
exists for the same reasons as $T^t\zerom$, representing the $Q$-function
(measure) of policy $\pi$. Therefore, $Q^\pi=\lim_{t\to\infty} (T^\pi)^t
\zerom\leq \lim_{t\to\infty} T^t\zerom=Q^\ast$. This proves that the
policy $\pi$ has returns no greater than $Q^\ast$.


\end{proof}

\subsection{Parametric goal-dependent $Q$-learning.}
\label{app:paramq}

In this section, we formally derive the $\ddqn$ update introduced in Section~\ref{sec:unbiased-dqn}. Let us consider parametric models for $Q$:
\begin{equation}
\label{eq:Qmodel}
Q_\theta(s,a,\d g)\deq q_\theta(s,a,g)\rho_{\mathcal{G}}(\d g)
\end{equation}
and we will learn $q_\theta$. 
\footnote{The factor $\rho_{\mathcal{G}}$, or some other measure, is needed to get a well-defined object in
continuous state spaces. In discrete spaces, it results in an
$g$-dependent scaling of the $Q$ function, which still has the same
optimal policy for each $g$.}

The resulting parametric update is  off-policy:
we assume access to a sampling distribution $(s,a,s') \sim \rhoSA(\d s, \d a)P(\d s'|s, a)$ in a Markov decision. Typically, this can correspond to transitions $(s_{k}, a_{k}, s_{k+1})$ from exploration trajectory with $g\sim\rho_{\mathcal{G}}$,  $s_{0}\sim \rho_{0}(.|g)$, then $a_{t}\sim \pi_{\expl}(.|s_{t}, g)$ and $s_{t+1}\sim P(.|s_{t}, a_{t})$. Here, our statement with a distribution $\rhoSA$ is more general. Given
a measure-valued function of $(s,a)$, such as $Q(s,a,\d g)$, we define its norm
as
\begin{equation}
\label{eq:normq}
\norm{Q}^2_{\rhoSA,\rho_{\mathcal{G}}}\deq \E_{(s,a)\sim\rhoSA,\,g\sim \rho_{\mathcal{G}}}[
q(s,a,g)^2]
\end{equation}
where $q(s,a,g)\deq Q(s,a,\d g)/\rho_{\mathcal{G}}(\d g)$ is the density of $Q$
with respect to $\rho_{\mathcal{G}}$, if it exists (otherwise the norm is infinite).

Let $Q_{\theta} = q_{\theta}(s, a, g)\rho(\d g)$ be our current estimate of $Q$, and $Q_{\tar}(s, a, \d g) = q_{\tar}(s, a, g)\rho_{\mathcal{G}}(\d g)$ a target measure, we define the loss:
\begin{equation}
  \label{eq:def_JQ}
  J_{Q}(\theta) \deq \norm{Q_{\theta} - T\cdot Q_{\tar}}^2_{\rhoSA,\rho_{\mathcal{G}}}
\end{equation}
where $T$ is the optimal Bellman operator, and our goal is to obtain an unbiased estimate of $\partial J_{Q}(\theta)$.
In the statement of Theorem~\ref{thm:update_q}, there is a hidden mathematical subtlety with continuous states regarding the norm $\norm{Q_{\theta} - T\cdot Q_{\tar}}^2_{\rhoSA,\rho_{\mathcal{G}}}$.
Indeed, $Q_{\theta}(s, a, \d g) = q_{\theta}(s, a, g)\rho_{\mathcal{G}}(\d g)$ is absolutely continuous with
respect to $\rho_{\mathcal{G}}$, while $T\cdot Q_\tar$ is not, due to the Dirac term $\delta_{\phi(s)}(\d g)$. This makes the
norm $\norm{Q_{\theta}-T\cdot Q_\tar}^2_{\rhoSA,\rho_{\mathcal{G}}}$ infinite (see its
definition in \eqref{eq:normq}). However, the \emph{gradient} of this norm
is actually still well-defined. There are at least two ways to handle this
rigorously, which lead to the same result. It is possible to do the computation in the finite state space case and observe
that the resulting gradient still makes sense in the continuous case
(which can be obtained by a limiting argument). The other way we will use here, is to
observe that the loss
$J_{Q}(\theta)$ is equal to 
\begin{equation}
J_{Q}(\theta) = \frac12\norm{Q_{\theta}}^2_{\rhoSA,\rho_{\mathcal{G}}}
-\langle Q_{\theta},T Q_\tar\rangle_{\rhoSA,\rho_{\mathcal{G}}}
+\frac12\norm{T\cdot Q_\tar}^2_{\rhoSA,\rho_{\mathcal{G}}}
\end{equation}

where
\begin{equation}
  \langle Q_{1},Q_{2}\rangle_{\rhoSA,\rho_{\mathcal{G}}} \deq \int_{s, a}Q_{1}(s, a, \d g)Q_{2}(s, a, \d g)\frac{1}{\rho_{\mathcal{G}}(\d g)}.
\end{equation}
Even though $\|Q_{1}-Q_{2}\|_{\rhoSA,\rho_{\mathcal{G}}}$ is finite only if $Q_{1}$ and $Q_{2}$ are both absolutely continuous with respect to $\rho_{\mathcal{G}}(\d g)$, the dot product $\langle Q_{1},Q_{2}\rangle_{\rhoSA,\rho_{\mathcal{G}}}$ is still defined if only one of $Q_{1}$ or $Q_{2}$ is absolutely continuous. Therefore, we can define:
\begin{equation}
  \label{eq:def_offset_Jq}
  J'_{Q}(\theta)=\frac12\norm{Q_{\theta}}^2_{\rhoSA,\rho_{\mathcal{G}}}-\langle
Q_{\theta},T\cdot Q_\tar\rangle_{\rhoSA,\rho_{\mathcal{G}}}
\end{equation}
For a given $Q_\tar$, $J'_{Q}(\theta)$ and $J_{Q}(\theta)$ have the same minima and gradients, but $J'_{Q}(\theta')$ is always well defined and finite. Namely, $J_{Q}$ and $J'_{Q}$ differ by a constant in the finite case, and by an
``infinite constant'' in the continuous case. We will 
work with the loss $J'_{Q}$, which is finite even in the continuous case.

\begin{thm}[ (Formal statement of Theorem~\ref{thm:update_q})]
  \label{thm:paramq}
   Let $Q_{\theta}(s, a, \d g) = q_{\theta}(s, a, g)\rho_{\mathcal{G}}(\d
  g)$ be a current estimate of $Q^{\ast}(s, a, \d g)$. Let likewise 
  $Q_{\tar}(s, a, \d g) = q_{\tar}(s, a, g)\rho_{\mathcal{G}}(\d
    g)$ be
    a target $Q$-function. We consider the loss function $
  J'_{Q}(\theta)$ defined in equation~\eqref{eq:def_offset_Jq}.

    We consider the following update to bring
  $Q_\theta$ closer to $TQ_{\tar}$ with $T$ the optimal Bellman operator: Let $(s, a, s')\sim \rho_{SA}(\d s, \d a)P(s'|s, a)$ be samples of the environment and $g\sim \rho_{\mathcal{G}}$ sampled independently. Let $\widehat{\delta \theta}_{\ddqn}(s, a, s', g)$ be
\begin{equation}
  \label{eq:q_udpate_delta}
  \widehat{\delta\theta}_{\ddqn}(s, a, s', g) \deq \partial_\theta q_\theta(s,a,\phi(s))+\partial_\theta
q_\theta(s,a,g)\left(\gamma \max_{a'} q_{\tar}(s',a',g)-q_\theta(s,a,g)\right)
  \end{equation}
  Then $\widehat{\delta\theta}_{\ddqn}$ is an unbiased estimate of $\partial_{\theta}J'_{Q}(\theta)$:
        $\E\left[\widehat{\delta\theta}_{\ddqn}\right] = - \partial_{\theta}J'_{Q}(\theta)$.

        In particular, the true optimal state-action measure $Q^\ast$ is
	a fixed point of this update: if $Q_\theta=Q_{\tar}=Q^{\ast}$
	then $\E\left[\widehat{\delta\theta}_{\ddqn}\right] =0$.
\end{thm}

Here we have presented the update using a fixed ``target network'' with parameter
$\theta_0$ (typically a previous value of $\theta$), a common practice
for parametric $Q$-learning.

For this theorem, we sample goals $g$ independently of $(s, a, s')$. In practice, this could be a source of variance, as sampling goals far from the current state should  produce close-to-$0$ Q-values. If we instead sample goals from a distribution $\mu(g|s, a)$, this introduces an implicit \emph{scaling} factor $\alpha(s, g)$ to the reward. This is discussed in details and the end of Appendix~\ref{app:value-measure-continuous} in the case of the $V$-function.

\begin{proof}
By definition of the optimal Bellman operator $T$ and the target $Q_{\tar}$, we have:
\begin{align}
TQ_{\tar}(s,a, \d g)&=\delta_{\phi(s)}(\d g)+\gamma \,\E_{s'\sim P(s'|s, a)}\left[ \sup_{a'}q_{\tar}(s',a', g)\right]\rho_{\mathcal{G}}(\d g)
\end{align}

By definition of $J'_{Q}(\theta)$ and of the norm
$\norm{\cdot}_{\rhoSA, \rho_{\mathcal{G}}}$, we have 
\begin{align}
\label{eq:j}
  J'_{Q}(\theta) &= \frac12\norm{Q_{\theta}}^2_{\rhoSA,\rho}-\langle Q_{\theta},T Q_\tar\rangle_{\rhoSA,\rho}
  \\ &= \frac12\int_{s, a, g}q_{\theta}^{2}(s, a, g)\rhoSA(\d s, \d a)\rho(\d g) - \int_{s, a, g}q_{\theta}(s, a, g)(T\cdot Q_{\tar})(s, a, \d g)\rhoSA(\d s, \d a) 
\end{align}
Consequently,
\begin{align}
\partial_\theta J'(\theta) &= \int_{s, a, g}\partial_{\theta}q_{\theta}(s, a, g) q_{\theta}(s, a, g)\rhoSA(\d s, \d a)\rho_{\mathcal{G}}(\d g) - \int_{s, a, g}\partial_{\theta}q_{\theta}(s, a, g)TQ_{\tar}(s, a, \d g)\rhoSA(\d s, \d a) 
  \\ &=\int_{s, a, g}  \rhoSA(\d s, \d a) \partial_{\theta}q_{\theta}(s, a, g)\left(Q_{\theta}(s, a, \d g) -  TQ_{\tar}(s, a, \d g)    \right)  
       \label{eq:grad_derivation}
\end{align}
assuming $q_\theta$ is smooth enough so that the derivative makes sense and
commutes with the integral.

Moreover, we have:
\begin{multline}
TQ_{\tar}(s,a, \d g)-Q_\theta(s,a, \d g)=\delta_{\phi(s)}(\d g)+\gamma \,\E_{s'\sim P(
  s'|s, a)}[ \sup_{a'}q_{\tar}(s',a', g) - q_{\theta}(s, a, g)]\rho_{\mathcal{G}}(\d g)
\label{eq:Qgap}
\end{multline}
Therefore, 
\begin{align}
  \label{eq:dJ_final}
  \begin{split}
    -\partial_\theta J'(\theta) = &\int_{s, a} \rhoSA(\d s, \d a)\partial_\theta
    q_\theta(s,a, g)\delta_{\phi(s)}(\d g)
    \\  &+ \int_{s, a, g} \rhoSA(\d s, \d a)\rho_{\mathcal{G}}(\d g)\left(\gamma \,\E_{s'\sim P(
      s'|s, a)}[ \sup_{a'}q_{\theta_0}(s',a', g) - q_{\theta}(s, a, g)]\right)
\end{split}
          \\
  \begin{split}
            = &\int_{s, a} \rhoSA(\d s, \d a)\partial_\theta
            q_\theta(s,a, \phi(s))
            \\ &+ \int_{s, a, g} \rhoSA(\d s, \d a)\rho_{\mathcal{G}}(\d g)\left(\gamma \,\E_{s'\sim P(
      s'|s, a)}[ \sup_{a'}q_{\theta_0}(s',a', g) - q_{\theta}(s, a, g)]\right)
  \end{split}
\end{align}

\medskip
By definition of $\widehat{\delta\theta_{Q}}$, we have:
\begin{align}
  \begin{split}
    \E_{s, a \sim \rhoSA}\left[\widehat{\delta\theta_{Q}}\right] =  &\E_{s, a \sim \rhoSA, g\sim\rho_{\mathcal{G}}(\d g)}\left[\partial_\theta q_\theta(s,a,\phi(s))\right]
    \\ &+ \E_{s, a \sim \rhoSA, g\sim\rho_{\mathcal{G}}(\d g)}\left[ \partial_{\theta}q_\theta(s,a,g)\left(
  \gamma \sup_{a'}q_{\theta_0}(s',a',g)-q_{\theta}(s,a,g)\right) \right]
  \end{split}
  \\ = -\partial_\theta J'(\theta)
\end{align}

Finally, if $Q_{\tar} = Q_{\theta} = Q^{\ast}$, then $TQ_{\tar} = Q^{\ast}$ and:
\begin{align}
  \partial J'_{Q}(\theta) = \int_{s, a, g}  \rhoSA(\d s, \d a) \partial_{\theta}q_{\theta}(s, a, g)\left(Q_{\theta}(s, a, \d g) -  TQ_{\tar}(s, a, \d g)    \right)
  \\ &= 0
\end{align}

\end{proof}

\subsection{Examples of MDPs with Infinite Mass for $Q^\ast$}
\label{app:exqopt}

Here are two simple examples of MDPs with finite action space, for which
the mass of the goal-dependent $Q$-measure $Q^\ast(s,a,\d g)$ is infinite.
The first has discrete states, the second, continuous ones.

Take for $\mathcal{S}$ an infinite rooted dyadic tree, namely, $\mathcal{S}=\{\emptyset,
0,1,00,01,\ldots\}$ the set of binary strings of finite length $k\geq 0$, and $\mathcal{G} = \mathcal{S}$.
Consider the two actions ``add a $0$ at the end'' and ``add a $1$ at the
end''. Then, for every state $s$, $Q^\ast(s,a,\cdot)$ is a measure that
gives mass $\gamma^k$ to all states $g$ that are extensions of $s$ by a
length-$k$ string that starts with $a$. Thus, its mass is $1+\sum_{k\geq
1} \gamma^k 2^{k-1}$. This is infinite as soon as $\gamma\geq 1/2$. This
extends to any number of actions by considering higher-degree trees.

A similar example with continuous states is obtained as follows. Let
$\mathcal{S}=[0;1)\times [0;1)$.
Let
$C=\{\emptyset,
0,1,00,01,\ldots\}$ the dyadic tree above. For each string $w\in X$,
consider the set $B_w\subset \mathcal{S}$ defined as follows: $B_w$ is made of
those points $(x,y)\in \mathcal{S}$ such that the binary expansion of $x$ starts
with $w$, and $y\in [1-1/2^k;1-1/2^{k+1})$ where $k$ is the length of
$w$. Graphically, this creates a tree-like partition of the square $\mathcal{S}$,
where the empty string corresponds to the bottom half, the strings $w=0$
and $w=1$ correspond to two sets on the left and right above the bottom
hald, etc. Define the following MDP with two actions $0$ and $1$: with
action $0$, every
state $s\in B_w$ goes to a uniform random state in $B_{w0}$, and with
action $1$, every state $s\in B_w$ goes to a uniform random state in
$B_{w1}$. The goal-dependent $Q$-function $Q^\ast$ is similar to the
dyadic tree above, but is continuous. Its mass is infinite for the same
reasons.

\section{The successor goal measure $M(s, g, \d g')$}

\subsection{Definition and existence of the successor goal measure}
\label{app:m-operator}

\begin{thm}
  The \emph{successor} measure
  \begin{align}
    \nu^{\pi}(\d s|s_{0}, g) = (1-\gamma)\sum_{k\geq 0}\gamma^{k}(P^{\pi})^{k}(\d s | s_{0}, g)
  \end{align}
  is a well defined probability measure over $\mathcal{S}$ for every $s_{0}, g$. It satisfies the fixed-point equation:
  \begin{align}
    \nu^{\pi}(\d s|s_{0}, g) = (1-\gamma)\delta_{s_{0}}(\d s) + \E_{a\sim \pi(\d a|s_{0}, g), s_{1}\sim P(\d s_{1}|s_{0}, a) }\left[\nu^{\pi}(\d s|s_{0}, g)\right]
  \end{align}
  
  The \emph{successor-goal measure} is defined as:
  \begin{align}
    M^{\pi}(s, g, .) \deq \frac{1}{1-\gamma}\phi_{\ast}\nu^{\pi}(.|s, g)
  \end{align}
  where $\phi_{\ast}$ is the \emph{push-forward} operator on measures. We define the Bellman
operator mapping $M(s, g_{1}, \d g2)$ to $T_{\pi}M$ with
\begin{equation}
(T_{\pi}\cdot M)(s, g_{1}, \d g_{2}) = \delta_{\phi}(s)(\d g_{2}) + \gamma \E_{a\sim\pi(a|s, g_{1}), s'\sim P(\d s'|s, a)}\left[M(s', g_{1}, \d g_{2})\right],
\end{equation}

Then, $M^{\pi}$ is a fixed point of $T^{\pi}$.
\end{thm}

\begin{proof}
  For every $k$, $(P^{\pi})^{k}(\d s |s_{0}, g)$ is a probability measure over $\mathcal{S}$. Therefore, for any measurable set $S\subset \mathcal{S}$, the sum $(1-\gamma)\sum_{k\geq 0}\gamma^{k}(P^{\pi})^{k}(S|s_{0}, g) \leq 1$, and the sum converges. $\nu^{\pi}(\d s|s_{0}, g)$ is a positive measure as a convergent sum of positive measure. Its total mass is  $(1-\gamma)\sum_{k\geq 0}\gamma^{k}(P^{\pi})^{k}(\mathcal{S}|s_{0}, g) = (1-\gamma)\sum_{k\geq 0}\gamma^{k} = 1$. Therefore, $\nu^{\pi}(\d s|s_{0}, g)$ is a well-defined probability measure.

  We now prove the fixed point equation. We have:
  \begin{align}
    \nu^{\pi}(\d s|s_{0}, g) &= (1-\gamma)(P^{\pi})^{0}(\d s | s_{0}, g) + (1-\gamma)\sum_{k\geq 1}\gamma^{k}(P^{\pi})^{k}(\d s | s_{0}, g)
    \\ &= (1-\gamma)\delta_{\phi(s_{0})}(\d s) + \left(P^{\pi} \left((1-\gamma)\sum_{k\geq 1}\gamma^{k}(P^{\pi})^{k-1}\right)\right)(\d s | s_{0}, g)
    \\ &= (1-\gamma)\delta_{\phi(s_{0})}(\d s) + \gamma\left(P^{\pi} \ast \nu^{\pi}\right)(\d s|s_{0}, g)
         \\ &= (1-\gamma)\delta_{s_{0}}(\d s) + \E_{a\sim \pi(\d a|s_{0}, g), s_{1}\sim P(\d s_{1}|s_{0}, a) }\left[\nu^{\pi}(\d s|s_{0}, g)\right]
  \end{align}
  where $(P^{\pi}\ast\nu^{\pi})(\d s|s_{0}, g) \deq \int_{s_{1}}P^{\pi}(\d s_{1}| s_{0}, g)\nu^{\pi}(\d s|s_{1}, g)$.
  We now show the Bellman fixed point equation of $M^{\pi}$. We now that 
  \begin{align}
    \nu^{\pi}(\d s|s_{0}, g) = (1-\gamma)\delta_{s_{0}}(\d s) + \E_{a\sim \pi(\d a|s_{0}, g), s_{1}\sim P(\d s_{1}|s_{0}, a) }\left[\nu^{\pi}(\d s|s_{0}, g)\right]
  \end{align}
  By applying the push-forward operator $\phi_{\ast}$ we have:
  \begin{align}
    (1-\gamma)M^{\pi}(s_{0}, g, \d g') &= (1-\gamma)\delta_{s_{0}}(\d s) + \phi_{\ast}\left(\int_{a, s_{1}}\pi(a|s_{0}, g)P(\d s_{1}|s_{0}, a)\nu^{\pi}(\d s|s_{1}, g)\right)
    \\ &= (1-\gamma)\delta_{s_{0}}(\d s) + \left(\int_{a, s_{1}}\pi(a|s_{0}, g)P(\d s_{1}|s_{0}, a)\phi_{\ast}\nu^{\pi}(\d s|s_{1}, g)\right)
         \\ &= (1-\gamma)\delta_{s_{0}}(\d s) +\int_{a, s_{1}}\pi(a|s_{0}, g)P(\d s_{1}|s_{0}, a)M^{\pi}(s_{1}, g, \d g')
  \end{align}
  
\end{proof}


\subsection{The Policy Evaluation Update}
\label{app:update-policy-eval}
In this section, we prove Theorem~\ref{thm:update_m} for learning $M^{\pi}$ via temporal differences algorithm, TD and $TD^{(n)}$. This theorem very similar to Theorem~\ref{thm:paramq}. We directly prove the result for $\tdn$, as the standard $\td$ update stated in Theorem~\ref{thm:update_m} corresponds to the $\tdn$ update for $n=1$.

The resulting parametric update is  on-policy: Let $\pi$ be a policy, 
we assume access to a sampling distribution $(g, s_{0}, ..., ,s_{n}) \sim \rhoSG(\d g, \d s_{0})P^{\pi}(\d s_{1}|s_{0}, g)...P^{\pi}(\d s_{n}|s_{n-1}, g)$, where $\rhoSG$ is any distribution on $\mathcal{S}\times\mathcal{G}$. Typically, this can correpond to couples $(g, s_{k})$ from trajectory with $g\sim\rho_{\mathcal{G}}$,  $s_{0}\sim \rho_{0}(.|g)$,  $s_{t+1}\sim P^{\pi}(.|s_{t},g)$. Here, our statement with a distribution $\rhoSG$ is more general.

Given a measure-valued function of $(s,s)$, such as $M(s,g,\d g')$, we define its norm
as
\begin{equation}
\label{eq:normM}
\norm{M}^2_{\rhoSG,\rho_{\mathcal{G}}}\deq \E_{(s,g)\sim\rhoSG,\,g'\sim \rho_{\mathcal{G}}}[
m(s,g,g')^2]
\end{equation}
where $m(s,g,g')\deq M(s,g,\d g')/\rho_{\mathcal{G}}(\d g')$ is the density of $M^{\pi}(s, g, .)$
with respect to $\rho_{\mathcal{G}}$, if it exists (otherwise the norm is infinite).

\begin{thm}[ (Formal statement of Theorem~\ref{thm:update_m})]
  \label{thm:app_update_m}
  Let $M_{\theta}(s, g, \d g') = m_{\theta}(s, g,
  g')\rho_{\mathcal{G}}(\d g')$ be a current estimate of $M^{\pi}(s, g,
  \d g')$. Let likewise $M_{\tar}(s, g, \d g') = m_{\tar}(s, g,
  g')\rho_{\mathcal{G}}(\d g')$ be a target $M$, and consider the following update to bring $M_{\theta}$ closer to $(T^{\pi})^{n} M_{\tar}$ with $T^{\pi}$ the Bellman operator.

  Let $\tau = (g, s_{0}, ..., s_{n}) \sim \rhoSG(\d g, \d s_{0})P^{\pi}(\d s_{1}|s_{0}, g)...P^{\pi}(\d s_{n}|s_{n-1}, g)$ be a sample of the environment and $g'\sim\rho_{\mathcal{G}}$ is a goal sampled independently. Let $\widehat{\delta \theta}_{\dtdn}$ be
  \begin{equation} \label{eq:app_m_update_tdn}
   \widehat{\delta\theta}_{\dtdn}(\tau, g') \deq
   \sum_{l=0}^{n-1}\gamma^{l}\partial_{\theta}m_{\theta}(s_{0}, g,
   \phi(s_{l})) +   \partial_{\theta}m_{\theta}(s_{0}, g, g')
   \left(\gamma^{n} m_{\theta}(s_{n}, g, g') - m_{\theta}(s_{n}, g, g')
   \right)
 \end{equation}
  Then $\widehat{\delta\theta}_{\dtdn}$  is an unbiased  estimate of the Bellman error:
  $\E_{\tau, g'}\left[\widehat{\delta\theta}_{\dtd}(\tau, g')\right] = \frac{1}{2}\partial_{\theta}\|M_{\theta} - (T^{\pi})^{n}M_{\tar}\|_{\rhoSG, \rho_{\mathcal{G}}}^{2}$.

  In particular, the true  $M^{\pi}$ is a fixed point of this udpate: if $M_{\theta} = M_{\tar} = M^{\pi}$, then 
  \begin{equation}
\E\left[\widehat{\delta\theta}_{\dtdn}\right]=0
\end{equation}

    \end{thm}
For this theorem, we sample goals $g'$ independently of $\tau$. In practice, this could be a source of variance, as sampling goals far from the current state should  produce close-to-$0$ V-values. If we instead sample goals from a distribution $\mu(g|s, a)$, this introduces an implicit \emph{scaling} factor $\alpha(s, g)$ to the reward. This is discussed in details and the end of Appendix~\ref{app:value-measure-continuous} in the case of the $V$-function.

    As in Appendix~\ref{app:paramq}, there is a hidden mathematical subtlety with continuous states regarding the norm $\|M_{\theta} - (T^{\pi})^{n}M_{\tar}\|$, which is infinite because $(T^{\pi})^{n}M_{\tar}$ is not absolutely continuous with respect to $\rho_{\mathcal{G}}$. However, as in  Appendix~\ref{app:paramq}, the \emph{gradient} of $\|M_{\theta} - (T^{\pi})^{n}M_{\tar}\|$ is finite. Because the rigorous way to handle it is exactly the same technique as in Appendix~\ref{app:paramq}, we will not derive it in this section.

\begin{proof}
      The proof is very similar to the proof of Theorem~~\ref{thm:paramq}. Similarly to the derivation of (\ref{eq:grad_derivation}), we have:
      \begin{equation}
        -\partial_{\theta}\frac{1}{2}\partial_{\theta}\|M_{\theta} - (T^{\pi})^{n}M_{\tar}\|_{\rhoSG, \rho_{\mathcal{G}}}^{2} = \int_{s_{0}, g, g'}\rhoSG(\d s_{0}, \d g)\partial_{\theta}m(s_{0}, g, g')((T^{\pi})^{n}M_{\tar}(s, g, \d g') - M_{\theta}(s, g, \d g'))        
      \end{equation}

      Moreover:
      \begin{align}
        \begin{split}
          (T^{\pi})^{n}M^{\tar}(s, g, \d g') - M_{\theta}(s, g, \d g') = &\sum_{k=0}^{n-1}\gamma^{k}
          \E_{s_{1}, ..., s_{k}|s_{0}, g}\left[\delta_{\phi(s_{k})}(\d g)\right]
          \\ &+ \left(\gamma^{n}\E_{s_{n}|s_{0}, g}\left[m_{\tar}(s_{n}, g, g')\right]- m_{\theta}(s_{0}, g, g')\right)\rho_{\mathcal{G}}(\d g')
      \end{split}
      \end{align}

      Therefore:
      \begin{align}
        \begin{split}
          -\partial_{\theta}J(\theta) = &\int_{s_{0}, g, g'}\rhoSG(\d s_{0}, \d g)\partial_{\theta}m(s_{0}, g, g')\sum_{k=0}^{n-1}\gamma^{k}
          \E_{s_{1}, ..., s_{k}|s_{0}, g}\left[\delta_{\phi(s_{k})}(\d g)\right]
          \\ &+ \int_{s_{0}, g, g'}\rhoSG(\d s_{0}, \d g)\partial_{\theta}m(s_{0}, g, g')\left(\gamma^{n}\E_{s_{n}|s_{0}, g}\left[m_{\theta}(s_{n}, g, g')\right]- m_{\theta}(s_{0}, g, g')\right)\rho_{\mathcal{G}}(\d g')
        \end{split}
        \\ \begin{split}
          -\partial_{\theta}J(\theta) = &\int_{s_{0}, g, g'}\rhoSG(\d s_{0}, \d g) \E_{s_{1}, ..., s_{n}|s_{0}, g}\left[\sum_{k=0}^{n-1}\gamma^{k}
         \partial_{\theta}m(s_{0}, g, \phi(s_{k}))\right]
          \\ &+ \int_{s_{0}, g, g'}\rhoSG(\d s_{0}, \d g)\partial_{\theta}m(s_{0}, g, g')\left(\gamma^{n}\E_{s_{n}|s_{0}, g}\left[m_{\theta}(s_{n}, g, g')\right]- m_{\theta}(s_{0}, g, g')\right)\rho_{\mathcal{G}}(\d g')
        \end{split}
      \end{align}
      Therefore, $\E_{\tau, g'}\left[\widehat{\delta\theta}_{\dtd}(\tau, g')\right] = \frac{1}{2}\partial_{\theta}\|M_{\theta} - (T^{\pi})^{n}M_{\tar}\|_{\rhoSG, \rho_{\mathcal{G}}}^{2}$.

      Finally, if $M_{\theta} = M_{\tar} = M^{\pi}$, we have:
      \begin{align}
        -\partial_{\theta}\frac{1}{2}\partial_{\theta}\|M_{\theta} - (T^{\pi})^{n}M_{\tar}\|_{\rhoSG, \rho_{\mathcal{G}}}^{2} &= \int_{s_{0}, g, g'}\rhoSG(\d s_{0}, \d g)\partial_{\theta}m(s_{0}, g, g')((T^{\pi})^{n}M_{\tar}(s, g, \d g') - M_{\theta}(s, g, \d g'))
        &= 0
      \end{align}
      This concludes the proof.
    \end{proof}

\section{The continuous density setting}
\label{app:equiv-betw-vareps}

\subsection{The continuous density assumption}
\label{app:cont-dens-assumpt}

Here, we introduce the continuity assumption, which will be used in this section, to formalize the relation between the multi-goal formulation with infinitely sparse Dirac rewards with the standard formulation with reward located in a neighborhood of size $\varepsilon$ around the goal, and to derive a policy gradient theorem.


\begin{assumption}
  \label{assumption:density}
  We assume that $\mathcal{S}$ and $\mathcal{G}$ are finite dimensional vector spaces, and that $\mathcal{A}$ is a compact subset of a finite dimensional vector space. 
  Moreover, $\rho_{\mathcal{G}}(\d g)$ is absolutely continuous with respect to the Lebesgue measure on $\mathcal{G}$, and we write $p_{\mathcal{G}}$ its density: $p_{\mathcal{G}}(g)\lambda(\d g)$, where $p_{\mathcal{G}}$ is a continuous function. 
  Similarly, $\rho(\d s_{0}|g)$ the distribution of initial states given a goal is supposed to be absolutely continuous with respect the Lebesgue measure: $\rho(\d s_{0}|g) = p_{0}(s_{0}|g)\lambda(\d g)$, with $p_{0}$ continuous. The transition probability measure $P(\d s'|s, a)$ is absolutely continuous with respect to the Lebesgue measure on $\mathcal{S}$, and we write $p(s'|s, a)$ its density, which is continuous.

   We assume that $\supp \rho_{\mathcal{G}}$ is compact and that there is a compact subset $\ks \subset \mathcal{S}$ such that for every $s, a\in\mathcal{S},\mathcal{G}$, $\supp P(\d s'|s, a)\subset \ks$. 

  We consider only policies in $\Pi$, the set of policies $\pi$ such that $\pi(a|s, g)$ is  a continuous function of $a, s, g$. 

  We assume $\dim \mathcal{G} \leq \mathcal{S}$ and $\phi$ is a surjective linear function, and $\phi(\mathcal{S}) = \mathcal{G}$.
\end{assumption}

Let us comment Assumption~\ref{assumption:density}. First, we require $P$, $\rho_{\mathcal{G}}$, and $\rho_{0}$ to be absolutely continuous with respect to Lebesgue measure. This is typically true in environments such that, at every step, the environment adds a noise absolutely continuous with respect to Lebesgue measure (for instance Gaussian) to the position. On the contrary, in environments such that the agent lies in a submanifold of dimension lower than $\dim \mathcal{S}$, the assumption is not satisfied. In the \texttt{Torus(n)} environment, with the state representation $s \in [0,1)^{n}$, the environment satisfies this assumption. But with the representation used in the experiments $\tilde s = (\cos(2\pi s_{1}), \sin(2\pi s_{1}), ..., \cos(2\pi s_{n}), \sin(2\pi s_{n}))\in [-1,1]^{2n}$, this assumption is not satisfied. The assumption that $\phi$ is linear is often satisfied in practice, when the achieved goal of a state corresponds to a some coordinates of $s$. For instance, in FetchReach, the state $s$ contains information on the position and velocity of the robotic arm, while the achieved goal is the position of the extremity of the robotic arm. This assumption could be generalized to $\phi$ a submersion (a differentiable function such that its $\d \phi_{s}$ is surjective for every $s$), but we used the linear assumption for the simplicity of the proof.

\bigskip

Under this assumption, we have the following lemma on the probability distribution $\nu^{\pi}$ introduced in Appendix~\ref{app:m-operator} and $M^{\pi}$:

\begin{lem}
  \label{lem:density}
  Under Assumption~\ref{assumption:density}, there is a function $q^{\pi}(s|s_{0}, g)$ such that for any $(s_{0}, g)$:
  \begin{equation}
    \nu^{\pi}(\d s|s_{0}, g) = (1-\gamma)\delta_{s_{0}}(\d s) + q^{\pi}(s|s_{0}, g)\lambda(\d s)
  \end{equation}
  and  $q^{\pi}(s|s_{0}, g)$ is a continuous function of $s, s_{0}, g$.

  Moreover, $M^{\pi}(s, g, \d g') = \phi_{\ast}\big(\nu^{\pi}(\d s'|s, g)\big)(\d g')$ (where $\phi_{\ast}$ is the push-forward operator) and there is a function $\tilde m^{\pi}(s, g, g')$ such that for any $s, g$: 
  \begin{equation}
    M^{\pi}(s, g, \d g') = \delta_{\phi(s)} + \tilde m^{\pi} (s, g, g')\lambda(\d g').
  \end{equation}
  and $\tilde m^{\pi}(s, g, g')$ is a continuous function of $(s, g, g')$.

  The function $\tilde m^{\pi}$ satisfies for every $(s, g, g') \in \ks\times\mathcal{G}\times\mathcal{G}$ the fixed point equation:
  \begin{equation}
    \label{eq:bellman_tildem}
    \tilde m(s, g, g') = \gamma\int_{a}\lambda(\d a)\pi(a|s, g)\left(\tilde p(g'|s, a) + \int_{s'}\lambda(\d s')p(s'|s, a)\tilde m^{\pi}(s', g, g')\right)
  \end{equation}
\end{lem}

\begin{proof}
  We have:
\begin{align}
  \nu^{\pi}(\d s|s_{0}, g) &= (1-\gamma)\sum_{k\geq 0 }\gamma^{k}(P^{\pi})^{k}(\d s|s_{0}, g)
\end{align}
We know that
$$(P^{\pi})(\d s'|s, g) = \lambda(\d s')\int_{a}\lambda(\d a)\pi(a|s, g)p(s'|s, a),$$
and by induction, for $k\geq 1$,
$$(P^{\pi})^{k}(\d s|s_{0}, g) = \lambda(\d s)\int_{a_{0}, ..., s_{k-1}, a_{k-1}} \pi(a_{0}|s_{0}, g)\left(\prod_{i=1}^{k-1}p(s_{i}|s_{i-1}, a_{i-1})\pi(a_{i}|s_{i}, g)   \right)p(s|s_{k-1}, a_{k-1}).$$
We define: 
\begin{equation}
q^{\pi}(s| g, s_{0}) \deq (1-\gamma) \sum_{k\geq 1}\gamma^{k}\int_{a_{0}, ..., s_{k-1}, a_{k-1}} \pi(a_{0}|s_{0}, g)\left(\prod_{i=1}^{k-1}p(s_{i}|s_{i-1}, a_{i-1})\pi(a_{i}|s_{i}, g)   \right)p(s|s_{k-1}, a_{k-1})
\end{equation}
We now check that $q^{\pi}$ is well-defined and continuous. For every $k \geq 1$, the function
$$(g, s_{0}, a_{0}, ..., s_{k-1}, a_{k-1}, s) \mapsto  \pi(a_{0}|s_{0}, g)\left(\prod_{i=1}^{k-1}p(s_{i}|s_{i-1}, a_{i-1})\pi(a_{i}|s_{i}, g)   \right)p(s|s_{k-1}, a_{k-1})$$
is continuous and the supports of $\pi$ are $p$ compact sets. Therefore, for every $k \geq 0$, the function $$(g, s_{0}, s) \mapsto \int_{a_{0}, s_{1}, ..., s_{k-1}, a_{k-1}}\pi(a_{0}|s_{0}, g)\left(\prod_{i=1}^{k-1}p(s_{i}|s_{i-1}, a_{i-1})\pi(a_{i}|s_{i}, g)   \right)p(s|s_{k-1}, a_{k-1})$$
is well defined and continuous.

Moreover, for every $k \geq 0$, and $(s, g)$: 
\begin{align}
  &\left|\gamma^{k}\int_{a_{0}, ..., s_{k-1}, a_{k-1}} \pi(a_{0}|s_{0}, g)\left(\prod_{i=1}^{k-1}p(s_{i}|s_{i-1}, a_{i-1})\pi(a_{i}|s_{i}, g)   \right)p(s|s_{k-1}, a_{k-1})\right| \leq
  \\ &\leq \gamma^{k} \int_{a_{0}, ..., s_{k-1}, a_{k-1}} \pi(a_{0}|s_{0}, g)\left(\prod_{i=1}^{k-1}p(s_{i}|s_{i-1}, a_{i-1})\pi(a_{i}|s_{i}, g)   \right)\|p\|_{\infty}
  \\ &= \gamma^{k}\|p\|_{\infty}   
\end{align}
and $\sum_{k\geq 0} \gamma^{k} \|p\|_{\infty} \leq \infty$.
Therefore, $q^{\pi}(s|g, s_{0})$ is a continuous function and we have:
\begin{align}
  \nu^{\pi}(\d s|s_{0}, g) = (1-\gamma)\delta_{s_{0}}(\d s) + q^{\pi}(s|s_{0}, g)\lambda(\d s).
\end{align}
Moreover, the support of $\nu^{\pi}$ is compact and for every $s_{0}\in\ks$, we have $\supp \left(\nu^{\pi}(.|s_{0}, g)\right)\subset \ks$.

We now show the existence of $\tilde m^{\pi}$. We have: 
\begin{equation*}
M^{\pi}(s, g, \d g') = \frac{1}{1-\gamma}\big(\phi_{\ast}\nu^{\pi}(\d s'|s, g)\big)(\d g') = \phi_{\ast}\left((\delta_{s}(\d s')\right)(\d g') + \frac{1}{1-\gamma}\phi_{\ast}\left(q^{\pi}(s'|s, g)\lambda(\d s')\right)(\d g')
\end{equation*}
First, $\phi_{\ast}(\delta_{s}) = \delta_{\phi(s)}$. Then, we study the second part $\phi_{\ast}\left(q^{\pi}(s'|s, g)\lambda(\d s')\right)(\d g')$, and show that there is a continuous function $\tilde m(s, g, g')$ such that
\begin{equation}
  \label{eq:existence_tildem}
\frac{1}{1-\gamma}\phi_{\ast}\left(q^{\pi}(s'|s, g)\lambda(\d s')\right)(\d g') = \tilde m(s, g, g')\lambda(\d g')
\end{equation}

Let $f(g)$ be a continuous test function. We have:
\begin{align}
  \int_{g'\in \mathcal{G}}f(g')\phi_{\ast}\left(q^{\pi}(s'|s, g)\lambda(\d s)\right)(\d g')  = \int_{s'}f(\phi(s'))q^{\pi}(s'|s, g)\lambda(\d s')
\end{align}
We use the change of variable $s' = e + k$ with $k \in \Ker \phi$ and $e \in  \Ker \phi^\perp$ and use that $\phi(s')=\phi(e)$, and $\phi_{\Ker\phi^{\perp}}$ the restriction of $\phi$ to $\Ker\phi^{\perp}$ is invertible. In order to use continuity theorems on integrals, we want to restrict the integral domains to compact sets. We define the orthogonal projections of $\ks$ on $\Ker\phi$ and $\Ker\phi^{\perp}$: $K = \mathrm{proj_{\Ker\phi}(\ks)}$ and $E = \mathrm{proj_{\Ker\phi^{\perp}}(\ks)}$. $K$ and $E$ are compact sets and $\supp \left(q^{\pi}(s'|s, g)\right) \subset  \{e+k\; , \; (k, e) \in K\times E\}$ for every $s\in\ks$. We have:
\begin{align}
  \int_{g'\in\mathcal{G}}f(g')\phi_{\ast}\left(q^{\pi}(s'|s, g)\lambda(\d s')\right)(\d g')  & = \int_{e\in \Ker\phi^{\perp}, k\in \Ker\phi}f(\phi(e+k))q^{\pi}(e+k|s, g)\lambda(\d e, \d k)
  \\ &= \int_{e\in E, k\in K}f(\phi(e+k))q^{\pi}(e+k|s, g)\lambda(\d e, \d k)
  \\ &= \int_{e\in E}f(\phi(e))\lambda(\d e)\int_{k\in K}q^{\pi}(e+k|s, g)\lambda(\d k)
\end{align}
where we can switch integrals because the sets are compact and the functions continuous. We use the change of variable: $g' = (\phi_{|\Ker\phi^{\perp}})^{-1}(e)$. For simplicity, we use the notation $\phi^{-1} = (\phi_{|\Ker\phi^{\perp}})^{-1}$. 
\begin{align}
\int_{g'\in\mathcal{G}}f(g')\phi_{\ast}\left(q^{\pi}(s'|s, g)\lambda(\d s')\right)(\d g')   &=\int_{g'\in\mathcal{G}}f(g')\lambda(\d g')\left(\det(\phi^{-1})\int_{k\in K} \1_{E}(\phi^{-1}(g') q^{\pi}(\phi^{-1}(g')+k|s, g)\lambda(\d k)\right)
       \\ &=\int_{g'\in\mathcal{G}}f(g')\lambda(\d g')\left(\det(\phi^{-1})\int_{k\in K}  q^{\pi}(\phi^{-1}(g')+k|s, g)\lambda(\d k)\right)
\end{align}
where the last line is obtained by using that $\1_{E}(s') q^{\pi}(s'|., .)= q^{\pi}(s'|., .)$ because $s'\notin E \Rightarrow q^{\pi}(s'|., .)=0$. We define $\tilde m^{\pi}(s, g, g') = \frac{1}{1-\gamma}\det(\phi)^{-1} \int_{k\in K}q^{\pi}(\phi^{-1}(g')+k|s, g)\lambda(\d k)$. The function $(s, g, k, g') \rightarrow q^{\pi}(\phi^{-1}(g')+k|s, g)$ is continuous and $K$ is compact. Therefore, $\tilde m^{\pi}$ is continuous and bounded, and:
\begin{equation*}
\frac{1}{1-\gamma}\phi_{\ast}\left(q^{\pi}(s'|s, g)\lambda(\d s')\right)(\d g') = \tilde m^{\pi}(s, g, g')\lambda(\d g')
\end{equation*}
Moreover, the support of $\tilde m(s, g, g')\lambda(\d g')$ is compact and $\supp\left(\tilde m^{\pi}(s, g, g')\lambda(\d g')\right)\subset \phi(\ks)$.

We now prove the fixed point equation on $\tilde m^{\pi}$. We consider the Bellman equation on $M^{\pi}(s, g, \d g')$. We have:
  \begin{align}
    M^{\pi}(s, g, \d g') = \delta_{\phi(s)}(\d g') + \gamma \int_{s', a}\lambda(\d s', \d a)\pi(a|s, g)p(s'|s, a)M^{\pi}(s', g, \d g')
  \end{align}
  By using $M^{\pi}(s, g, \d g')=\delta_{\phi(s)}(\d g') + \tilde m^{\pi}(s, g, g')\lambda(\d g')$, we have:
  \begin{align}
   \tilde m^{\pi}(s, g, g')\lambda(\d g') =  \gamma \int_{s', a}\lambda(\d s', \d a)\pi(a|s, g)p(s'|s, a)\left(\delta_{\phi(s')}(\d g') + \tilde m(s, g, g')\lambda(\d g')\right)
  \end{align}  
  Let $f(g')$ be a continuous test function, we have:
  \begin{align}
    \label{eq:tildep_der_1}
    \int_{g'}&f(g')\tilde m^{\pi}(s, g, g')\lambda(\d g') =
               \\ &= \gamma \int_{s', a, g'}\lambda(\d s', \d a)f(g')\pi(a|s, g)p(s'|s, a)\left(\delta_{\phi(s')}(\d g') + \tilde m(s, g, g')\lambda(\d g')\right)
    \\ &= \gamma \int_{s', a}\lambda(\d s', \d a)\pi(a|s, g)p(s'|s, a)\left( f(\phi(s')) + \int_{g'}\lambda(\d g')f(g')\pi(a|s, g)p(s'|s, a)\tilde m(s', g, g')\right)
    \\ &= \gamma \int_{a, g'}f(g')\pi(a|s, g)\tilde p(g'|s, a) + \gamma\int_{a, s', g'}\lambda(\d a, \d s', \d g')f(g')\pi(a|s, g)p(s'|s, a)\tilde m(s', g, g')f(g')
         \label{eq:tildep_der_2}
  \end{align}
  
  where $\tilde p(g|s, a)$ is the density with respect to Lebesgue measure $\lambda(\d g)$ of $\phi_{\ast}P(\d s'|s, a)$. Formally, the existence proof of $\tilde p$ is the same than for $\tilde m$ in equation~\eqref{eq:existence_tildem}, and is using the fact that $P$ is continuous with respect to $\lambda(\d s)$ and $\phi$ is a surjective linear operator. Therefore, we have, for $\lambda$-almost $s, g, g'$:
  \begin{align}
    \tilde m(s, g, g') = \gamma\int_{a}\lambda(\d a)\pi(a|s, g)\left(\tilde p(g'|s, a) + \gamma\int_{s'}\lambda(\d s')p(s'|s, a)\tilde m^{\pi}(s', g, g')\right)
  \end{align}
  Because $\tilde m^{\pi}$ is continuous, this relation is true for every $s, g, g'$, in particular if $g=g'$.
\end{proof}


\subsection{The Value Measure Under the Continuous Density Assumption}
\label{app:value-measure-continuous}

Under Assumption~\ref{assumption:density}, we can rigorously define the value measure $V^{\pi}(s, \d g)$ as follows. Then, we briefly show why learning directly $V^{\pi}(s, \d g)$ without bias poses technical issues as states in Section~\ref{sec:unbi-policy-eval}, which is the reason why we learn $M^{\pi}$.

\begin{thm}
  \label{thm:value-measure-continuous}
  Under Assumption~\ref{assumption:density}, we can define the value-\emph{measure} $V^{\pi}(s, \d g)$ as the measure on $\mathcal{G}\times\mathcal{G}$: 
  \begin{equation}
    V^{\pi}(s, \d g) = \delta_{\phi(s)}(\d g) + \tilde m(s, g, g)\lambda(\d g)    
  \end{equation}

  The value measure $V^{\pi}$ satisfies the fixed point equation:
  \begin{equation}
       V^{\pi}(s, \d g) = \delta_{\phi(s)}(\d g) + \gamma \E_{s'\sim P(\d s'|s, g)}\left[V^{\pi}(s', \d g)\right]
     \end{equation}

     Finally, the value-measure is consistent with the value function $V^{\varepsilon}(s, g)$ when $\varepsilon\rightarrow 0$. Formally, the measure on $\ks \times\mathcal{G}$:  $\lambda(\d s)\frac{1}{\lambda(\varepsilon)}V_{\varepsilon}^{\pi}(s, g)\lambda(\d g)$ converges weakly to $\lambda(\d s)V^{\pi}(s,\d g)$ when $\varepsilon\rightarrow 0$.

\end{thm}

\begin{proof}
  Let $f(g)$ be a continuous test function. We have:
  \begin{align}
    \int_{g}V^{\pi}(s, \d g)f(g) = f(\phi(s)) + \int_{g}\tilde m^{\pi}(s, g, g)f(g)\lambda(\d g)
  \end{align}
  Moreover, we know from Lemma~\ref{lem:density} that
  \begin{equation}
    \tilde m^{\pi}(s, g, g) = \gamma\int_{a}\lambda(\d a)\pi(a|s, g)\left(\tilde p(g|s, a) + \int_{s'}\lambda(\d s')p(s'|s, a)\tilde m^{\pi}(s', g, g)\right)
  \end{equation}
  Therefore:
  \begin{align}
    \int_{g}&V^{\pi}(s, \d g)f(g)  = f(\phi(s)) + \gamma\int_{g, a}\lambda(\d a, \d g)f(g)\pi(a|s, g)\left(\tilde p(g|s, a) + \int_{s'}\lambda(\d s')p(s'|s, a)\tilde m^{\pi}(s', g, g)\right)
  \end{align}

  On the other side, we have:
  \begin{align}
    \int_{g}f(g)&  \E_{a\sim\pi(.|s, g), s'\sim P(\d s'|s, a)}\left[V^{\pi}(s', \d g)\right] =
    \\ &=  \int_{g, a, s'}\lambda(\d a, \d s')f(g)\pi(a|s, g)p(s'|s, a)\left(\delta_{\phi(s')}(\d g) + \tilde m^{\pi}(s', g, g)\lambda(\d g)\right)
    \\ &= \int_{a, s'}\lambda(\d a, \d s')\pi(a|s, g)p(s'|s, a)f(\phi(s')) + \int_{a, s', g}\lambda(\d a, \d s', \d g)f(g)\pi(a|s, g)p(s'|s, a)\tilde m^{\pi}(s', g, g)
  \end{align}
  For the first part, we use the change of variable $g = \phi(s')$, and we have:
  \begin{align}
    \int_{g}f(g)&  \E_{a\sim\pi(.|s, g), s'\sim P(\d s'|s, a)}\left[V^{\pi}(s', \d g)\right] =
    \\ &= \int_{a, g}\lambda(\d a, \d g)\pi(a|s, g)\tilde p(g|s, a)f(g) + \int_{a, s', g}\lambda(\d a, \d s', \d g)f(g)\pi(a|s, g)p(s'|s, a)\tilde m^{\pi}(s', g, g)
  \end{align}
  where $\tilde p(.|s, a)$ is the density of $\phi_{\ast}P(.|s, a)$ (where $\phi_{\ast}$ is the push-forward operator) with respect to Lebesgue measure. Therefore, we have:
  \begin{equation}
    \int_{g}V(s, \d g)f(g) = \int_{g}f(g)\left(\delta_{\phi(s)}(\d g) + \gamma \E_{s'\sim P^{\pi}(\d s'|s, g)}\left[V^{\pi}(s', \d g)\right]\right)
  \end{equation}
  and we can conclude:
  \begin{align}
    V^{\pi}(s, \d g)f(g) = \delta_{\phi(s)}(\d g) + \gamma \E_{s'\sim P^{\pi}(\d s'|s, g)}\left[V^{\pi}(s', \d g)\right]
  \end{align}

  We now show that know that the measure on $\ks \times\mathcal{G}$:  $\lambda(\d s)\frac{1}{\lambda(\varepsilon)}V_{\varepsilon}^{\pi}(s, g)\lambda(\d g)$ converges weakly to $\lambda(\d s)V^{\pi}(s,\d g)$ when $\varepsilon\rightarrow 0$. We know that:
  \begin{align}
    V^{\pi}_{\varepsilon}(s, g) &= \E\left[\sum_{k\geq 0}\gamma^{k}R_{\varepsilon}(s_{k}, g)|s_{0}=s\right]
    \\ &= \frac{1}{1-\gamma}\int_{s'\in \mathcal{S}}\nu^{\pi}(\d s'|s, g)R_{\varepsilon}(s', g)
    \\ &= \int_{s'\in \mathcal{S}}\nu^{\pi}(\d s'|s, g)\1_{\|\phi(s')-g\|\leq \varepsilon}
  \end{align}
  We use the change of variable $g' = \phi(s')$, and we have (with $\phi_{\ast}$ the push-forward operator):
  \begin{align}
    V^{\pi}_{\varepsilon}(s, g) &=  \int_{g'\in \mathcal{S}}(\phi_{\ast}\nu^{\pi})(\d g'|s, g)\1_{\|g'-g\|\leq \varepsilon}
    \\ &= \int_{g'\in \mathcal{S}}M^{\pi}(s, g, \d g')\1_{\|g'-g\|\leq \varepsilon}
    \\ &= M^{\pi}(s, g, B(g, \varepsilon))
  \end{align}

  Let $F \deq \{g\in \mathcal{G}, \inf_{s\in\ks} \|g-\phi(s)\| < 1\}$.  Therefore, for every $\varepsilon < 1$, the support of $\lambda(\d s)\frac{1}{\lambda(\varepsilon)}V_{\varepsilon}^{\pi}(s, g)\lambda(\d g)$ is compact and is a subset of $F$. Let $f(s, g)$ be a continuous bounded test function and $0 < \varepsilon < 1$. We have:
  \begin{align}
   \int_{s\in \ks, g\in \mathcal{G}}f(s, g)\frac{1}{\lambda(\varepsilon)}V_{\varepsilon}^{\pi}(s, g)\lambda(\d s, \d g)  &=
      \int_{s\in \ks, g\in F}f(s, g)\frac{1}{\lambda(\varepsilon)}M^{\pi}(s, g, B(g, \varepsilon))\lambda(\d s, \d g) 
  \end{align}
  We know that $M^{\pi}(s, g, \d g') = \delta_{\phi(s)}(\d g') + \tilde m^{\pi}(s, g, g') \lambda(\d g')$. Therefore, $M^{\pi}(s, g, B(g, \varepsilon)) = \1_{\|g-\phi(s)\|\leq \varepsilon} + \int_{g'} \frac{\1_{\|g-g'\|\leq \varepsilon}}{\lambda(\varepsilon)}\tilde m^{\pi}(s, g, g')$, and we have: 
 \begin{align}
   \int_{s\in \ks, g\in \mathcal{G}}&f(s, g)\frac{1}{\lambda(\varepsilon)}V_{\varepsilon}^{\pi}(s, g)\lambda(\d s, \d g)  =
   \\ &= \int_{s\in \ks, g\in F}\frac{\1_{\|g-\phi(s)\|\leq\varepsilon}}{\lambda(\varepsilon)}f(s, g) \lambda(\d s, \d g)   
         + \int_{s\in \ks, g\in F, g'\in \mathcal{G}}f(s, g)\frac{\1_{\|g-g'\|\leq \varepsilon}}{\lambda(\varepsilon)}\tilde m^{\pi}(s, g, g') \lambda(\d s, \d g, \d g')
    \\ &= \int_{u}\left(\int_{s\in \ks}\lambda(\d s)f(s, \phi(s)+u) + \int_{s\in\ks, g\in \mathcal{G}}f(s, g)\tilde m^{\pi}(s, g, g+u) \lambda(\d s, \d g)\right)U_{\varepsilon}(\d u)
 \end{align}
  
 where $U_{\varepsilon}(\d u)$ is the uniform measure on $B(0, \varepsilon)$ the ball of size $\varepsilon$ around $0$: $U_{\varepsilon}(\d u) \deq \frac{\1_{\|u\|\leq\varepsilon}}{\lambda(\varepsilon)}\lambda(\d u)$. We can switch the order of integration because $f$, $\tilde m^{\pi}$ are continuous, bounded, and the integral is computed on compact sets. The function $u\rightarrow \int_{s\in \ks}\lambda(\d s)f(s, \phi(s)+u) + \int_{s\in\ks, g\in F}f(s, g)\tilde m(s, g, g+u) \lambda(\d s, \d g)$ is bounded and continuous. Since $U_{\varepsilon}(\d u)$ converges weakly to $\delta_{0}(\d u)$, we have:

  \begin{align}
    \lim_{\varepsilon\rightarrow 0}&\int_{s\in \ks, g\in \mathcal{G}}f(s, g)\frac{1}{\lambda(\varepsilon)}V_{\varepsilon}^{\pi}(s, g)\lambda(\d s, \d g)   = \\ &= \lim_{\varepsilon\rightarrow 0}\int_{u}  \left(\int_{s\in \ks}\lambda(\d s)f(s, \phi(s)+u) + \int_{s\in\ks, g\in F}f(s, g)\tilde m^{\pi}(s, g, g+u) \lambda(\d s, \d g)\right) U_{\varepsilon}(\d u)
    \\ &= \int_{u}  \left(\int_{s\in \ks}\lambda(\d s)f(s, \phi(s)+u) + \int_{s\in\ks, g\in \mathcal{G}}f(s, g)\tilde m^{\pi}(s, g, g+u) \lambda(\d s\d g)\right) \delta_{0}(\d u)
    \\ &= \int_{s\in\ks}\lambda(\d s) f(s, \phi(s))  + \int_{s\in\ks, g}f(s, g)\tilde m^{\pi}(s, g, g) \lambda(\d s, \d g)
    \\ &= \int_{s\in \ks, g}f(s, g)V^{\pi}(s, \d g)\lambda(\d s)
  \end{align}
  This concludes the proof.
\end{proof}

\paragraph{Obstacles for learning $V^\pi$ directly.}
We briefly show why learning $V^{\pi}$ directly without bias poses technical
issues, stemming from the necessity to work on-policy for $V$ and the
resulting correlation between visited states and goals along trajectories in the training
set. As a result,
the ``obvious'' analogue of $\ddqn$
for $V$ introduces uncontrolled bias and implicit preferences among all
possible states $s$ that achieve the same goal $g$. This problem
disappears only 
if the correspondence between $s$ and $g$ is one-to-one (e.g.,
$\phi=\Id$).
This is why we learn the more complicated object $M^{\pi}$ instead of
$V^\pi$ in
Section~\ref{sec:unbi-policy-eval}. 

Assume similarly to
Theorem~\ref{thm:app_update_m} that we can sample state-goal pairs from a distribution $\rhoSG(\d s,
\d g)$ over $\mathcal{S}\times\mathcal{G}$, and define the norm $\|\cdot\|_{\rhoSG}$ as
\begin{equation}
  \|V\|_{\rhoSG} = \int_{s, g}\rhoSG(\d s, \d g)\left(\frac{V(s, \d g)}{\rho_{\mathcal{G}}(\d g)}\right)^{2}
\end{equation}
where $\frac{V(s, \d g)}{\rho_{\mathcal{G}}(\d g)}$ is the density of $V(s, \d g)$ with respect to $\rho_{\mathcal{G}}(\d g)$ (if it does not exist, the norm is infinite). We assume we have a model $V_{\theta}(s, \d g) = v_{\theta}(s, g)\rho_{\mathcal{G}}(\d g)$, a target $V_{\tar}(s, \d g) = v_{\tar}(s, g)\rho_{\mathcal{G}}(\d g)$, and want to estimate:
\begin{equation}
  \frac12\partial_{\theta}\|V_{\theta} - T^{\pi}V_{\tar}\|_{\rhoSG}^{2}
\end{equation}
where $T^{\pi}V(s, \d g) = \delta_{\phi(s)}(\d g) + \gamma \E_{s'\sim P^{\pi}(.|s, g)}V(s', \d g)$. Then, informally, we have:
\begin{align}
  \frac12\partial_{\theta}\|V_{\theta} - T^{\pi}V_{\tar}\|_{\rhoSG}^{2} &= \frac12 \partial_{\theta}
  \int_{s, g}\rhoSG(\d s, \d g)\left(
                                                                          \frac{V_{\theta}(s, \d g)}{\rho_{\mathcal{G}}(\d g)} -\frac{TV_{\tar}(s, \d g)}{\rho_{\mathcal{G}}(\d g)} \right)^{2} \\
  &= \frac12 \partial_{\theta}
  \int_{s, g}\rhoSG(\d s, \d g)\left(
  v_{\theta}(s, g) - \frac{TV_{\tar}(s, \d g)}{\rho_{\mathcal{G}}(\d g)} \right)^{2} \\
  &= \int_{s, g}\rhoSG(\d s, \d g)\partial_{\theta}v_{\theta}(s, g)\left(
    v_{\theta}(s, g) -\frac{TV_{\tar}(s, \d g)}{\rho_{\mathcal{G}}(\d g)} \right)
  \\ &= \int_{s, g}\rhoSG(\d s, \d g)\partial_{\theta}v_{\theta}(s, g)
       \left(
         v_{\theta}(s, g) - \gamma \E_{s'\sim P^{\pi}(.|s, g)}\left[v_{\tar}(s', g)\right]
       \right) +
         \\ &+ \int_{s, g}\rhoSG(\d s, \d g)\partial_{\theta}v_{\theta}(s, g)\frac{\delta_{\phi(s)}(\d g)}{\rho_{\mathcal{G}}(\d g)}
\end{align}
If we assume that $\rhoSG(\d s, \d g)$ has a density $\alpha(s, g)$ with
respect to $\rhoSG(\d s)\otimes \rho_{\mathcal{G}}(\d g)$, namely,
$\rhoSG(\d s, \d g) = \alpha(s, g)\rhoSG(\d s)\rho_{\mathcal{G}}(\d g)$,
then the second part, corresponding to the Dirac reward, is equal to:
\begin{align}
  \int_{s, g}\rhoSG(\d s, \d g)\partial_{\theta}v_{\theta}(s, g)\frac{\delta_{\phi(s)}(\d g)}{\rho_{\mathcal{G}}(\d g)} &= \int_{s, g}\rhoSG(\d s )\alpha(s, g)\partial_{\theta}v_{\theta}(s, g)\delta_{\phi(s)}(\d g)
  \\   &= \int_{s}\rhoSG(\d s )\alpha(s, \phi(s))\partial_{\theta}v_{\theta}(s, \phi(s))
\end{align}

If $\alpha(s, g)$ is always equal to $1$, the integral $\int_{s}\rhoSG(\d
s )\partial_{\theta}v_{\theta}(s, \phi(s))$ can be estimated without bias
by sampling $s\sim \rhoSG(\d s)$ and estimating $v_{\theta}(s, \phi(s))$.

However, the case $\alpha(s, g) = 1$ for every $s, g$ corresponds to $s$
and $g$ independent in $\rhoSG$.
This is difficult to realize in practice. Learning
$V$ requires actions to be selected on-policy (term
$\E_{s'\sim P^\pi(.|s,g)}$ above). If we set a goal $g$ and an initial state
$s_0$, and generate an exploration trajectory by following the policy
$\pi(.|.,g)$ for that goal, obviously the states $s$ visited by the
trajectory are going to
be correlated to $g$, by an unknown factor $\alpha$.
Independence could be ensured by re-sampling a
new target goal at each step, independently from the
current state, and selecting the next action
from the policy for this goal. But such an exploration strategy would be
essentially random and would not be efficient. 

Assume we just ignore this problem and
sample exploration trajectories $(g, s_{0}, s_{1}, ...)$ as with other
methods, namely, with $g\sim\rho_{\mathcal{G}}$, $s_{0}\sim \rho_{0}(\d s_{0}|g)$ and $s_{t+1}\sim P^{\pi}(.|s_{t}, g)$, and define the estimate
\begin{equation}
  \widehat{\delta\theta}_{V}(s, s', g) = \partial_{\theta}v_{\theta}(s, \phi(s)) + \partial_{\theta}v_{\theta}(s, g)\left(\gamma v_{\tar}(s', g)  v_{\theta}(s, g)\right)
\end{equation}
similarly to updates of $\ddqn$ or $\dtd$. In that case, we have:
\begin{equation}
  \E_{s, g\sim\rhoSG, s'\sim P^{\pi}(.|s, g)}  \left[\widehat{\delta\theta}_{V}(s, s', g)\right] = \|V_{\theta} - T_{\alpha}^{\pi}V_{\tar}\|_{\rhoSG}
\end{equation}
where:
\begin{equation}
T_{\alpha}^{\pi}V = \alpha(s, g)\delta_{\phi(s)} + \E_{s'\sim
P^{\pi}(.|s, g)}\left[V(s', \d g)\right].
\end{equation}
This is an unbiased estimate of the TD error
with the \emph{rescaled reward} $\alpha(s, g)\delta_{\phi(s)}(\d g)$ instead of $\delta_{\phi(s)}(\d g)$.

If $\mathcal{S} = \mathcal{G}$ and $\phi = \Id$, such a reward rescaling
is not an issue. Indeed, in that case, $\alpha(s,g)\delta_s(\d
g)=\alpha(g,g)\delta_s (\d g)$ as the Dirac measure is nonzero only for
$s=g$. This means that for every goal $g$, the value function for that
goal is rescaled by a constant $\alpha(g,g)$, and we learn
$\alpha(g,g)V(s,\d g)$ instead of $V(s,\d g)$. This does not change the
ranking of state values for each goal $g$, nor the direction of policy
improvement for each goal (but it changes the relative importance
of learning different goals $g$).

On the contrary, if $\mathcal{S}\neq \mathcal{G}$, for a fixed goal $g$,
this implicit reward rescaling can favor some states $s$ over others
among the set of states $s$ achieving this goal ($\phi(s) = g$). 
For instance, 
assume the the agent starts at $s_{0}$ and wants to reach $g$, and that
there are two states $s_{1}, s_{2}$ such that $\phi(s_{1}) = \phi(s_{2})
= g$. Even if $s_{1}$ is easier to reach than $s_{2}$ from $s_0$, the
policy $\pi$ might \emph{prefer} to reach $s_{2}$ because its implicitly
rescaled reward is higher. Therefore, the algorithm could converge to
non-optimal policies and is not unbiased. It would still learn to reach
$g$, but not necessarily in an optimal way.

\subsection{Equivalence Between $\varepsilon \rightarrow 0$ and the Dirac Setting}
\label{app:equiv-rewards}

\begin{defi}
We say that $\pi_{2}$ is \emph{better} than $\pi_{1}$ with infinitely sparse rewards if the two measures $\lambda(\d s)V^{\pi_{1}}(s, \d g)$ and $\lambda(\d s)V^{\pi_{2}}(s, \d g)$ on $\ks\times \mathcal{G}$ satisfy: $\lambda(\d s)V^{\pi_{1}}(s, \d g) \leq \lambda(\d s)V^{\pi_{2}}(s, .)$. 

We say that $\pi_{2}$ is \emph{asymptotically better} than $\pi_{1}$  when $\varepsilon \rightarrow 0$  if for all $s$, $g$, $$\lim\inf_{\varepsilon\rightarrow 0} \frac{V^{\pi_{2}}_{\varepsilon}(s, g)}{V^{\pi_{1}}_{\varepsilon}(s, g)} \geq 1.$$ 
\end{defi}

\begin{thm}
  \label{thm:eq_eps_delta}
  We assume Assumption~\ref{assumption:density} and take $\pi_{1}, \pi_{2}\in\Pi$.

  Then, $\pi_{2}$ is better than $\pi_{1}$ with infinitely sparse rewards if and only if $\pi_{2}$ is \emph{asymptotically better} than $\pi_{1}$  when $\varepsilon \rightarrow 0$. In particular, a policy $\pi^{\ast}$ is an optimal policy with infinitely sparse rewards if and only if it is an optimal policy when $\varepsilon \rightarrow 0$.

\end{thm}

\begin{proof}

  We know that $V^{\pi}(s, \d g) = \delta_{\phi(s)}(\d g) + \tilde m(s, g, g)^\pi\lambda(\d g)$. Moreover:
  \begin{align}
    V^{\pi}_{\varepsilon}(s_{0}, g) &= M(s_{0}, g, B(g, \varepsilon))
    \\ & = \1_{\phi(s_{0})=g} +  \lambda(\varepsilon) \tilde m^{\pi}(s_{0}, g, g) + o(\lambda(\varepsilon))
  \end{align}
  Therefore, for any policies $\pi_{1}, \pi_{2}\in \Pi$:
  \begin{align}
    \frac{V^{\pi_{2}}_{\varepsilon}(s, g)}{V^{\pi_{1}}_{\varepsilon}(s, g)} &= \frac{\1_{\phi(s)=g} + \tilde m^{\pi_{2}}(s, g, g) \lambda(\varepsilon) + o(\lambda(\varepsilon))}{\1_{\phi(s)=g} + \tilde m^{\pi_{1}}(s, g, g) \lambda(\varepsilon) + o(\lambda(\varepsilon))}
    \\ &= \1_{\phi(s)=g} + \1_{\phi(s)\neq g} \frac{\tilde m^{\pi_{2}}(s, g, g)}{\tilde m^{\pi_{1}}(s, g, g)} + o_{\varepsilon\rightarrow 0}(1)
  \end{align}
  Therefore, by definition, $\pi_{2}$   is asymptotically better than $\pi_{1}$ when $\varepsilon\rightarrow 0$ if and only if, for all $(s, g)\in\mathcal{S}\times\mathcal{G}$:
  \begin{equation}
    \1_{\phi(s)=g} + \1_{\phi(s)\neq g} \frac{\tilde m^{\pi_{2}}(s, g, g)}{\tilde m^{\pi_{1}}(s, g, g)} \geq 1
  \end{equation}
  If $\phi(s)\neq g$, this inequality is equivalent to  $\tilde m^{\pi_{2}}(s, g, g)\geq m^{\pi_{1}}(s, g, g)$. Because $\tilde m^{\pi_{1}}$ and $\tilde m^{\pi_{2}}$ are continuous, $\tilde m^{\pi_{2}}(s, g, g)\geq m^{\pi_{1}}(s, g, g)$ for all $\phi(s)\neq g$ is equivalent to $\tilde m^{\pi_{2}}(s, g, g)\geq m^{\pi_{1}}(s, g, g)$ for every $(s, g)$. Therefore, $\pi_{2}$ is asymptotically better than $\pi_{1}$ when $\varepsilon\rightarrow 0$ if and only if, for all $(s, g)$, $m^{\pi_{2}}(s, g, g)\geq m^{\pi_{1}}(s, g, g)$.

On the other side $\pi_{2}$ is better than $\pi_{1}$ with infinitely sparse rewards if and only if:
\begin{align}
  & & \lambda(\d s)V^{\pi_{1}}(s, \d g) &\preceq \lambda(\d s)V^{\pi_{2}}(s, \d g)
  \\ &\Leftrightarrow & \lambda(\d s)\delta_{\phi(s)}(\d g) + m_{\pi_{1}}(s, g, g)\lambda(\d s, \d g) &\preceq \lambda(\d s)\delta_{\phi(s)}(\d g) + m_{\pi_{2}}(s, g, g)\lambda(\d s, \d g)
  \\ &\Leftrightarrow  & m_{\pi_{1}}(s, g, g)\lambda(\d s, \d g) &\leq m_{\pi_{2}}(s, g, g)\lambda(\d s, \d g) 
\end{align}
Therefore, for $\lambda$-almost every $(s, g)$, $m_{\pi_{1}}(s, g, g)\leq m_{\pi_{2}}(s, g, g)$.
Therefore: $\pi_{2}$ is better than $\pi_{1}$ with infinitely sparse rewards if and only if $\tilde m^{\pi_{2}}(s, g, g)\geq m^{\pi_{1}}(s, g, g)$ for $\lambda$-almost every $s, g$. This concludes the proof.

\end{proof}

\bigskip
In the following statement, we introduce 3 definitions of expected return: the return $J(\pi)$ with infinitely sparse reward, the return $J_{\varepsilon}(\pi)$ with sparse reward $R_{\varepsilon}$, and the estimated return $J_{n}(\pi)$ with the value measure approximator $v_{n}$. Then, we show that these three definitions are consistent.

\begin{thm}
  \label{thm:eq_J}
  We define $J(\pi)$, the expected return with infinitely sparse rewards
  for the goal density $p_{\mathcal{G}}$, as: 
  \begin{equation}
J(\pi) \deq \int_{s_{0}, g}\lambda(\d s_{0})p_{\mathcal{G}}(g)p_{0}(s_{0}|g)V^{\pi}(s_{0}, \d g).
\end{equation}

  We consider the expected return for the reward $R_\varepsilon$ and the goal distribution $\rho(\d g)$. 
  \begin{align}
    J_{\varepsilon}(\pi) &= \E_{g\sim\rho(\d g), s_{0}, a_{0}, ...}\left[\sum_{t\geq 0}\gamma^{t}R_{\varepsilon}(s_{t}, g)\right] = \int_{g, s_{0}}p_{\mathcal{G}}(g)\lambda(\d s_{0}, \d g)p_{0}(s_{0}|s)V_{\varepsilon}(s_{0}, g)
  \end{align}
Let  $(\hat v_{n}(s, g))_{n\geq 0}$ be any sequence of densities on $\mathcal{S}\times\mathcal{G}$ such that the measure on $\mathcal{S}\times\mathcal{G}$: $\lambda(\d s)\hat v_{n}(s, g)\rho_{\mathcal{G}}(\d g)$ converges weakly to $\lambda(\d s)V^{\pi}(s, \d g)$. We define $\tilde \rho(\d g) \deq \frac{1}{c}p^{2}_{\mathcal{G}}(g)\lambda(\d g)$ with $c \deq \int_{g}p^{2}_{\mathcal{G}}(g)\lambda(\d g)$, and $J_{n}(\pi)$  the estimator of the average return for the goal distribution $\tilde \rho$ with estimator $\hat v_{n}$:
  \begin{align}
    J_{n}(\pi) &\deq \E_{g\sim\tilde \rho(\d g), s_{0}\sim p(s_{0}|g)}\left[\hat v_{n}(s_{0}, g)\right]
  \end{align}
  Then the two estimators $J_{n}$ and $J_{\varepsilon}$ converge to $J$:
  \begin{align}
    \frac{1}{\lambda(\varepsilon)}J_{\varepsilon} (\pi)&\rightarrow_{\varepsilon\rightarrow 0} J(\pi)
    \\ c J_{n}(\pi) &\rightarrow_{n\rightarrow\infty}  J(\pi)
  \end{align}

\end{thm}

\begin{proof}
  We have:
  \begin{align}
    J_{\varepsilon}(\pi) &= \int_{s_{0}, g}V_{\varepsilon}(s_{0}, g)p_{\mathcal{G}}(g)p_{0}(s_{0}|g)\lambda(\d s_{0}, \d g)
    \\ J_{n}(\pi) &= \int_{s_{0}, g}\hat v_{n}(s_{0}, g)\frac{1}{c}p^{2}_{\mathcal{G}}(g)p_{0}(s_{0}|g)\lambda(\d s_{0}, \d g)
  \end{align}
  and whe know from Theorem~\ref{thm:value-measure-continuous} that $\frac{V_{\varepsilon}(s, g)}{\lambda(\varepsilon)}\lambda(\d s, \d g)$ and $\hat v_{n}(s, g)p_{\mathcal{G}}(\d g)\lambda(\d s, \d g)$ converge weakly to $\lambda(\d s)V^{\pi}(s, \d g)$ on $\ks\times\mathcal{G}$ when $\varepsilon\rightarrow 0$ and $n\rightarrow\infty$. Therefore, because $p_{0}$ and $p_{\mathcal{G}}$ are continuous bounded functions,
  \begin{align}
    \lim_{\varepsilon\rightarrow 0}\frac{1}{\lambda(\varepsilon)}J_{\varepsilon}(\pi) = \int_{s_{0}, g}V^{\pi}(s_{0}, \d g)p_{\mathcal{G}}(g)p_{0}(s_{0}|g)\lambda(\d s_{0})
    = J(\pi)
  \end{align}

  Similarly: 
  \begin{align}
    \lim_{n\rightarrow\infty}J_{n}(\pi) &= \int_{s_{0}, g}V^{\pi}(s, \d g)\frac{1}{c}p_{\mathcal{G}}(g)p_{0}(s_{0}|g)\lambda(\d s_{0})
    \\ &= \frac{1}{c}J(\pi)
  \end{align}

\end{proof}

\begin{prop}
  We assume Assumption~\ref{assumption:density}. Moreover, we assume that $p_{\mathcal{G}}(g) > 0$ for every $g\in\phi(\ks)$, and $p_{0}(s_{0}|g) > 0$ for every $(s_{0}, g)\in\ks\times\mathcal{G}$.

  We consider the partial order $\prec$ defined as:  $\pi_{1}\prec \pi_{2}$ if $\pi_{2}$ is strictly better than $\pi_{1}$ with infinitely sparse rewards: $\lambda(\d s)V^{\pi_{1}}(s, \d g) \prec \lambda(\d s)V^{\pi_{2}}(s, \d g)$ on $\ks \times \mathcal{G}$.


  Then $\pi\mapsto J(\pi)$ is strictly increasing for  $\prec$.
  
  \end{prop}

  \begin{proof}
    The function $J(\pi)$ is clearly non-decreasing, and we  have to check that we cannot have $\pi_{1} \prec \pi_{2}$ with $J(\pi_{1}) = J(\pi_{2})$. Let $\pi_{1}, \pi_{2}\in \Pi$ such that $\pi_{1}\prec \pi_{2}$.  Therefore, there is $U \subset \ks\times\mathcal{G}$ such that $(\lambda \otimes V^{\pi_{2}}(., .))(U) > (\lambda \otimes V^{\pi_{1}}(., .))(U)$. Moreover, because $\supp\left(\lambda(\d s) V^{\pi}(s, \d g)\right)\subset \ks \times \phi(\ks)$, therefore we can suppose $U\subset \ks\times \phi(\ks)$, and we have:
    \begin{align}
      \int_{(s, g)\in U}\lambda(\d s, \d g) (\tilde m^{\pi_{2}}(s, g, g) - \tilde m^{\pi_{1}}(s, g, g)) > 0
    \end{align}
   
We already know that  $\tilde m^{\pi_{2}}(s, g, g) - \tilde m^{\pi_{1}}(s, g, g) \geq 0$ for almost every $s, g$ (see proof of Theorem~\ref{thm:eq_eps_delta}). Therefore, there is $\varepsilon' > 0$ and $V\subset U$ with $\d\lambda(V) >0$ such that for every $s, g\in V$, $\tilde m^{\pi_{2}}(s, g, g) - \tilde m^{\pi_{1}}(s, g, g) > \varepsilon'$.

  We have:
  \begin{align}
    J(\pi_{2}) - J(\pi_{1}) &= \int_{s\in\ks, g\in\mathcal{G}}p_{0}(s|g)p_{\mathcal{G}}(g)(V^{\pi_{2}}(s, \d g)- V^{\pi_{1}}(s, \d g)
    \\ &\geq \int_{(s, g)\in V}p_{0}(s|g)p_{\mathcal{G}}(g)\left(\tilde m^{\pi_{2}}(s, g, g) - \tilde m^{\pi_{1}}(s, g, g)\right)
     \\  &\geq \varepsilon'\int_{(s, g)\in V}p_{0}(s|g)p_{\mathcal{G}}(g)
         \\ &> 0
  \end{align}
  because $p_{\mathcal{G}}(g) > 0$ for $\lambda$-almost every $g$ in $\phi(\ks)$, and $p_{0}(s|g)  > 0$ for $\lambda$-almost every $s, g$ in $\ks\times\mathcal{G}$. This concludes the proof.
  
\end{proof}

\subsection{Policy Gradient}
\label{app:policy-gradient}

\begin{thm}[ (Formal statement of Informal Theorem~\ref{nfthm:pol_grad})]
  \label{thm:pol_grad}
  Let $\pi_{\theta}(a|s, g)$ be a parametrized goal-dependent policy, defined for every $\theta\in\Theta$. We assume that for every $\theta\in\Theta, s\in\mathcal{S}, g\in\mathcal{G}, a\in\mathcal{A}$, $\pi_{\theta}(a|s, g) >0$. Moreover, we assume $\pi_{\theta}(a|s, g)$ is a continuous function of $a, s, g, \theta$, and continuously differentiable with respect to $\theta$. 

 We define $\tilde \rho(\d g) \deq \frac{1}{c}p^{2}_{\mathcal{G}}(g)\lambda(\d g)$ with $c \deq \int_{g}p^{2}_{\mathcal{G}}(g)\lambda(\d g)$. We assume access to samples $g\sim \tilde \rho(\d g)$, $s_{0}\sim\rho_{0}(\d s|g) = p_{0}(s_{0}|g)\lambda(\d s_{0})$,  $s\sim \nu^{\pi_{\theta}}(\ s|s_{0}, g)$, $a \sim \pi(a|s, g)$ and $s' \sim P(\d s'|s, a)$. Let $(\hat v_{n}(s, g))_{n\geq 0}$ be a sequence of densities, such that $\lambda(\d s)\hat v_{n}(s, g)\rho(\d g)$ converges weakly to $\lambda(\d s)V^{\pi_{\theta}}(s, \d g)$. We define the stochastic actor critic $\widehat{\delta\theta}_{\dac}^{(n)}$ for estimate $n$ as:
  \begin{equation}
        \widehat{\delta\theta}_{\dac}^{(n)}(s, a, s', g) \deq \partial_{\theta}\log \pi_{\theta}(a|s, g)\left(\gamma \hat v_{n}(s', g) - \hat v_{n}(s, g)\right)
      \end{equation}
      Then, we have:
      \begin{equation}
         \lim_{n\rightarrow\infty}  \E_{g\sim\tilde\rho, s\sim \nu^{\pi}(.|g), a\sim\pi_{\theta}(.|s, g), s'\sim P(.|s, a)}\left[\widehat{\delta\theta}_{\dac}^{(n)}(s, a, s', g)\right] = \frac{1-\gamma}{c} \partial_{\theta}J(\pi_{\theta})
      \end{equation}

      Moreover, we have:
      \begin{equation}
        \lim_{\varepsilon\rightarrow 0}\frac{1}{\lambda(\varepsilon)}\partial_{\theta}J_{\varepsilon}(\pi_{\theta}) = \partial_{\theta}J(\pi_{\theta})
      \end{equation}
\end{thm}

\begin{proof}
  We first compute $\partial_{\theta}J(\pi_{\theta})$.  We have:
\begin{align}
  J(\pi_{\theta}) = \int_{s_{0}, g}V^{\pi_{\theta}}(s_{0}, \d g)p_{\mathcal{G}}(g)p_{0}(s_{0}|g)\lambda(\d s_{0})
\end{align}
We know that $V^{\pi}(s, \d g) = \delta_{\phi(s)}(\d g) + \tilde m^{\pi}(s, g, g)\lambda(\d g)$. We define for simplicity $v^{\pi}(s, g) = \tilde m^{\pi}(s, g, g)$. Moreover, we know, by taking $g'=g$ in Equation~\eqref{eq:bellman_tildem} in Lemma~\ref{lem:density} that for every $(s, g)$, we have:
\begin{equation}
    v^{\pi_{\theta}}(s, g) = \gamma\int_{a}\lambda(\d a)\pi(a|s, g)\left(\tilde p(g|s, a) + \int_{s'}\lambda(\d s')p(s'|s, a) v^{\pi_{\theta}}(s', g)\right)
  \end{equation}


We define $F(s, g, \theta) = \gamma \int_{a}\pi_{\theta}(a|s, g)\tilde p(g|s, a)$. The function $F_{\theta}$ is continuous in $s$ and $g$ and continuously differentiable in $\theta$, because $\tilde p$ is and $\pi_{\theta}$ are continuous, $\pi_{\theta}$ is continuously differentiable, and $\mathcal{A}$ is compact.
From the proof of Equation~\eqref{eq:bellman_tildem} in Lemma~\ref{lem:density}, we know that $F(s, g, \theta)$ is the density of $\gamma\int_{a, s'}\pi(a|s, g)p(s'|s, a)\delta_{\phi(s')}(\d g)$ with respect to the Lebesgue measure $\lambda(\d g)$. This remark will be used later in the computation. 
  We now have:
\begin{align}
  \label{eq:fixedpoint_v}
  v^{\pi_{\theta}}(s, g) = F(s, g, \theta) + \gamma \int_{a, s'}\pi_{\theta}(a|s, g)p(s'|s, a)v^{\pi_{\theta}}(s', g)\lambda(\d a, \d s')
\end{align}
Therefore: 
\begin{align}
  \label{eq:solution_fixedpoint_v}
  v^{\pi_{\theta}}(s, g) &= F(s, g, \theta) + \sum_{k\geq 1}\gamma^{k} \int_{a_{0}, s_{1}, ...}\lambda(\d a_{0}, \d s_{1}, ..., \d s_{k})\left(\prod_{i=0}^{k-1}\pi_{\theta}(a_{i}|s_{i}, g)p(s_{i+1}|s_{i}, a_{i}) \right)F(s_{k}, g, \theta) 
\end{align}
because it is a fixed point of $v^{\pi}$ equation, and is the only fixed point which is continuous and bounded, because the space is compact, and $\pi_{\theta}$, $p$ are continuous an bounded. 

Equation~\eqref{eq:solution_fixedpoint_v} can also be written:
\begin{align}
  v^{\pi_{\theta}}(s, g) &= \frac{1}{1-\gamma}\int_{s'}\nu^{\pi_{\theta}}(\d s'|s, g)F(s', g, \theta) 
\end{align}


Because $F(s', g, \theta)$ is continuously differentiable in $\theta$ and the support of $\nu^{\pi}$ is compact, $v^{\pi_{\theta}}(s, g)$ is differentiable. We will now now derive a fixed point equation on $\partial_{\theta}v^{\pi_{\theta}}$. We differentiate equation~\eqref{eq:fixedpoint_v} and we get:
\begin{align}
  \partial_{\theta}v^{\pi_{\theta}}(s, g) = \partial_{\theta}F(s, g, \theta) &+ \gamma\int_{a, s}\lambda(\d a, \d s)\partial_{\theta} \pi_{\theta}(a|s, g)p(s'|s, a)v^{\pi_{\theta}}(s', g) +
    \\ &+ \gamma\int_{a, s}\lambda(\d a, \d s) \pi_{\theta}(a|s, g)p(s'|s, a)\partial_{\theta}v^{\pi_{\theta}}(s', g) 
\end{align}
We define $G(s, g, \theta) \deq \partial_{\theta}F(s, g, \theta) + \gamma\int_{a, s'}\lambda(\d a, \d s')\partial_{\theta} \pi_{\theta}(a|s, g)p(s'|s, a)v^{\pi_{\theta}}(s', g)$. We have:
\begin{align}
  \partial_{\theta}v^{\pi_{\theta}}(s, g) &= G(s, g, \theta) + \gamma\int_{a, s'} \lambda(\d a, \d s')\pi_{\theta}(a|s, g)p(s'|s, a)\partial_{\theta}v^{\pi_{\theta}}(s', g) 
\end{align}
Similarly to the derivation of $v^{\pi}$ from its fixed point equation (from \eqref{eq:fixedpoint_v} to \eqref{eq:solution_fixedpoint_v}):
\begin{align}
  \partial_{\theta}v^{\pi_{\theta}}(s, g) &= G(s, g, \theta) + \sum_{k \geq 1}\gamma^{k}\int_{a_{0}, s_{1}, ...}\lambda(\d a_{0}, \d s_{1}, ..., \d s_{k})\left(\prod_{i=0}^{k-1}\pi_{\theta}(a_{i}|s_{i}, g)p(s_{i+1}|s_{i}, a_{i}) \right)G(s_{k}, g, \theta)
  \\ &= \frac{1}{1-\gamma}\int_{s'}\nu^{\pi_{\theta}}(\d s'|s, g)G(s', g, \theta)
\end{align}
We now compute $\partial_{\theta}J(\pi_{\theta})$. We have:
\begin{align}
  \partial_{\theta}J(\theta) &= \partial_{\theta}\left(\int_{s_{0}, g}\lambda(\d s_{0})p_{\mathcal{G}}(g)p_{0}(s_{0}|g)V^{\pi_{\theta}}(s_{0}, \d g)\right)
  \\ &= \partial_{\theta}\left(\int_{s_{0}, g}\lambda(\d s_{0})p_{\mathcal{G}}(g)p_{0}(s_{0}|g)\left(\delta_{\phi(s_{0})}(\d g) + v^{\pi_{\theta}}(s_{0}, g)\lambda(\d g)\right)\right)
  \\ &= \partial_{\theta}\left(\int_{s_{0}, g}\lambda(\d s_{0}, \d g)p_{\mathcal{G}}(g)p_{0}(s_{0}|g)v^{\pi_{\theta}}(s_{0}, g)\right)
  \\ &=  \int_{s_{0}, g}\lambda(\d s_{0}, \d g) p_{\mathcal{G}}(g)p_{0}(s_{0}|g)
       \partial_{\theta}v^{\pi_{\theta}}(s_{0}, g)
  \\ &= \frac{1}{1-\gamma}\int_{s_{0}, s, g}\lambda(\d s_{0}, \d g) p_{\mathcal{G}}(g)p_{0}(s_{0}|g)
       \nu^{\pi_{\theta}}(\d s|s_{0}, g)G(s, g, \theta)
       \label{eq:partialv_solvedfixedpoint}
\end{align}

We now show that:
\begin{align}
  \label{eq:g_closedform}
  G(s, g, \theta)\lambda(\d g) = \gamma\int_{s', a}V(s', \d g)\partial_{\theta}\pi_{\theta}(a|s, g)p(s'|s, a)
\end{align}
While this result might seem to come out of nowhere, remember that $F(s, g, \theta)$ was derived above as the measure density of $\gamma\int_{s', a}\pi(a|s, g)p(s'|s, g)\delta_{\phi(s')}(\d g)$ with respect to Lebesgue measure. With the following informal computation, we have: 
\begin{align}
  G(s, g, \theta)\lambda(\d g) &= \lambda(\d g)\partial_{\theta}\frac{1}{\lambda(\d g)}\int_{s', a}\lambda(\d s', \d a)\gamma\pi_{\theta}(a|s, g)p(s'|s, a)\delta_{\phi(s')}(\d g) + \gamma\int_{s', a}\lambda(\d s', \d a)v^{\pi_{\theta}}(s',  g)\partial_{\theta}\pi_{\theta}(a|s, g)p(s'|s, a)
  \\ &= \int_{s', a}\lambda(\d s', \d a)\gamma\partial_{\theta}\pi_{\theta}(a|s, g)p(s'|s, a)\left(\delta_{\phi(s)}(\d g) + v^{\pi_{\theta}}(s', g)\lambda(\d g)\right)
       \\ &= \int_{s', a}\lambda(\d s', \d a)\gamma V^{\pi_{\theta}}(s', \d g)\partial_{\theta}\pi_{\theta}(a|s, g)p(s'|s, a)
\end{align}

This derivation is informal because we differentiated through a density: we use $\lambda(\d g)\partial_{\theta}\frac{1}{\lambda(\d g)} = \partial_{\theta}$. We now derive the result rigorously. Let $f(g)$ be a continuous test function. We have: 
\begin{align}
  &\int_{g}f(g)G(s, g, \theta)\lambda(\d g) =
                                             \\ &=\int_{g}\lambda(\d g)f(g)\left(\gamma \int_{a}\lambda(\d a)\partial_{\theta}\pi_{\theta}(a|s, g)\tilde p(g|s, a) + \gamma\int_{a, s'}\lambda(\d s', \d a)\partial_{\theta} \pi_{\theta}(a|s, g)p(s'|s, a)v^{\pi_{\theta}}(s', g)\right)
\end{align}
We consider the first part. The following is the reversed derivation of $\tilde p$ in Equations~\eqref{eq:tildep_der_1}-\eqref{eq:tildep_der_2}.We have: 
\begin{align}
  \int_{g}\lambda(\d g)f(g)\left(\gamma \int_{a}\partial_{\theta}\pi_{\theta}(a|s, g)\tilde p(g|s, a) \right) &= \gamma \int_{g, a}\lambda(\d g, \d a)f(g)\partial_{\theta}\pi_{\theta}(a|s, g)\tilde p(g|s, a)
  \\ &= \gamma \int_{g, a, s'}\lambda(\d a, \d s')f(g)\partial_{\theta}\pi_{\theta}(a|s, g) p(s'|s, a)\delta_{\phi(s')}(\d g)
\end{align}
Therefore:
\begin{align}
  \int_{g}f(g)G(s, g, \theta)\lambda(\d g) &= \gamma\int_{g, a, s'}\lambda(\d a, \d s')f(g) \partial_{\theta}\pi_{\theta}(a|s, g)p(s'|s, a) \left(\delta_{\phi(s')}(\d g) + v^{\pi_{\theta}}(s', g) \lambda(\d g)\right)
  \\ &= \gamma\int_{g, a, s'}\lambda(\d a, \d s')f(g)\partial_{\theta}\pi_{\theta}(a|s, g)p(s'|s, a) V^{\pi_{\theta}}(s', \d g)
\end{align}
This establishes equation~\eqref{eq:g_closedform}. Finally, from \eqref{eq:partialv_solvedfixedpoint} and \eqref{eq:g_closedform}, we have:
\begin{align}
  \partial_{\theta}J(\pi_{\theta}) = \frac{1}{1-\gamma}\int_{g, s_{0}, s, a, s'}\lambda(\d s_{0}) \gamma p_{\mathcal{G}}(g)p_{0}(s_{0}|g)\nu^{\pi_{\theta}}(\d s|s_{0}, g)\partial_{\theta}\pi_{\theta}(a|s, g)p(s'|s, a)   V^{\pi_{\theta}}(s', \d g)      
\end{align}

\bigskip
We now show that  Then, we have:
         $\lim_{n\rightarrow\infty}  \E\left[\widehat{\delta\theta}_{\dac}^{(n)}(s, a, s', g)\right] = \frac{1-\gamma}{c} \partial_{\theta}J(\pi_{\theta})$
and
$\lim_{\varepsilon\rightarrow 0}\frac{1}{\lambda(\varepsilon)}\partial_{\theta}J_{\varepsilon}(\pi_{\theta}) = \partial_{\theta}J(\pi_{\theta})$.

We first compute $\partial_{\theta}J_{\varepsilon}(\pi_{\theta})$. We apply the policy gradient theorem \citep{sutton2018reinforcement} to the augmented state augmented (non-multi goal) environment $\tilde{\mathcal{S}} = \mathcal{S}\times\mathcal{G}$, and we have, for any \emph{baseline} function $b(\tilde s)$ with $\tilde s\in \tilde{\mathcal{S}}$:
\begin{align}
  \partial_{\theta}J_{\varepsilon}(\pi_{\theta}) &= \frac{1}{1-\gamma}\int_{\tilde s_{0}, \tilde s, a, \tilde s'}\lambda(\d a) \rho_{0}(\tilde s_{0})\nu^{\pi_{\theta}}(\d \tilde s|\tilde s_{0})\tilde P(\d \tilde s'|\tilde s, a)\partial_{\theta}\pi_{\theta}(a|\tilde s)\left(R_{\varepsilon}(\tilde s) + \gamma V_{\varepsilon}^{\pi}(\tilde s') - b(\tilde s)\right)
  \\ &= \frac{1}{1-\gamma}\int_{g, s_{0},  s, a,  s'}\lambda(\d a) \rho_{\mathcal{G}}(\d g)\rho_{0}(\d s_{0}|g)\nu^{\pi_{\theta}}(\d s|s_{0}, g)\tilde P(\d s'|s, a)\partial_{\theta}\pi_{\theta}(a|s, g)\left(R_{\varepsilon}(s, g) + \gamma V_{\varepsilon}^{\pi}(s', g) - b(s, g)\right)
\end{align}
with the change of variable $\tilde s = (s, g)$, $\tilde s' = (s', g)$, $\tilde s_{0} = (s_{0}, g)$. We use the baseline $b(s, g) = R_{\varepsilon}(s, g)$, and we have:
\begin{align}
  \frac{1}{\lambda(\varepsilon)}\partial_{\theta}J_{\varepsilon}(\pi_{\theta}) &= \frac{1}{1-\gamma}\int_{s_{0}, s, a, s', g}\lambda(\d s_{0}, \d a, \d g)p_{\mathcal{G}}(g)p(s_{0}|g)\nu^{\pi}(\d s|s_{0}, g) \partial_{\theta}\pi_{\theta}(a|s, g) \left( \frac{\gamma V_{\varepsilon}(s', g)}{\lambda(\varepsilon)} \right)
\end{align}

We now compute $\E\left[\widehat{\delta\theta}_{\dac}^{(n)}(s, a, s', g)\right]$. We have:
\begin{align}
  &\E\left[\widehat{\delta\theta}_{\dac}^{(n)}(s, a, s', g)\right]
    \\ &= \int_{s_{0}, s, a, s', g}\lambda(\d g, \d s_{0}, \d s', \d a)\frac{1}{c}p_{\mathcal{G}}(g)^{2}p(s_{0}|g)\nu^{\pi}(\d s|s_{0}, g)\pi_{\theta}(a|s, g) \partial_{\theta}\log \pi_{\theta}(a|s, g) (\gamma v_{n}(s', g) - v_{n}(s, g))
  \\ &= \int_{s_{0}, s, a, s', g}\lambda(\d g, \d s_{0}, \d s', \d a)\frac{1}{c}p_{\mathcal{G}}(g)^{2}p(s_{0}|g)\nu^{\pi}(s|s_{0}, g) \partial_{\theta}\pi_{\theta}(a|s, g) (\gamma v_{n}(s', g) - v_{n}(s, g))
\end{align}

We know that for every \emph{baseline} function $b(s,g)$:
\begin{align}
  \int_{a} \partial_{\theta} \pi_{\theta}(a|s, g)b(s, g) = b(s, g)\partial_{\theta}\int_{a}  \pi_{\theta}(a|s, g) = 0
\end{align}
We define $b(s, g) =  v_{n}(s, g)$, and we have:
\begin{align}
  \E\left[\widehat{\delta\theta}_{\dac}^{(n)}(s, a, s', g)\right] = \gamma \int_{s_{0}, s, a, s', g}\lambda(\d g, \d s_{0}, \d s', \d a)\frac{1}{c}p_{\mathcal{G}}(g)^{2}p_{0}(s_{0}|g)\nu^{\pi}(\d s|s_{0}, g) \partial_{\theta} \pi_{\theta}(a|s, g)p(s'|s, a)  v_{n}(s', g) 
\end{align}

We know from Lemma~\ref{lem:density} that $\nu^{\pi}(\d s|s_{0}, g) = (1-\gamma)\delta_{s_{0}}(\d s) + q^{\pi}(s|s_{0}, g)\lambda(\d s)$ where $q^{\pi}$ is continuous, bounded, and with compact support as a density. Therefore, for any goal $g$, if we take the expectation with respect to $s_{0} \sim p_{0}(s_{0}|g)$:
\begin{align}
  \int_{s_{0}}p_{0}(s_{0}|g)\nu^{\pi}(\d s|s_{0}, g) &= \int_{s_{0}}(1-\gamma)p_{0}(s_{0}|g)\delta_{s_{0}}(\d s) + q^{\pi}(s|s_{0}, g)\lambda(\d s)
  \\ &= \left((1-\gamma)p_{0}(s|g) + \int_{s_{0}}p_{0}(s_{0}|g) q^{\pi}(s|s_{0}, g)\right)\lambda(\d s)
       \\ &= \tilde q^{\pi}(s|g)\lambda(\d s)
\end{align}
where $\tilde q^{\pi}$ is continuous, bounded and with compact support as a density.
Moreover, $p_{\mathcal{G}}$ and $p_{0}(s_{0}|g)$ are continuous bounded functions. Therefore:
\begin{align}
  \label{eq:final_dtheta_n}
  \E\left[\widehat{\delta\theta}_{\dac}^{(n)}(s, a, s', g)\right] = \gamma\int_{s', g} \lambda(\d s', \d g)v_{n}(s', g)\frac{1}{c}p_{\mathcal{G}}(g)^{2}\int_{s, a}\lambda(\d s, \d a)\tilde q(s, g) p(s'|s, a)\partial_{\theta} \pi_{\theta}(a|s, g)
\end{align}
and similarly:
\begin{align}
  \frac{1}{\lambda(\varepsilon)}\partial_{\theta}J_{\varepsilon}(\pi_{\theta}) &=  \frac{\gamma}{1-\gamma} \int_{s, a, s', g}\lambda(\d s, \d a, \d s', \d g)p_{\mathcal{G}}(g)\tilde q(s| g)\partial_{\theta}\pi(a|s, g) p(s'|s, a)  \frac{V_{\varepsilon}(s', g)}{\lambda(\varepsilon)}
  \\ &= \frac{\gamma}{1-\gamma} \int_{s', g}\lambda(\d s', \d g)\frac{V_{\varepsilon}(s', g)}{\lambda(\varepsilon)}p_{\mathcal{G}}(g)\int_{s, a}\lambda(\d s, \d a)\tilde q(s, g) p(s'|s, a)\partial_{\theta} \pi_{\theta}(a|s, g)
       \label{eq:final_dtheta_eps}
\end{align}

We know that the two measures on $\ks\times\mathcal{G}$ defined as $\lambda(\d s, \d g)\frac{V_{\varepsilon}(s, g)}{\lambda(\varepsilon)}$  and $\lambda(\d s)v_{n}(s, g)\rho(\d g)= \lambda(\d s, \d g)v_{n}(s, g)p_{\mathcal{G}}(g)$ converges weakly to $\lambda(\d s)V^{\pi_{\theta}}(s, \d g)$. Moreover, $(s', g) \rightarrow \int_{s, a}\tilde q(s, g) p(s'|s, a)\partial_{\theta} \pi_{\theta}(a|s, g)$ is continuous and bounded because $\tilde q$, $p$ and $\partial_{\theta}\pi_{\theta}$ are continuous, bounded, and the supports are compact. Therefore, from equation~\eqref{eq:final_dtheta_n}:
\begin{align}
  \E&\left[\widehat{\delta\theta}_{\dac}^{(n)}(s, a, s', g)\right] \rightarrow_{n\rightarrow\infty}  \frac{\gamma}{c}\int_{s, a, s', g}\lambda(\d s, \d s', \d a)p_{\mathcal{G}}(g)V^{\pi_{\theta}}(s', \d g)\tilde q(s, g) p(s'|s, a)\partial_{\theta}\pi_{\theta}(a|s, g)
  \\ &= \frac{\gamma}{c}\int_{s_{0}, s, a, s', g}\lambda(\d s_{0}, \d s, \d a, \d s')p_{\mathcal{G}}(g)p_{0}(s_{0}|g)\nu^{\pi}(\d s|s_{0}, g)V^{\pi_{\theta}}(s', \d g)\gamma p(s'|s, a)\partial_{\theta}\pi_{\theta}(a|s, g)
       \\ &= \frac{1-\gamma}{c}\partial_{\theta}J(\pi_{\theta})
\end{align}
and from equation~\eqref{eq:final_dtheta_eps}
\begin{align}
  \frac{1}{\lambda(\varepsilon)}\partial_{\theta}J_{\varepsilon}(\pi_{\theta}) \rightarrow_{\varepsilon\rightarrow 0} \partial_{\theta}J(\pi_{\theta})
\end{align}
This concludes the proof.
\end{proof}

\end{document}

%% file: header.en.tex
\usepackage{cmap}
\usepackage[english]{babel}
\usepackage{amssymb,amsmath}
\usepackage{amsthm}
\usepackage{bbm}
\usepackage{microtype}

\input{header.tex}

{
\theoremstyle{yannthm}
\newtheorem{defi}{Definition}
\newtheorem*{defi*}{Definition}
\newtheorem{prop}[defi]{Proposition}
\newtheorem*{prop*}{Proposition}
\newtheorem{thm}[defi]{Theorem}
\newtheorem*{thm*}{Theorem}
\newtheorem{defithm}[defi]{Definition-Theorem}
\newtheorem{nfthm}[defi]{Informal Theorem}
\newtheorem{lem}[defi]{Lemma}
\newtheorem*{lem*}{Lemma}

\newtheorem*{cor*}{Corollary}

\newtheorem*{ex*}{Example}

\newtheorem{assumption}{Assumption}

\newtheorem*{subenonce}{}

\theoremstyle{yannthm2}

\newtheorem*{exo*}{Exercise}

\newtheorem*{rem*}{Remark}

\newtheorem*{subenonce2}{}

}

\newcommand{\transp}[1]{#1^{\!\top}\!}


%% file: header.tex
\newcommand{\thmheadercommand}[1]{\textbf{\scshape{}#1.\\*}}

\newtheoremstyle{yannthm}{\topsep}{\topsep}{\slshape}{}{\scshape\bfseries}{.}{.5em}{%
\thmname{#1}\thmnumber{ #2}\thmnote{#3}%
}
\newtheoremstyle{yannthm2}{\topsep}{\topsep}{}{}{\scshape\bfseries}{.}{.5em}{%
\thmname{#1}\thmnumber{ #2}\thmnote{#3}%
}

\def\d{\operatorname{d}\!{}}

\def\Z{{\mathbb{Z}}}

\def\R{{\mathbb{R}}}

\renewcommand{\geq}{\geqslant}
\renewcommand{\leq}{\leqslant}

\renewcommand{\emptyset}{\varnothing}

\newcommand{\deq}{\mathrel{\mathop{:}}=}
\newcommand{\eqd}{=\mathrel{\mathop{:}}}

\newcommand{\from}{\colon} 

\def\eps{\varepsilon}
\renewcommand{\epsilon}{\varepsilon}
\renewcommand{\phi}{\varphi}

\let\oldPr\Pr
\renewcommand{\Pr}{\oldPr\nolimits}
\newcommand{\E}{\mathbb{E}}

\DeclareMathOperator{\Ker}{Ker}

\DeclareMathOperator{\Id}{Id}

\newcommand{\abs}[1]{\left\lvert#1\right\rvert}
\newcommand{\norm}[1]{\left\lVert#1\right\rVert}

\newcommand{\1}{\mathbbm{1}}

\DeclareMathOperator*{\argmax}{arg\,max}